\def\OSDNARXIV{1}
\documentclass{article}

\newif\ifosdnarxiv
\ifdefined\OSDNARXIV
  \osdnarxivtrue
\else
  \osdnarxivfalse
\fi

\PassOptionsToPackage{numbers,compress}{natbib}
\ifosdnarxiv
  \usepackage[preprint]{neurips_2026}
\else
  \usepackage{neurips_2026}
\fi

\usepackage[utf8]{inputenc}
\usepackage[T1]{fontenc}
\usepackage{hyperref}
\usepackage{url}
\usepackage{graphicx}
\usepackage{booktabs}
\usepackage{amsfonts}
\usepackage{nicefrac}
\usepackage{microtype}
\usepackage[table]{xcolor}

\usepackage{amsmath}
\usepackage{amsthm}
\usepackage{mathrsfs}
\usepackage{bm}
\usepackage{siunitx}

\usepackage{float}
\usepackage{subcaption}
\usepackage{wrapfig}
\usepackage{multirow}
\usepackage{tabularx}
\usepackage{adjustbox}
\usepackage{arydshln}

\usepackage{algorithm}
\usepackage{algpseudocode}
\usepackage{listings}

\usepackage{tikz}
\usetikzlibrary{arrows.meta,positioning,calc,fit,backgrounds,shadows,matrix}
\usepackage[most]{tcolorbox}
\usepackage{pifont}
\usepackage{enumitem}
\usepackage{xspace}
\usepackage{appendix}
\usepackage{pgfgantt}
\usepackage{lipsum}

\makeatletter
\@ifundefined{nolinenumbers}{\newenvironment{nolinenumbers}{}{}}{}
\makeatother

\makeatletter
\newsavebox{\@tabnotebox}

\makeatother

\newcommand{\name}{OSDN\xspace}
\newcommand{\apf}{OSDN-APF\xspace}

\ifosdnarxiv
  \newcommand{\osdncodeurl}{\url{https://github.com/Lhongpei/OSDN}}
\else
  \newcommand{\osdncodeurl}{\url{https://anonymous.4open.science/status/OSDN-2F2D}}
\fi

\definecolor{OSDNBlue}{HTML}{005C7F}
\definecolor{OSDNOrange}{HTML}{F18F3B}
\definecolor{OSDNHighlight}{HTML}{C0005A}
\definecolor{OSDNGreen}{HTML}{5E9F6E}
\definecolor{OSDNRed}{HTML}{B92622}
\definecolor{OSDNPurple}{HTML}{7A5BA6}
\definecolor{OSDNGrey}{HTML}{6B7280}
\definecolor{OSDNLight}{HTML}{F5F7FA}
\definecolor{OSDNRowBlue}{HTML}{E5F1F5}
\definecolor{OSDNRowOrange}{HTML}{FCEBDD}
\definecolor{OSDNRowGreen}{HTML}{E8F3EA}

\newcommand{\osdnrow}{\rowcolor{OSDNRowBlue}}
\newcommand{\apfrow}{\rowcolor{OSDNRowOrange}}

\newcommand{\osdnhi}[1]{\textcolor{OSDNHighlight}{#1}}

\lstdefinestyle{pytorch}{
  basicstyle={\ttfamily\fontsize{7.4}{8.8}\selectfont},
  keywordstyle={\color{OSDNBlue!82!black}\bfseries},
  keywordstyle=[2]{\color{OSDNPurple!72!black}},
  keywordstyle=[3]{\color{OSDNOrange!88!black}},
  commentstyle={\color{OSDNGrey!92!black}\itshape},
  stringstyle={\color{OSDNGreen!55!black}},
  morekeywords={None,True,False,self,lambda,is,not,and,or,in,as,with,from,import},
  morekeywords=[2]{torch,rearrange,np,nn,F,einops},
  morekeywords=[3]{float,ones,ones_like,empty_like,zeros,zeros_like,eye,triu,
    einsum,linalg,solve_triangular,transpose,masked_fill,exp,sum,new_zeros,
    new_ones,clamp_min,cumprod,shape,device,dtype,size,view,reshape},
  numbers=left,
  numberstyle={\fontsize{5.6}{6.8}\selectfont\color{OSDNGrey!65}},
  numbersep=7pt,
  stepnumber=1,
  columns=fullflexible,
  keepspaces=true,
  showstringspaces=false,
  breaklines=true,
  breakatwhitespace=true,
  backgroundcolor={\color{OSDNLight}},
  frame=leftline,
  framerule=1.6pt,
  rulecolor={\color{OSDNBlue!72!black}},
  framesep=14pt,
  xleftmargin=3.6em,
  xrightmargin=0.5em,
  aboveskip=10pt,
  belowskip=2pt,
  captionpos=b,
  abovecaptionskip=6pt,
  belowcaptionskip=2pt
}

\hypersetup{
  colorlinks=true,
  linkcolor=OSDNBlue,
  citecolor=OSDNGreen!65!black,
  urlcolor=OSDNBlue
}

\newtheorem{theorem}{Theorem}[section]
\newtheorem{lemma}[theorem]{Lemma}
\newtheorem{proposition}[theorem]{Proposition}
\newtheorem{corollary}[theorem]{Corollary}
\theoremstyle{definition}

\theoremstyle{remark}
\newtheorem{remark}[theorem]{Remark}

\tcbset{
  osdn theorem box/.style={
    enhanced,
    breakable,
    sharp corners,
    arc=1mm,
    boxrule=0.45pt,
    left=6pt,
    right=6pt,
    top=5pt,
    bottom=5pt,
    before skip=8pt,
    after skip=8pt,
    borderline west={2pt}{0pt}{#1!78!black},
    colframe=#1!34!black,
    colback=#1!5!white
  }
}

\title{OSDN: Improving Delta Rule with Provable Online Preconditioning in Linear Attention}

\ifosdnarxiv
\author{%
  Chenyu Zhou\thanks{Equal contribution.}\\
  Shanghai Jiao Tong University\\
  \texttt{chenyuzhou@sjtu.edu.cn}
  \And
  Hongpei Li\footnotemark[1]\\
  Northwestern University\\
  \texttt{HongpeiLi2031@u.northwestern.edu}
  \And
  Yuerou Liu\\
  Huazhong University of Science and Technology\\
  \texttt{u202210581@hust.edu.cn}
  \AND
  Jianghao Lin\\
  Shanghai Jiao Tong University\\
  \texttt{linjianghao@sjtu.edu.cn}
  \And
  Dongdong Ge\\
  Shanghai Jiao Tong University\\
  \texttt{ddge@sjtu.edu.cn}
  \And
  Yinyu Ye\\
  Stanford University\\
  \texttt{yinyu-ye@stanford.edu}
}
\else
\author{Anonymous Author(s)}
\fi

\begin{document}

\maketitle

\begin{abstract}
Linear attention and state-space models offer constant-memory alternatives to softmax attention, but often struggle with in-context associative recall. The Delta Rule mitigates this by writing each token via one step of online gradient descent. However, its step size relies on a single scalar gate that ignores the feature-wise curvature of the inner objective. We propose \textbf{Online Scaled DeltaNet (\name)}, which augments the scalar gate with a diagonal preconditioner updated online via hypergradient feedback. Crucially, this right-preconditioning is algebraically equivalent to a per-feature scaling of the write-side key. This equivalence allows \name{} to strictly preserve the hardware-friendly chunkwise parallel pipeline of DeltaNet without incurring high-dimensional state overhead. Theoretically, by exploiting the exact-quadratic structure of the inner regression loss, we establish super-geometric convergence against a right-Newton comparator and prove an algorithm-aligned token-local residual contraction bound. To handle non-stationary contexts, we further introduce Adaptive Preconditioner Forgetting (APF) to dynamically refresh stale calibration. Empirically, \name{} demonstrates strong performance across scales. At the 340M-parameter scale, \name{} improves JRT-style in-context recall by 32\% over DeltaNet. Scaling to 1.3B parameters, it achieves a 39\% reduction in the recall residual ratio while maintaining parity on general downstream tasks (e.g., perplexity and LongBench) -- demonstrating that our online-preconditioning mechanism effectively transfers and amplifies at the billion-parameter scale. Code is available at \osdncodeurl.
\end{abstract}

\section{Introduction}

Linear attention~\citep{katharopoulos_transformers_2020,choromanski_rethinking_2021}
and state-space models~\citep{gu_efficiently_2021,gu_mamba_2024,dao_transformers_2024}
compress the prefix into a fixed-size matrix-valued state
$S_t \in \mathbb{R}^{V \times K}$ updated by an additive recurrence,
restoring $\mathcal{O}(N)$ inference at the documented cost of weakened
in-context retrieval~\citep{Sun2023RetentiveNA,arora_zoology_2024,arora_simple_2024,wen_rnns_2024}.
The \emph{delta rule}~\citep{schmidhuber_learning_1992,schlag_linear_2021}
narrows that gap with a read-then-write update
$S_t = S_{t-1}(I - \beta_t k_t k_t^\top) + \beta_t v_t k_t^\top$, which
can be read as one step of online gradient descent on the per-token
regression loss $f_t(S) = \tfrac12\|S k_t - v_t\|_F^2$~\citep{yang_parallelizing_2024}.
Combined with chunkwise WY parallelisation~\citep{yang_parallelizing_2024}
and gated variants~\citep{yang_gated_2024-1,kimiteam2025kimilinearexpressiveefficient},
DeltaNet-style models close much of the recall gap. The optimisation
view, however, exposes one structural choice that has remained
untouched: the learning rate $\beta_t$ is a single scalar, applied
uniformly to every key dimension -- the recurrent counterpart of vanilla
SGD, forgoing the by-now-standard role of diagonal preconditioning in
adaptive optimisation~\citep{duchi_adaptive_2011,kingma_adam_2015,gupta_shampoo_2018}.

This scalar choice is especially restrictive in associative recall. A
single prompt may contain stable identifiers, high-entropy values,
formatting tokens, and distractors whose keys recur at different rates and
with different empirical curvature. A scalar write gate must compromise
across these directions: increasing it helps one association overwrite
stale memory, but can over-correct another direction that is already
well-calibrated. The natural fix is the one used throughout optimisation
-- rescale coordinates before taking the gradient step -- but in a
sequence layer this fix is only useful if it does not require a dense
second-order state, does not look at the high-dimensional memory residual,
and does not break the chunkwise parallel DeltaNet kernel.

\begin{figure*}[t]
    \centering
    \includegraphics[width=0.9\textwidth]{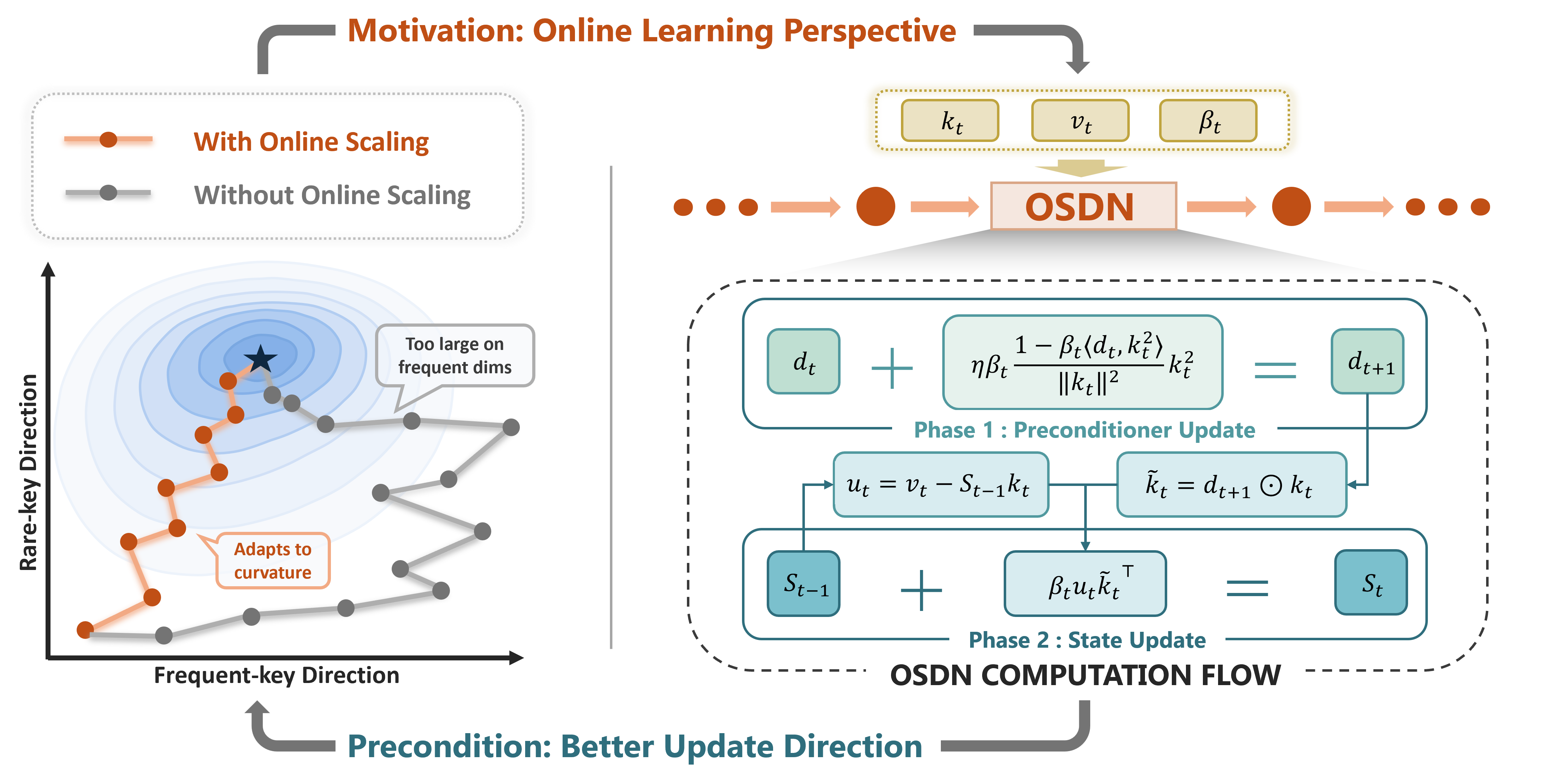}
    \caption{\textbf{Motivation and Computation Flow of Online Scaled DeltaNet (OSDN).} \textbf{Left:} From an online learning perspective, standard DeltaNet applies a uniform scalar learning rate, struggling to adapt to optimization directions with varying curvatures (e.g., frequent vs. rare keys). OSDN resolves this by introducing a diagonal preconditioner that dynamically scales the update directions. \textbf{Right:} The OSDN computation flow is decoupled into two phases: a lightweight preconditioner update (Phase 1) followed by the primary state update (Phase 2). This decoupled design strictly preserves the efficiency of hardware-friendly chunkwise parallelization.}
    \label{fig:osdn_banner}
\end{figure*}

We propose \textbf{Online Scaled DeltaNet (\name)}, which augments the
scalar gate with a diagonal preconditioner
$D_t = \mathrm{diag}(d_t)$ learned online through the hypergradient
feedback of OSGM~\citep{gao2024gradient}. The construction rests on three
properties of the inner regression objective $f_t$.
\textbf{(i) Seamless Integration.} Right-preconditioning the gradient is
algebraically identical to a per-feature scaling of the write-side key,
$\tilde k_t = d_t \odot k_t$; the chunkwise WY pipeline is preserved
under the single substitution $K\mapsto\tilde K$ on the storage side
(Section~\ref{sec:preconditioned_delta}).
\textbf{(ii) Decoupling.} The hypergradient that drives $d_t$ depends
only on the key sequence and the scalar gates, never on $S_t$, $v_t$, or
the residual; the full sequence $\{d_t\}$ thus reduces to an
$\mathcal O(K)$-state affine recurrence schedulable before the
chunkwise write pass.
\textbf{(iii) Provability.} Because the inner loss is exactly quadratic,
its Hessian is constant and the Hessian-Lipschitz constant vanishes; we
obtain a population-limit super-geometric rate against the right-Newton
comparator (Theorem~\ref{thm:main_population}) and an
algorithm-aligned token-local residual-contraction bound on the
per-token hypergradient surrogate the implementation actually
optimises (Theorem~\ref{thm:main_tokenlocal}).

A pure hypergradient update is accumulative -- adequate for stationary
keys, but stale under non-stationary language contexts. We therefore
introduce \emph{Adaptive Preconditioner Forgetting (APF)}, a learned
token-wise, head-wise retention gate applied to $d_t$ alone (not to the
high-dimensional state $S_t$); the resulting recurrence remains affine
and the two-phase scan is preserved (Section~\ref{subsec:apf}). We refer
to the resulting variant as \apf.

\paragraph{Contributions and findings.}
The paper makes four main contributions. First, it identifies diagonal
right-preconditioning as a minimal extension of the Delta rule and proves
its exact write-key scaling equivalence. Second, it derives the
hypergradient update and shows that the whole preconditioner trajectory is
an $\mathcal O(K)$ affine scan over keys and scalar gates. Third, it
introduces APF as a non-stationary refinement that refreshes the
meta-optimiser state without decaying the memory state. Fourth, it proves
two mechanism-level super-geometric guarantees against diagonal and
right-Newton comparators, with the algorithm-facing bound stated directly
on the residual-ratio surrogate measured in our diagnostics.

Empirically, the gains appear where the mechanism predicts. At matched
340M-parameter / 10B-token compute, vanilla \name{} improves JRT-style
in-context recall by 32\% over DeltaNet and lowers the directly measured
repeated-prompt residual-ratio geometric mean from 0.537 to 0.433;
\apf{} retains a 17\% recall gain and is the more stable variant on
long-context perplexity. Scaling the same comparison to 1.3B parameters
and 100B tokens nearly doubles the residual-ratio reduction (0.432 to
0.265, a 39\% drop -- the lowest $q_{\mathrm{geo}}$ across either scale),
while perplexity, commonsense, and LongBench averages stay at parity
with DeltaNet (Section~\ref{sec:exp}, Appendix~\ref{sec:app_scale_1p3b}).

\section{Related Work}
\label{sec:related}

Linear attention and state-space models trade softmax attention's
quadratic cache for a fixed recurrent state~\citep{katharopoulos_transformers_2020,choromanski_rethinking_2021,gu_mamba_2024,dao_transformers_2024}.
This makes long-sequence inference practical, but it also concentrates
all prefix information into a limited matrix state, producing known
weaknesses on associative recall and exact copying~\citep{arora_zoology_2024,arora_simple_2024,wen_rnns_2024,jelassi_repeat_2024}.
Modern recurrent linear-attention layers improve this tradeoff through
decay, data-dependent gating, or more expressive transition maps:
RetNet~\citep{Sun2023RetentiveNA}, GLA~\citep{yang_gated_2024},
Gated DeltaNet~\citep{yang_gated_2024-1}, RWKV-7~\citep{peng_rwkv7_2025},
and KDA / Kimi Linear~\citep{kimiteam2025kimilinearexpressiveefficient}
all modify how the state is retained or overwritten. \name{} is
orthogonal to this line: it leaves the base transition and memory state
layout intact, but changes the geometry of the write step by replacing a
scalar update size with an online diagonal preconditioner.

The delta-rule branch is especially natural for this intervention.
Fast-weight programmers~\citep{schmidhuber_learning_1992,ba_using_2016,schlag_linear_2021}
interpret the recurrent state as weights written by the slow network at
test time; DeltaNet makes this explicit by writing through one online
gradient step on a token-local regression loss~\citep{schlag_linear_2021,yang_parallelizing_2024}.
Prior improvements mainly change parallelisation, gating, or transition
expressivity: chunkwise WY parallelisation closes the hardware gap with
softmax attention~\citep{yang_parallelizing_2024}; Gated DeltaNet adds a
state forget gate~\citep{yang_gated_2024-1}; KDA uses fine-grained channel
decay~\citep{kimiteam2025kimilinearexpressiveefficient}; and
DeltaProduct raises per-step expressivity with Householder
transitions~\citep{siems_deltaproduct_2025}. In contrast, \name{} keeps
DeltaNet's rank-one residual write and scalar gate, then learns a
feature-wise multiplier $d_t$ on the write key. The contribution is
therefore deliberately narrow: an optimizer-style preconditioner inside
the existing recurrence, not a new recurrent backbone.

\name{} also connects to the broader view of sequence layers as
test-time optimisers. Linear attention can implement gradient descent on
in-context regression~\citep{vonoswald_transformers_2023,akyurek_what_2023};
TTT-style models train an explicit inner model during the forward
pass~\citep{sun_learning_2024,wang_testtime_2025}; and MesaNet solves a
cumulative least-squares problem to high accuracy at each token~\citep{vonoswald_mesanet_2025}.
\name{} sits between scalar first-order updates and exact prefix solves.
It imports diagonal preconditioning from adaptive optimisation and online
hypergradient methods~\citep{duchi_adaptive_2011,kingma_adam_2015,baydin_online_2018,gao2024gradient}
while preserving the $\mathcal O(K)$ recurrent state and the chunkwise
DeltaNet kernel. Appendix~\ref{sec:app_related} gives the expanded
taxonomy, including Table~\ref{tab:ttt_view}.

\section{Preliminary}
\label{sec:prelim}

We use lower-case bold for vectors and upper-case for matrices. A sequence
layer maintains a state $S_t \in \mathbb{R}^{V \times K}$ updated from
query/key/value triples $(q_t, k_t, v_t)$, with $q_t, k_t \in \mathbb{R}^K$
and $v_t \in \mathbb{R}^V$. A chunk of length $C$ stacks intra-chunk tokens
row-wise as $K_{[t]}, \osdnhi{\tilde K_{[t]}} \in \mathbb{R}^{C\times K}$ and
$V_{[t]} \in \mathbb{R}^{C\times V}$. We write $\odot$ for the Hadamard
product, $\sigma(\cdot)$ for the elementwise sigmoid,
$\mathrm{tril}(A,-1)$ for the strict lower triangle of $A$, and
$\mathcal H^\dagger$ for the Moore--Penrose pseudoinverse. Let
$\Sigma_k = \mathbb{E}[k_t k_t^\top] \in \mathbb{R}^{K\times K}$ denote the
(uncentred) key covariance and $L = \lambda_{\max}(\Sigma_k)$ its top
eigenvalue, which serves as the smoothness constant of the in-context
regression loss.

\paragraph{Linear attention, gating, and the Delta rule.}
A canonical linear-attention layer~\citep{katharopoulos_transformers_2020}
reads from and writes to $S_t$ via the additive recurrence
$S_t = S_{t-1} + v_t k_t^\top$, $o_t = S_t q_t$. Subsequent work generalises
the transition with multiplicative gating, decay, or a rank-one
perturbation, $S_t = S_{t-1} P_t + \omega_t v_t k_t^\top$. The choices
$(P_t, \omega_t) = (I, 1)$ recover the additive case;
$(\alpha_t I, 1)$ recovers RetNet~\citep{Sun2023RetentiveNA};
$(\mathrm{Diag}(\alpha_t), 1)$ recovers GLA~\citep{yang_gated_2024}; and
$(I - \beta_t k_t k_t^\top, \beta_t)$ recovers DeltaNet's read-then-write
update $S_t = S_{t-1} + \beta_t u_t k_t^\top$ with residual
$u_t = v_t - S_{t-1} k_t \in \mathbb{R}^V$ and scalar gate
$\beta_t = \sigma(w_\beta^\top x_t + b_\beta)\in(0,1)$. Gated DeltaNet adds
a scalar forget gate~\citep{yang_gated_2024-1}; KDA replaces the scalar
gate by a fine-grained vector $\boldsymbol\alpha_t\in(0,1]^K$~\citep{kimiteam2025kimilinearexpressiveefficient}.

The WY chunkwise transform of~\citet{yang_parallelizing_2024} avoids
per-token sequentialism by reducing all intra-chunk computations to dense
matrix multiplications well-suited to tensor cores; only the chunk-boundary
state $S_{[t]} \to S_{[t+1]}$ propagates between chunks. We adopt this
chunkwise framework throughout, and Section~\ref{subsec:recurrent_chunk}
specifies its OSDN form. A unified taxonomy of these layers under the
``sequence layer as inner optimiser'' reading -- which places OSDN
between MesaNet's exact $\arg\min$~\citep{vonoswald_mesanet_2025} and
the additive Hebbian write -- is given in
Appendix~\ref{sec:app_related}.

\paragraph{The Online Scaled Gradient Method.}
Our analysis builds on the OSGM framework of~\citet{gao2024gradient},
which casts the choice of preconditioner in a gradient step as an online
learning problem. For an $L$-smooth convex objective $f$ and a step
$x^+ = x - P\nabla f(x)$, OSGM defines the \emph{hypergradient surrogate}
\begin{equation}
    h_x(P) \;:=\; \frac{f(x - P\nabla f(x)) - f(x)}{\|\nabla f(x)\|^2},
    \label{eq:osgm_surrogate}
\end{equation}
i.e.\ the relative loss change of one step. The surrogate is convex in
$P$, non-positive only for descending preconditioners, and admits a
sublinear-regret online learner whose cumulative regret translates into
a function-value bound
$f(x^{T+1}) - f^* \le (C\,\mathcal R_T/T)^T\,(f(x^1)-f^*)$ via an AM--GM
reduction. On non-quadratic losses the analysis carries a residual
scaling with the Hessian-Lipschitz constant. The DeltaNet sequence-level
loss is exactly quadratic with a constant Hessian, which collapses that
residual to zero; this enables the sharp super-geometric statement we
prove in the idealised setting and motivates the diagonal instantiation
($d_t$) developed next. The full OSGM background, including the three
convexity / descent / regret properties invoked by our theory, is
recapitulated in Appendix~\ref{sec:theory}.

\section{A Preconditioned Delta Rule}
\label{sec:preconditioned_delta}

The motivation is the OSGM view of preconditioned gradient dynamics:
on a quadratic memory-regression objective, an online learner with low
regret against a good preconditioner can drive the residual left by each
write to contract much faster than a fixed scalar step. A direct use of
that population-level hypergradient, however, is not a usable DeltaNet
layer. It is tied to the current memory state \(S_t\), the full gradient
or residual, and a dense right-preconditioner comparator, so the
preconditioner trajectory cannot be isolated from the state-side write
pass; this would destroy the chunkwise WY schedule that makes DeltaNet
efficient.

This section therefore derives an implementation-aligned version of the
OSGM idea. We restrict the right preconditioner to a diagonal vector
\(d_t\), use the token-local DeltaNet regression loss, and exploit its
exact quadratic form. The resulting hypergradient collapses to a
closed-form recurrence depending only on the key stream and scalar gates.
This decoupling lets us first materialise write-side keys
\(\tilde k_t=d_t\odot k_t\), then run the baseline chunkwise DeltaNet
kernel with a single storage-side key substitution. The full chunkwise
derivation, all proofs, and the hardware-level cost analysis are
deferred to Appendices~\ref{app:method_details} and~\ref{app:chunk_impl}.

\subsection{From scalar gate to diagonal preconditioner}
\label{subsec:equivalence}

DeltaNet's update $S_t = S_{t-1} + \beta_t u_t k_t^\top$ with
$u_t = v_t - S_{t-1} k_t$ is one step of online gradient descent on the
per-token regression loss $f_t(S) = \tfrac12 \|Sk_t - v_t\|_F^2$, since
$\nabla f_t(S_{t-1}) = -u_t k_t^\top$~\citep{yang_parallelizing_2024}. The
scalar gate $\beta_t$ acts as a single learning rate shared across key
coordinates; we replace it by a gate--preconditioner composition
$\beta_t \osdnhi{D_t}$ with $\osdnhi{D_t = \mathrm{diag}(d_t)}$ and
$\osdnhi{d_t} \in \mathcal{D} := [d_{\min}, d_{\max}]^K$, $d_{\min}>0$. The
right-preconditioned step then reads
\begin{equation}
    S_t \;=\; S_{t-1} - \beta_t \nabla f_t(S_{t-1})\,\osdnhi{D_t}
        \;=\; S_{t-1} + \beta_t u_t \osdnhi{(d_t \odot k_t)}^\top
        \;=\; S_{t-1} + \beta_t u_t \osdnhi{\tilde k_t^\top},
    \label{eq:equiv_update}
\end{equation}
where the second equality uses $k_t^\top D_t = (d_t \odot k_t)^\top$ to
turn the diagonal preconditioner into a per-feature scaling of the
write-side key. Defining the \emph{preconditioned key}
$\osdnhi{\tilde k_t := d_t \odot k_t}$, the OSDN update is identical to
DeltaNet up to substitution of $k_t$ by $\tilde k_t$ on the write side
only; the read $S k_t$, residual $u_t$, and rank-one structure are
unchanged.

\subsection{Decoupled hypergradient feedback}
\label{subsec:hypergradient}

The preconditioner $d_t$ is updated by online gradient descent on the
hypergradient surrogate adapted from~\citet{gao2024gradient},
$h_t(d_t) = \bigl[f_t(S_t) - f_t(S_{t-1})\bigr]/\|\nabla f_t(S_{t-1})\|_F^2$.
The exact-quadratic structure of $f_t$ closes this surrogate analytically:

\begin{lemma}[Closed-form hypergradient]
\label{lem:hypergradient}
For any $d_t \in \mathbb{R}^K$, $\beta_t \in (0,1)$, and $k_t \neq 0$,
\begin{equation}
    h_t(d_t) \;=\; \frac{(1 - \beta_t \langle d_t, k_t^2\rangle)^2 - 1}{2\,\|k_t\|_2^2},
    \qquad
    \nabla_d h_t(d_t) \;=\; -\beta_t\,\frac{1 - \beta_t \langle d_t, k_t^2\rangle}{\|k_t\|_2^2}\, k_t^2,
\end{equation}
where $k_t^2$ denotes the elementwise square of $k_t$. Proof in
Appendix~\ref{app:method_details}.
\end{lemma}

Two consequences make this practical. \textbf{(i) Decoupling.} Both $h_t$
and $\nabla_d h_t$ depend only on the key $k_t$ and the scalar gate $\beta_t$;
neither $S_t$, $v_t$, nor the residual $u_t$ enters. The full sequence
$\{d_t\}$ is therefore computable from the key stream alone, before any
state-side work. \textbf{(ii) Affine recurrence.} The OGD step
$d_{t+1} = d_t - \eta\,\nabla_d h_t(d_t)$, followed by projection onto
$\mathcal D$, is a piecewise-affine map in $d_t$:
\begin{equation}
    \bar d_{t+1} \;=\;
    \Bigl(I - \tfrac{\eta\beta_t^2}{\max(\|k_t\|_2^2,\epsilon)}\,k_t^2 (k_t^2)^\top\Bigr) d_t
    + \tfrac{\eta\beta_t}{\max(\|k_t\|_2^2,\epsilon)}\,k_t^2,
    \quad
    d_{t+1} \;=\; \Pi_{\mathcal D}(\bar d_{t+1}),
    \label{eq:dt_update}
\end{equation}
giving an $\mathcal O(K)$ streaming-state scan per head. Under normalised
keys $\|k_t\|_2^2 = 1$ the surrogate is at most $1$-smooth, the bounded box
$\mathcal D=[0.5, 2.0]^K$ used in our experiments yields strict per-token
descent of $f_t$ (Corollary~\ref{cor:monotone_descent} in
Appendix~\ref{app:method_details}), and the reported runs use
$\eta = 0.003$ as the practical online step size
(Appendix~\ref{sec:app_reproducibility}).

\newcommand{\algHi}[1]{\begingroup\setlength{\fboxsep}{1.5pt}\colorbox{OSDNHighlight!10}{\(\displaystyle #1\)}\endgroup}

\begin{algorithm}[H]
\caption{\textbf{Phase~1: online preconditioner sweep.}
Shaded lines emit the write key for phase~2 and update the
preconditioner state. The box clamp realising $\Pi_{\mathcal D}$ is omitted
for clarity.}
\label{alg:precond_online}
\begin{algorithmic}[1]
\Require Keys $k_1,\dots,k_L \in \mathbb{R}^K$, gates
    $\beta_1,\dots,\beta_L \in (0,1)$, initial $d^{(0)} \in \mathbb{R}^K_{>0}$,
    learning rate $\eta>0$, floor $\epsilon>0$.
\Ensure Preconditioned keys $\tilde k_1,\dots,\tilde k_L$ and final $d_{L+1}$.
\State $d \gets d^{(0)}$
\For{$t=1,\ldots,L$}
    \State \algHi{\tilde k_t \gets d \odot k_t}
    \Comment{materialise write key for phase~2}
    \State $k_t^2 \gets k_t \odot k_t$;\quad $n_t \gets \max(\|k_t\|_2^2,\epsilon)$
    \State \algHi{d \gets d \;+\; \eta\beta_t\,
        \dfrac{1 - \beta_t\langle d, k_t^2\rangle}{n_t}\, k_t^2}
    \Comment{hypergradient step; $d$ becomes $d_{t+1}$}
\EndFor
\end{algorithmic}
\end{algorithm}

\subsection{Recurrent and chunkwise form}
\label{subsec:recurrent_chunk}

Once phase~1 has emitted the write-side key \(\tilde k_t\), the state
update is just DeltaNet with an asymmetric storage factor. Substituting
$u_t = v_t - S_{t-1} k_t$ into~\eqref{eq:equiv_update} yields
\begin{equation}
    S_t \;=\; S_{t-1}\bigl(I - \beta_t k_t \osdnhi{\tilde k_t^\top}\bigr)
            + \beta_t v_t \osdnhi{\tilde k_t^\top}.
    \label{eq:recurrent_form}
\end{equation}
Compared with DeltaNet's symmetric transition $(I - \beta_t k_t k_t^\top)$,
the rank-one perturbation becomes asymmetric but retains its identity-minus
form, so the WY chunkwise transform of~\citet{yang_parallelizing_2024}
applies essentially verbatim. The only changes are at the chunk Gram and
the intra-chunk score: for a chunk of length $C$ stacking
$K_{[t]}, \osdnhi{\tilde K_{[t]}} \in \mathbb{R}^{C\times K}$ and gates
$B_{[t]} = \mathrm{diag}(\boldsymbol\beta_{[t]})$,
\begin{equation}
    B_{[t]} K_{[t]} K_{[t]}^\top \;\longmapsto\;
    B_{[t]} K_{[t]} \osdnhi{\tilde K_{[t]}^\top},\qquad
    Q_{[t]} K_{[t]}^\top \;\longmapsto\;
    Q_{[t]} \osdnhi{\tilde K_{[t]}^\top}.
\end{equation}
The UT inverse, matrix shapes, and tensor-core layout carry over from the
DeltaNet kernel; the full chunkwise derivation appears in
Appendix~\ref{app:method_details}.

\subsection{Adaptive preconditioner forgetting}
\label{subsec:apf}

The hypergradient update of~\eqref{eq:dt_update} accumulates evidence
without forgetting, which is appropriate for a stationary key
distribution but stale under non-stationary language contexts where
topics, formats, and local key directions shift within a document. We
therefore add \emph{Adaptive Preconditioner Forgetting (APF)}, a
retention gate applied only to the preconditioner state $d_t$ and
\emph{not} to the high-dimensional memory $S_t$. The implementation
predicts a token-wise, head-wise scalar
$\osdnhi{r_{t,h}} = \sigma(w_{r,h}^\top x_t + b_{r,h})$ and broadcasts it
over the key dimension; the affine recurrence becomes
\begin{equation}
    \bar d_{t+1} \;=\; \osdnhi{r_{t,h}}\, d_t
        + \eta\beta_t\,\frac{1 - \beta_t\langle d_t, k_t^2\rangle}{\max(\|k_t\|_2^2,\epsilon)}\, k_t^2,
    \qquad d_{t+1} = \Pi_{\mathcal D}(\bar d_{t+1}).
    \label{eq:dt_update_apf}
\end{equation}
Setting $r_{t,h}\equiv 1$ recovers~\eqref{eq:dt_update}. The recurrence
remains affine in $d_t$ followed by a coordinate-wise projection, so the
phase-1 sweep is unchanged and adds only $H(d_\mathrm{m}+1)$ parameters
per layer ($\le 0.3\%$ at our 340M scale). APF is \emph{not} memory-state
decay: $S_t$, $v_t$, and the residual update are untouched. A
proposition formalising the affine-recurrence preservation, and its
proof, are in Appendix~\ref{app:method_details}.

\subsection{Two-phase implementation}
\label{subsec:twophase}

The decoupling property organises OSDN into two isolated phases.
\textbf{Phase~1} runs Algorithm~\ref{alg:precond_online} (or its APF
variant) on the key stream and gates to emit
$\{\tilde k_1,\dots,\tilde k_L\}$ and the next $d$; this is an
$\mathcal O(LK)$ scan with $\mathcal O(K)$ streaming state.
\textbf{Phase~2} runs the standard chunkwise DeltaNet pass with
$\tilde K_{[t]}$ replacing $K_{[t]}$ wherever the chunk rule reads the
write-side key (the chunk Gram and intra-chunk score). The UT inverse,
matrix shapes, and tensor-core tile layout are unchanged from DeltaNet,
and the persistent recurrent state grows by only the $K$-vector $d_t$
($\le 0.05\%$ of the recurrent state at our scale; cost breakdown in
Appendix~\ref{app:chunk_impl}). The backward pass factors cleanly:
phase~2 produces $\partial\mathcal L/\partial \tilde k_t$, from which
$\partial\mathcal L/\partial d_t = k_t \odot \partial\mathcal L/\partial \tilde k_t$
and $\partial\mathcal L/\partial k_t \mathrel{+}= d_t \odot \partial\mathcal L/\partial \tilde k_t$;
the gradient through $d_t$ is propagated by reverse-mode sweep over the
projected affine recurrence.

\section{Theory: Residual Contraction}
\label{sec:theory_summary}

The benefit of online scaling is most transparent in the ideal OSGM
picture. For the quadratic memory-regression objective, a right
preconditioner that competes well with the Newton comparator yields a
super-geometric contraction of the fast-weight residual. This is the
first guarantee below, and it is the conceptual reason to replace
DeltaNet's scalar step by a learned preconditioner.

This first guarantee is not, by itself, a chunkwise layer. Its hypergradient is
defined through a population objective and the current memory state, so
the resulting preconditioner cannot be computed independently of the
state-side write pass. Section~\ref{sec:preconditioned_delta} therefore
uses an efficiency-aligned surrogate: diagonal, token-local, and
closed-form in the key stream. The second guarantee shows that this
implemented surrogate retains the same regret-to-contraction shape,
while controlling the product of token-local residual ratios rather than
the suboptimality of a single global objective. Full proofs and
edge-case discussions are in Appendix~\ref{sec:theory}.

\begin{theorem}[Population-limit right-Newton comparator]
\label{thm:main_population}
Let
$f(S)=\tfrac12\mathbb E\|S k_t-v_t\|_F^2$, with key covariance
$\Sigma_k=\mathbb E[k_t k_t^\top]$ and
$L=\lambda_{\max}(\Sigma_k)$. Consider full-gradient right-preconditioned
dynamics
$S_{t+1}=S_t-\nabla f(S_t)D_t$ and the OSGM feedback
\[
    h_t(D)=\frac{f(S_t-\nabla f(S_t)D)-f(S_t)}
                 {\|\nabla f(S_t)\|_F^2}.
\]
If the updates are monotone, $f(S_{t+1})\le f(S_t)$, and the online
preconditioner has regret $\mathcal R_T$ against the ideal right-Newton
comparator $D_\star=\Sigma_k^\dagger$,
\[
    \sum_{t=1}^T \bigl(h_t(D_t)-h_t(D_\star)\bigr)\le \mathcal R_T,
\]
then
\[
    f(S_{T+1})-f^\star
    \le
    \bigl[f(S_1)-f^\star\bigr]
    \left(\frac{2L\,\mathcal R_T}{T}\right)^T .
\]
\end{theorem}

This theorem gives the clean OSGM motivation: low regret against
\(D_\star\) makes the average progression ratio shrink, and AM--GM turns
that into the \(T\)-th power. Its assumptions also mark the implementation
gap: the comparator is a full right preconditioner, and the feedback is
coupled to the population loss and the state trajectory.

\begin{theorem}[Algorithmic token-local residual contraction]
\label{thm:main_tokenlocal}
For the implemented diagonal update, let
$f_t(S)=\tfrac12\|S k_t-v_t\|_F^2$, assume $\|k_t\|_2=1$, write
$s_t=k_t^{\odot 2}$, and let the online learner satisfy diagonal regret
$R_T(d)$ on the token-local feedback
$h_t(d)=[f_t(S_t(d))-f_t(S_{t-1})]/\|\nabla f_t(S_{t-1})\|_F^2$.
Define
\[
    \varepsilon_T(d)=\frac{1}{2T}\sum_{t=1}^T
      \bigl(1-\beta_t\langle d,s_t\rangle\bigr)^2,\qquad
    \varepsilon_T^{\mathrm{diag}}=\min_{d\in\mathcal D}\varepsilon_T(d).
\]
Then, along the actual sequence of token-local writes,
\[
    \prod_{t=1}^T
    \frac{f_t(S_t)}{f_t(S_{t-1})}
    \le
    \left(2\varepsilon_T^{\mathrm{diag}}
          +\frac{2R_T}{T}\right)^T .
\]
If a feasible diagonal comparator satisfies the gated Newton condition
$\beta_t\langle d^\star,s_t\rangle=1$ for all $t$, then
$\varepsilon_T^{\mathrm{diag}}=0$; with standard projected-OGD regret
$R_T=O(\sqrt T)$, the contraction becomes
$\bigl(O(1)/\sqrt T\bigr)^T$.
\end{theorem}

\section{Experiments}
\label{sec:exp}

We assess whether online preconditioning changes the behaviour of a
DeltaNet backbone in the regime it is designed for: associative
retrieval through fast-weight writes. The OSDN rows share the same
DeltaNet architecture and training budget, so they isolate the effect of
replacing DeltaNet's scalar write step with the online preconditioner
$d_t$, and then of adding APF. Matched-scale Gated DeltaNet (GDN) and
KDA rows are reported alongside as scope checks.
\begin{table}[H]
\centering\scriptsize
\setlength{\tabcolsep}{2.2pt}
\renewcommand{\arraystretch}{1.12}
\caption{\textbf{Main results: matched 340M runs and 1.3B / 100B
scaling.} LM PPL is the WikiText/LAMBADA geometric mean; PG-19 is the
20K-token length-extrapolation perplexity ($\downarrow$). Recall is
JRT-style \textit{contains} accuracy at 2K context: all averages the
single and repeated splits, single averages FDA/SWDE/SQuAD, and repeated
averages their \emph{-twice} variants ($\uparrow$). The 1.3B recall
columns average FDA and SWDE only because SQuAD did not return reliable
contains-accuracy at that scale. $\Delta_{\mathrm{rep}}$ is the absolute
repeated-recall difference to the matched baseline. Common.\ averages
eight zero-shot tasks; LongBench is the 14-task English average; tok/s
$\Delta$ is the single-H100 inference-throughput change relative to the
matched baseline at the same scale and backbone. Boldface marks the best value
within each dashed block.}
\label{tab:main_results}
\begin{adjustbox}{max width=\linewidth}
\begin{tabular}{@{}l|cc|cccc|cc|c@{}}
\toprule
\textbf{Model} &
\multicolumn{2}{c|}{\textbf{PPL} \(\downarrow\)} &
\multicolumn{4}{c|}{\textbf{JRT recall @2K} \(\uparrow\)} &
\shortstack[c]{\textbf{Common.}\\[-0.15ex]\tiny \(\uparrow\)} &
\shortstack[c]{\textbf{LongBench}\\[-0.15ex]\tiny \(\uparrow\)} &
\shortstack[c]{\textbf{tok/s}\\[-0.15ex]\tiny \(\Delta\)} \\
\cmidrule(lr){2-3}\cmidrule(lr){4-7}
 &
\shortstack[c]{\textbf{LM}\\[-0.15ex]\tiny GeoMean} &
\shortstack[c]{\textbf{PG-19}\\[-0.15ex]\tiny final} &
\shortstack[c]{\textbf{all}\\[-0.15ex]\tiny avg.} &
\shortstack[c]{\textbf{single}\\[-0.15ex]\tiny avg.} &
\shortstack[c]{\textbf{rep.}\\[-0.15ex]\tiny avg.} &
\shortstack[c]{\(\boldsymbol{\Delta_{\mathrm{rep}}}\)\\[-0.15ex]\tiny vs. base} &
 &
 &
 \\
\midrule
\multicolumn{10}{@{}l}{\textit{340M / 10B}} \\
DeltaNet
                         & 32.00          & 20.78          & 0.150          & 0.155          & 0.145          & --              & \textbf{0.457} & 0.072          & -- \\
\osdnrow
\name{}                   & 31.67          & 20.02          & \textbf{0.198} & \textbf{0.179} & \textbf{0.218} & \textbf{+0.073} & 0.456          & \textbf{0.087} & \(\mathbf{-0.2\%}\) \\
\apfrow
\apf{}                    & \textbf{30.99} & \textbf{19.85} & 0.176          & 0.152          & 0.199          & +0.054          & 0.456          & 0.073          & \(-2.2\%\) \\
\hdashline
GDN                       & 30.01          & 20.11          & 0.154          & 0.169          & 0.139          & --              & \textbf{0.463} & 0.073          & -- \\
\osdnrow
OSGDN                     & 29.78          & \textbf{19.70} & 0.182          & 0.169          & 0.195          & +0.056          & \textbf{0.463} & 0.073          & \(\mathbf{-2.0\%}\) \\
\apfrow
OSGDN-APF                 & \textbf{29.50} & 20.21          & \textbf{0.203} & \textbf{0.185} & \textbf{0.221} & \textbf{+0.082} & 0.458          & \textbf{0.080} & \(\mathbf{-2.0\%}\) \\
\hdashline
KDA                       & \textbf{26.75} & 18.73          & 0.168          & 0.187          & 0.150          & --              & 0.470          & 0.088          & -- \\
\osdnrow
OSKDA                     & 27.93          & \textbf{18.53} & 0.175          & \textbf{0.218} & 0.133          & -0.017          & 0.470          & 0.090          & \(\mathbf{+5.5\%}\) \\
\apfrow
OSKDA-APF                 & 28.02          & 19.00          & \textbf{0.185} & 0.191          & \textbf{0.179} & \textbf{+0.029} & \textbf{0.473} & \textbf{0.098} & \(+1.4\%\) \\
\midrule
\multicolumn{10}{@{}l}{\textit{1.3B / 100B}} \\
DeltaNet
                         & 14.28          & --             & 0.260          & 0.293          & \textbf{0.227} & --              & 0.560          & 0.115          & -- \\
\apfrow
\apf{}                    & \textbf{14.22} & --             & \textbf{0.266} & \textbf{0.315} & 0.217          & -0.010          & \textbf{0.566} & \textbf{0.116} & \(-6.8\%\) \\
\bottomrule
\end{tabular}
\end{adjustbox}
\end{table}

\paragraph{Setup.}
The 340M sweep uses a 24-layer, 1024-hidden, 8-head DeltaNet backbone
trained from scratch on FineWeb-Edu (10B tokens, 20{,}480 optimizer
steps), with matched optimizer, schedule, sequence packing, and hardware
(full configuration in Appendix~\ref{sec:app_reproducibility}). Vanilla
\name{} is parameter-free; \apf{} adds at most $0.3\%$ parameters. The
1.3B / 100B rows scale the same protocol to a matched DeltaNet vs.\
\apf{} pair on the same corpus
(Appendix~\ref{sec:app_scale_1p3b}). The primary diagnostic is in-context
recall, measured with the JRT-style cloze format
of~\citet{arora_just_2024} at 2K context over FDA, SWDE, SQuAD and
their \emph{-twice} variants (the \emph{-twice} variants repeat the
context before the query, exercising recurring key directions that give
$d_t$ multiple opportunities to calibrate). Commonsense
(PIQA, HellaSwag, WinoGrande, ARC-E/-C, SIQA, BoolQ, LAMBADA) and the
14-task LongBench English average serve as broader scope checks; PG-19
length extrapolation is in Appendix~\ref{sec:app_pg19_details}.

\subsection{Main results}
\label{subsec:exp_main}

Table~\ref{tab:main_results} is the main empirical summary. It combines
the matched 340M sweep with the 1.3B / 100B scale-up rows,
so the reader can compare the targeted retrieval effect against
language-modelling perplexity, long-context PG-19 perplexity, broader
short-context averages, LongBench, and inference throughput in one place.
We report a WikiText/LAMBADA perplexity geometric mean; JRT-style
in-context recall as overall, single-pass, and repeated-context averages;
the repeated-recall lift $\Delta_{\mathrm{rep}}$ relative to the matched
baseline; commonsense and LongBench averages; and the single-H100
tokens/sec change from the corresponding baseline. The theorem-facing residual-ratio
diagnostic $q_{\mathrm{geo}}$ is deliberately kept out of this summary
table because it is a DeltaNet replay measurement rather than a broad
benchmark axis; it is defined and reported in
Section~\ref{subsec:exp_mechanism}. Per-task breakdowns at both scales,
the visual summary, and the inference-throughput table are deferred to
Appendices~\ref{sec:app_general_checks},~\ref{sec:app_scale_1p3b},
and~\ref{sec:app_throughput}.

\paragraph{Reading the matched 340M sweep.}
The largest gain appears in the DeltaNet rows: vanilla \name{} raises
overall recall by 32\% over DeltaNet (0.150$\to$0.198,
with the gain concentrated on repeated context, +50\% relative), while
\apf{} retains 17\% (0.176) and gives the best DeltaNet-block
WikiText/LAMBADA GeoMean, while the refreshed no-APF screen improves
LongBench and also improves over DeltaNet on perplexity
(32.00$\to$31.67; \apf{} reaches 30.99). The same
pattern transfers to the GDN rows: OSGDN-APF lifts repeated recall
from 0.139 to 0.221 (+59\%) and improves WikiText/LAMBADA GeoMean and
LongBench. On the strongest broad baseline, KDA, OSKDA improves
single-pass recall and gives the strongest KDA-block PG-19 final
perplexity, while OSKDA-APF improves repeated recall from
0.150 to 0.179 and gives the strongest KDA-block commonsense and
LongBench averages.  KDA still has the best WikiText/LAMBADA
perplexity and retrieval LM-eval average, so we read \name{} as a
targeted retrieval-mechanism addition that applies across these baselines,
rather than a universal benchmark improver. Commonsense and LongBench
averages stay within a few hundredths of the matched baseline across all
three blocks; the refreshed no-APF OSDN screen further improves FW-Edu,
WikiText/LAMBADA, and the open-domain NQ/TriviaQA average
(Appendix~\ref{sec:app_general_checks}).

\subsection{Mechanism diagnostic}
\label{subsec:exp_mechanism}

The recall gains in the DeltaNet rows are accompanied by a direct
measurement of the quantity controlled by
Theorem~\ref{thm:main_tokenlocal}.
For each checkpoint we replay JRT-twice prompts, reconstruct the
write-side variables ($k_t$, $v_t$, $\beta_t$, and $d_t$ for online-scaled
models), run the fast-weight recurrence in fp32, and record
$q_t = f_t(S_t)/f_t(S_{t-1})$ token-by-token; $q_{\mathrm{geo}}$ averages
over 16 prompts per repeated-recall task (48 total), all 24 layers, and
all 8 heads -- $7.85\!\times\!10^6$ token-layer-head measurements.
Vanilla \name{} reduces $q_{\mathrm{geo}}$ from 0.537 to 0.433
(a 19\% reduction), \apf{} reaches 0.425 (21\%), and the reduction
holds task-wise on FDA-tw, SWDE-tw, and SQuAD-tw.
Figure~\ref{fig:mechanism_residual} resolves the contraction by
relative position: recurring associations in the second half of the
prompt contract more aggressively, exactly the regime the theory
targets.

\begin{figure}[t]
\centering
\includegraphics[width=0.8\linewidth]{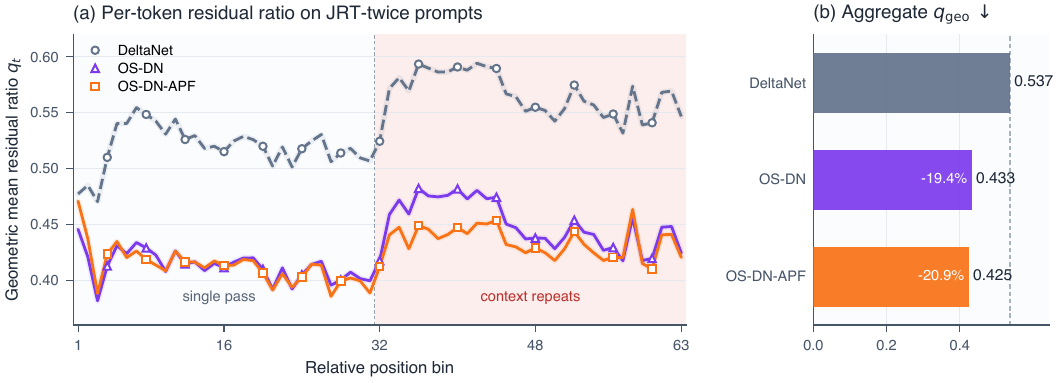}
\caption{\textbf{Direct theorem-facing residual contraction in the
DeltaNet rows.}
\textbf{(a)} Geometric mean of $q_t = f_t(S_t)/f_t(S_{t-1})$ by relative
position bin on JRT-twice prompts; the dashed boundary marks the
single-pass / repeated-context transition. \textbf{(b)} Overall
$q_{\mathrm{geo}}$ across $7.85\!\times\!10^6$ token-layer-head
measurements.}
\label{fig:mechanism_residual}
\vspace{-1em}
\end{figure}

\subsection{Scaling to 1.3B / 100B}
\label{subsec:exp_scale}

The matched DeltaNet vs.\ \apf{} pair at 1.3B parameters and 100B tokens
shows the mechanism scales cleanly: $q_{\mathrm{geo}}$ drops from 0.432
to \textbf{0.265}, a 39\% reduction that nearly doubles the 19--21\%
seen at 340M and is the lowest contraction recorded across either scale
(per-task numbers in Table~\ref{tab:mechanism_residual_ratio} of
Appendix~\ref{sec:app_general_checks}; SQuAD-tw alone drops from 0.473
to 0.102). Downstream averages stay at parity rather than regressing:
WikiText/LAMBADA GeoMean is essentially tied (14.22 vs.\ 14.28), with
LAMBADA improving from 11.84 to 10.98; single-pass FDA/SWDE recall
improves from 0.293 to 0.315; repeated recall is within noise; and the
commonsense and LongBench averages remain at parity with DeltaNet. The
mechanism-level signal therefore transfers and continues to amplify the
residual-ratio contraction at billion-parameter scale, with downstream
language-modelling and capability axes remaining at parity with the
matched DeltaNet baseline. Online preconditioning is also
lightweight at inference time: every OS-* variant lands within
$\pm 5.5\%$ of its matched baseline on tokens/sec at 340M, the 1.3B
\apf{} checkpoint runs 6.8\% slower than its DeltaNet baseline under the
same single-H100 generation benchmark, and the persistent recurrent state
grows by $\le 0.05\%$ from the OSGM diagonal (full table in
Appendix~\ref{sec:app_throughput}).

\section{Discussion}
\label{sec:discussion}

\name{} adds a learned diagonal preconditioner to DeltaNet's scalar
write step. Three properties of the exact-quadratic inner loss carry
the construction: a \emph{seamless integration} that preserves the
chunkwise WY pipeline under $K\mapsto\tilde K$
(Section~\ref{subsec:twophase}); a \emph{decoupling} that reduces
$\{d_t\}$ to an $\mathcal O(K)$-state affine recurrence; and the
collapse of the Hessian-Lipschitz residual that yields super-geometric
contraction guarantees (Theorems~\ref{thm:main_population}
and~\ref{thm:main_tokenlocal}). Among ``sequence layer as inner optimiser''
architectures, \name{} sits as a first-order, $\mathcal O(K)$-state
counterpart to MesaNet's exact $\arg\min$~\citep{vonoswald_mesanet_2025}
and is the in-recurrence analogue of the move from SGD to AdaGrad /
Adam~\citep{duchi_adaptive_2011,kingma_adam_2015}; APF is orthogonal
to memory-state decay, acting on the meta-optimiser state $d_t$ rather
than on $S_t$.

\paragraph{Limitations and future work.}
Six limitations bound the present results.  (i) Theorem~\ref{thm:main_population}
requires monotone iterates ($f(S_{t+1})\le f(S_t)$) and a regret bound against
the dense right-Newton comparator $D_\star=\Sigma_k^\dagger$; the $D_\star$
regret is not proved for the implemented diagonal update.  (ii)
Theorem~\ref{thm:main_tokenlocal} controls the geometric mean of token-local
residual ratios on the algorithm's own surrogate; lifting to a global
$f(S_T)-f^*$ statement requires either the no-conflict repeated-key regime
of Corollary~\ref{cor:repeated_keys} or the full-gradient population limit,
and the conditional-regret assumption sidesteps an explicit
step-size-specific regret derivation for Algorithm~\ref{alg:precond_online}'s
practical $\eta=0.003$ update.  (iii) Both bounds concern the inner regression
objective $f_t$ rather than next-token cross-entropy; APF, designed for
non-stationary contexts, would also need a dynamic-regret formulation.
(iv) The 1.3B / 100B sweep (Section~\ref{sec:exp},
Appendix~\ref{sec:app_scale_1p3b}) covers only \apf{} on DeltaNet, leaving
the cleanness of vanilla \name{} at billion-parameter scale, and the
composition with Gated DeltaNet / KDA at scale, open; SQuAD and SQuAD-twice
are excluded at this scale because the matched JRT \emph{contains}-accuracy
harness did not return reliable values, and PG-19 length-extrapolation
perplexity is not part of the 1.3B reporting protocol because the
20K-token sweep did not complete on a matched harness for both rows.
(v) The empirical headline is mechanism-level: the residual-ratio
contraction transfers and amplifies at billion-parameter scale, but the
downstream WikiText / LAMBADA, commonsense, and LongBench averages remain
at parity with DeltaNet rather than improving uniformly, and we therefore
do not claim a universal benchmark lift.  (vi) All matched 340M / 10B-token
runs share a fixed random seed, FineWeb-Edu shard ordering, and batch
schedule, so within-family deltas isolate the architectural change but do
not constitute seed-bootstrapped confidence intervals
(Appendix~\ref{sec:app_reproducibility}).  Two further directions follow:
scaling the residual-ratio diagnostic to broader prompts and longer
contexts, and lifting the diagonal preconditioner to a low-rank or
per-head block-diagonal form to close the gap to the full right-Newton
comparator $D_\star=\Sigma_k^\dagger$.

\bibliographystyle{plainnat}
\bibliography{main}

\clearpage
\appendix
\section*{Appendix Roadmap}
The appendix is organized to separate mechanism, implementation, theory,
evaluation protocol, additional evidence, and background.  Appendix~\ref{app:method_details}
collects derivations and proof details for the DeltaNet case.
Appendices~\ref{sec:extension}--\ref{app:chunk_impl} extend the same
write-key substitution to Gated DeltaNet and KDA and give the chunkwise
implementation pseudocode.  Appendix~\ref{sec:theory} states the quadratic
memory-regression guarantees.  Appendix~\ref{sec:app_benchmark_suite} defines
the evaluation suite, Appendix~\ref{sec:app_reproducibility} records the
training and evaluation protocol, and Appendices~\ref{sec:app_general_checks}--\ref{sec:app_scale_1p3b}
expand the 340M and 1.3B results.  Appendices~\ref{sec:app_throughput}--\ref{sec:app_related}
record inference-throughput measurements and the extended related-work taxonomy.

\section{Method Details and Proofs}
\label{app:method_details}

This appendix collects the chunkwise WY derivation, lemma and proposition
proofs, smoothness/monotone-descent corollaries, and the implementation
cost summary, all referenced from the main-text Section~\ref{sec:preconditioned_delta}.

\subsection{DeltaNet as online gradient descent (longer reading)}

DeltaNet's update is one step of stochastic gradient descent on the per-token
regression loss $f_t(S) = \tfrac12\|S k_t - v_t\|_F^2$, with
$\nabla f_t(S_{t-1}) = (S_{t-1}k_t - v_t)k_t^\top = -u_t k_t^\top$ and the
residual $u_t = v_t - S_{t-1}k_t \in \mathbb R^V$. The unit-step case
$\beta_t = 1$ with normalised keys $\|k_t\|_2^2 = 1$ exactly replaces the
value at $k_t$ ($A + (B-A) = B$); otherwise the readback shifts by
$\beta_t\|k_t\|_2^2(v_t - S_{t-1}k_t)$. Two equivalent readings -- Hebbian
with error correction, and online gradient descent on $f_t$ -- both lead
to the preconditioned Delta rule of the main text. Standard adaptive
optimisers (AdaGrad~\citep{duchi_adaptive_2011}, Adam~\citep{kingma_adam_2015})
likewise rescale gradient directions by a preconditioner; OSDN imports
this idea into the recurrent fast-weight write while retaining DeltaNet's
scalar gate $\beta_t$.

\subsection{Chunkwise WY derivation}
\label{app:chunkwise_wy}

We expand the chunk-level recurrence summarised in
Section~\ref{subsec:recurrent_chunk}. For a chunk of length $C$ with
intra-chunk index $i \in [1, C]$, stack the physical keys, preconditioned
keys, and values row-wise as $K_{[t]}, \osdnhi{\tilde K_{[t]}} \in \mathbb{R}^{C \times K}$
and $V_{[t]} \in \mathbb{R}^{C \times V}$, collect the scalar gates as
$\boldsymbol{\beta}_{[t]} \in \mathbb{R}^C$ with
$B_{[t]} = \mathrm{diag}(\boldsymbol{\beta}_{[t]})$, and write
$S_{[t]} \in \mathbb{R}^{V \times K}$ for the chunk-boundary state. Starting
from~\eqref{eq:recurrent_form}, define the lower-triangular UT-transform
matrix
\begin{equation}
    T_{[t]} \;=\; \bigl(I + \mathrm{tril}(B_{[t]} K_{[t]} \osdnhi{\tilde K_{[t]}^\top}, -1)\bigr)^{-1}
    \;\in\; \mathbb{R}^{C \times C},
\end{equation}
which can be obtained by forward substitution on the intra-chunk key--key
interactions. The cumulative-write matrices for keys and values are then
\begin{equation}
    W_{[t]} \;=\; T_{[t]} B_{[t]} K_{[t]} \;\in\; \mathbb{R}^{C \times K},
    \qquad
    U_{[t]} \;=\; T_{[t]} B_{[t]} V_{[t]} \;\in\; \mathbb{R}^{C \times V}.
\end{equation}
The chunk-internal cumulative transition admits the asymmetric WY form
\begin{equation}
    P_{[t]}
    \;=\; \prod_{i=1}^{C} \bigl(I - \beta_{[t]}^i k_{[t]}^i (\osdnhi{\tilde k_{[t]}^i})^\top\bigr)
    \;=\; I - W_{[t]}^\top \osdnhi{\tilde K_{[t]}} \;\in\; \mathbb{R}^{K\times K},
\end{equation}
and the cross-chunk state propagation reads
\begin{equation}
    S_{[t+1]} \;=\; S_{[t]}\bigl(I - W_{[t]}^\top \osdnhi{\tilde K_{[t]}}\bigr)
                + U_{[t]}^\top \osdnhi{\tilde K_{[t]}},
\end{equation}
or equivalently
$S_{[t+1]} = S_{[t]} + (U_{[t]} - W_{[t]} S_{[t]}^\top)^\top \osdnhi{\tilde K_{[t]}}$
after expanding $W_{[t]}^\top\osdnhi{\tilde K_{[t]}}$ and regrouping.
Compared with standard DeltaNet~\citep{yang_parallelizing_2024}, the
scalar gates occupy the same positions; the state-transition Gram changes
from $K_{[t]} K_{[t]}^\top$ to $K_{[t]}\osdnhi{\tilde K_{[t]}^\top}$, and the
intra-chunk output score uses the same write-side substitution
$Q_{[t]} K_{[t]}^\top \to Q_{[t]}\osdnhi{\tilde K_{[t]}^\top}$. The matrix
shapes and tensor-core GEMM layout are unchanged.

\subsection{Proof of Lemma~\ref{lem:hypergradient}}

By substituting the preconditioned update $S_t = S_{t-1} + \beta_t u_t (d_t \odot k_t)^\top$
into the residual equation we obtain
\begin{align*}
    S_t k_t - v_t &= (S_{t-1} + \beta_t u_t (d_t \odot k_t)^\top) k_t - v_t \\
    &= (S_{t-1} k_t - v_t) + \beta_t u_t \langle d_t \odot k_t, k_t \rangle \\
    &= -u_t (1 - \beta_t \langle d_t, k_t^2 \rangle).
\end{align*}
Hence $f_t(S_t) = \tfrac12\|u_t\|_2^2 (1 - \beta_t\langle d_t, k_t^2\rangle)^2$.
With $f_t(S_{t-1}) = \tfrac12\|u_t\|_2^2$ and
$\|\nabla f_t(S_{t-1})\|_F^2 = \|u_t\|_2^2 \|k_t\|_2^2$, substituting into
the definition $h_t(d_t) = [f_t(S_t)-f_t(S_{t-1})]/\|\nabla f_t(S_{t-1})\|_F^2$
yields the closed form. Differentiating once gives $\nabla_d h_t$;
differentiating again gives the rank-one Hessian
$\nabla^2 h_t(d) = \beta_t^2 k_t^2 (k_t^2)^\top / \|k_t\|_2^2$ used below.
\hfill$\square$

\subsection{Smoothness and monotone-descent corollaries}

\begin{proposition}[Smoothness of the hypergradient surrogate]
\label{prop:smoothness}
Under the normalisation $\|k_t\|_2^2 = 1$, the hypergradient feedback
$h_t(d)$ is $L_h$-smooth with
\begin{equation*}
    L_h \;=\; \|\nabla^2 h_t(d)\|_2 \;=\; \beta_t^2 \sum_{i=1}^K (k_t)_i^4 \;\le\; 1.
\end{equation*}
\end{proposition}

\begin{proof}
The Hessian computed in the proof of Lemma~\ref{lem:hypergradient} is
rank-one PSD with sole non-zero eigenvalue
$\lambda_{\max} = \beta_t^2 \|k_t^2\|_2^2 / \|k_t\|_2^2$. Under
$\|k_t\|_2^2 = 1$, $\lambda_{\max} = \beta_t^2 \sum_i (k_t)_i^4$. Since
$\beta_t \in (0,1)$, the power-mean inequality gives
$\sum_i (k_t)_i^4 \le \bigl(\sum_i (k_t)_i^2\bigr)^2 = 1$.
\end{proof}

The reported runs use the practical online step size $\eta=0.003$ inside
$\mathcal D=[0.5,2.0]^K$, with reproduction details in
Appendix~\ref{sec:app_reproducibility}.

\begin{corollary}[Monotone descent under bounded box]
\label{cor:monotone_descent}
Assume $\|k_t\|_2^2 = 1$, $\beta_t \in (0,1)$, and $d_t \in \mathcal D = [d_{\min}, d_{\max}]^K$
with $0 < d_{\min} \le d_{\max} \le 2$. Then $0 < \beta_t \langle d_t, k_t^2\rangle < 2$ and
\[
    \frac{f_t(S_t)}{f_t(S_{t-1})} \;=\; \bigl(1 - \beta_t\langle d_t, k_t^2\rangle\bigr)^2 \;<\; 1.
\]
\end{corollary}

\begin{proof}
$\|k_t^2\|_1 = \|k_t\|_2^2 = 1$ and $d_{t,j}\in[d_{\min},d_{\max}]$ give
$\langle d_t, k_t^2\rangle \in [d_{\min}, d_{\max}]\subseteq(0,2]$. With
$\beta_t \in (0,1)$ strict, $\beta_t\langle d_t, k_t^2\rangle \le 2\beta_t < 2$
and $\beta_t \langle d_t, k_t^2\rangle \ge \beta_t d_{\min} > 0$, so
$|1-\beta_t\langle d_t, k_t^2\rangle| < 1$. The ratio identity follows
from the residual computation in the proof of Lemma~\ref{lem:hypergradient}.
\end{proof}

\subsection{APF preserves the affine two-phase scan}

\begin{proposition}[APF is a piecewise-affine recurrence]
\label{prop:apf_affine}
For fixed $k_t$, $\beta_t$, and $r_{t,h}$, the unconstrained APF step is
affine in $d_t$,
\begin{equation*}
    \bar d_{t+1}
    \;=\; \Bigl(r_{t,h}I - \tfrac{\eta\beta_t^2}{\max(\|k_t\|_2^2,\epsilon)}\,k_t^2 (k_t^2)^\top\Bigr) d_t
    + \tfrac{\eta\beta_t}{\max(\|k_t\|_2^2,\epsilon)}\,k_t^2,
\end{equation*}
and the realised update is the coordinate-wise projection
$d_{t+1} = \Pi_{\mathcal D}(\bar d_{t+1})$. The vector-retention extension
replaces $r_{t,h}I$ by $\mathrm{diag}(\boldsymbol r_t)$. In both forms
$\boldsymbol r_t$ is independent of the high-dimensional state $S_t$,
value $v_t$, and residual $u_t$, so the two-phase scan of
Algorithm~\ref{alg:precond_online} is preserved.
\end{proposition}

\begin{proof}
Direct expansion of~\eqref{eq:dt_update_apf} from the main text. The
coordinate-wise clamp realising $\Pi_{\mathcal D}$ does not introduce
state coupling.
\end{proof}

\paragraph{Interpretation.}
The unprojected APF step can be read as a gradient step on the
linearisation of $h_t$ around $d_t$ with proximal centre shifted from
$d_t$ to $\boldsymbol r_t \odot d_t$. Equivalently,
$\bar d_{t+1} = \arg\min_D\{\langle\nabla h_t(d_t), D\rangle + \tfrac1{2\eta}\|D - \boldsymbol r_t \odot d_t\|^2\}$.
The implicit step on the unlinearised surrogate would solve a rank-one
Sherman--Morrison system at each token; the explicit affine-then-project
recurrence avoids that solve while preserving the proximal-centre
interpretation. The super-geometric guarantees in
Appendix~\ref{sec:theory} apply to the unforgotten OSDN update
($\boldsymbol r_t \equiv \mathbf 1$); APF is justified empirically by
preserving the projected affine scan and improving long-context
stability.

\subsection{Implementation cost summary}
\label{app:hw_cost}

Table~\ref{tab:hw_cost} expands the per-layer overhead claim from the
main text. The baseline kernel retains DeltaNet's matrix layout; the
practical cost concentrates in the phase-1 preconditioner stream.

\begin{table}[H]
\centering
\footnotesize
\setlength{\tabcolsep}{4pt}
\renewcommand{\arraystretch}{1.18}
\caption{\textbf{Per-layer implementation overhead relative to DeltaNet.}
The main pass keeps DeltaNet's matrix shapes; OSDN only materialises a
write-side scaled key sequence before that pass.}
\label{tab:hw_cost}
\begin{tabular}{@{}lccc@{}}
\toprule
\textbf{Cost item} & \textbf{DeltaNet} & \textbf{\name{}} & \textbf{\apf} \\
\midrule
\multicolumn{4}{@{}l}{\textit{Persistent state and parameters}} \\
Recurrent state & $V K$ & $V K + K$ & $V K + K$ \\
Parameters / layer & -- & $0$ & $H(d_\text{m}+1)$ \\
\addlinespace[2pt]
\multicolumn{4}{@{}l}{\textit{Phase-1 preconditioner stream}} \\
Work / token & -- & $\mathcal{O}(K)$ & $\mathcal{O}(K+d_\text{m})$/head \\
Sequential depth & -- & $L$ & $L$ \\
\addlinespace[2pt]
\multicolumn{4}{@{}l}{\textit{DeltaNet pass}} \\
Write key & $K^\top$ & $\tilde K^\top$ & $\tilde K^\top$ \\
GEMM layout & DeltaNet & unchanged & unchanged \\
\bottomrule
\end{tabular}
\end{table}

The phase-1 stream updates $d_t$ and materialises $\tilde k_t = d_t\odot k_t$
before phase~2; the write-key substitution affects only the write-side
key in the chunk Gram and the intra-chunk score. The full implementation
diff and the kernel-level pseudocode appear in
Appendix~\ref{app:chunk_impl}.

\section{Backbone Extensions: Gated DeltaNet and KDA}
\label{sec:extension}

This appendix extends online scaling beyond the DeltaNet backbone studied in the main experiments. We define a scaled key \(\tilde{\mathbf{k}}_t = \mathbf{d}_t \odot \mathbf{k}_t\), where \(\mathbf{d}_t\) is a diagonal preconditioner vector. As shown in Table~\ref{tab:model_recurrence}, DeltaNet and Gated DeltaNet use the \(V \times K\) state convention, so the write-side substitution appears as \(\mathbf{k}_t^\top \mapsto \tilde{\mathbf{k}}_t^\top\). KDA uses the transposed \(K \times V\) convention, so the same storage-side preconditioning appears on the left write key, \(\mathbf{k}_t \mapsto \tilde{\mathbf{k}}_t\), while the residual read key remains \(\mathbf{k}_t\). The recurrent equations and chunkwise parallel forms retain their original matrix shapes; only the Gram entries involving key--key interactions become asymmetric.

For all three backbones, the implementation treats the preconditioner as a lightweight phase-1 recurrent state. We write \(\rho_t^d\in(0,1]\) for an optional retention applied to \(d_t\) itself, distinct from the recurrent-state gate of the layer. The beta-aware phase-1 update used by the OSGDN and OSKDA kernels is
\begin{equation}
    d_{t+1}
    =
    \rho_t^d d_t
    + \eta \beta_t
      \frac{1-\beta_t\langle d_t,k_t^2\rangle}{\max(\|k_t\|_2^2,\epsilon)}
      k_t^2.
    \label{eq:extension_beta_aware_phase1}
\end{equation}
The DeltaNet derivation in Section~\ref{subsec:hypergradient} corresponds to \(\rho_t^d=1\), while APF uses a data-dependent \(\rho_t^d\). The 1.3B OSDN configuration uses a non-beta-aware phase-1 recurrence \(d_{t+1}=\rho_t^d d_t+\eta(1-\langle d_t,k_t^2\rangle)k_t^2\); the DeltaNet update still retains the scalar gate \(\beta_t\), and the chunkwise substitution below is unchanged.

\subsection{Gated DeltaNet with Online Scaling}

Starting from the Gated Delta recurrence in Table~\ref{tab:model_recurrence},

\[
\mathbf{S}_t = \mathbf{S}_{t-1}\bigl(\alpha_t(\mathbf{I}-\beta_t\mathbf{k}_t\mathbf{k}_t^\top)\bigr) + \beta_t\mathbf{v}_t\mathbf{k}_t^\top,
\]

we replace every occurrence of \(\mathbf{k}_t^\top\) in the write-side rank-1 factor by \(\tilde{\mathbf{k}}_t^\top\). The left factor \(\mathbf{k}_t\), which determines the residual direction, remains unchanged:

\[
\mathbf{S}_t = \mathbf{S}_{t-1}\bigl(\alpha_t(\mathbf{I}-\beta_t\mathbf{k}_t\tilde{\mathbf{k}}_t^\top)\bigr) + \beta_t\mathbf{v}_t\tilde{\mathbf{k}}_t^\top.
\]

The transition matrix remains \(\alpha_t\) times an identity-minus-rank-1 map, with the symmetric factor \(\mathbf{k}_t\mathbf{k}_t^\top\) replaced by the asymmetric factor \(\mathbf{k}_t\tilde{\mathbf{k}}_t^\top\). The scalar gates and value vector are unchanged.

The OSGDN implementation uses the \emph{post-gate regret} form of the phase-1 update. Define the decayed reference state
\[
    \bar{\mathbf S}_{t-1}=\alpha_t\mathbf S_{t-1},
    \qquad
    \mathbf e_t=\mathbf v_t-\bar{\mathbf S}_{t-1}\mathbf k_t .
\]
The write is then \(\mathbf S_t=\bar{\mathbf S}_{t-1}+\beta_t \mathbf e_t\tilde{\mathbf k}_t^\top\), so the residual against the same post-gate reference contracts as \((1-\beta_t\langle d_t,k_t^2\rangle)\mathbf e_t\). This yields exactly the beta-aware direction in Equation~\eqref{eq:extension_beta_aware_phase1}. Using the pre-gate residual would introduce an additive \((1-\alpha_t)\mathbf v_t\) term that is not controlled by \(d_t\); this is why the OSGDN implementation computes the hypergradient feedback after applying the GDN state gate.

For the chunkwise parallel algorithm, let $K_{[t]}, \tilde{K}_{[t]} \in \mathbb{R}^{C\times K}$ stack the intra-chunk keys and preconditioned keys, $V_{[t]} \in \mathbb{R}^{C\times V}$ the values, and $\beta_{[t]} \in \mathbb{R}^C$ the scalar step sizes. Define the cumulative gating prefix product $\gamma_i = \prod_{l=1}^i \alpha_{[t]}^l$, collect these scalars in $\boldsymbol{\gamma}_{[t]} \in \mathbb{R}^C$, and let $\gamma_C = \gamma_{[t]}^C$ be the chunk-level decay scalar~\citep{yang_gated_2024-1}. We use the standard Gated DeltaNet decay-ratio gauge, in which the chunk Gram keeps the factor $\gamma_i/\gamma_j$:
\begin{equation}
    M_{[t]} =
    \Bigl(I + \mathrm{strictLower}\Bigl(
    \mathrm{diag}(\beta_{[t]})\,
    (\boldsymbol{\gamma}_{[t]} \odot K_{[t]})
    (\tilde K_{[t]} / \boldsymbol{\gamma}_{[t]})^\top
    \Bigr)\Bigr)^{-1}\mathrm{diag}(\beta_{[t]})
    \in \mathbb{R}^{C \times C},
\end{equation}
where row-wise multiplication and division by $\boldsymbol{\gamma}_{[t]}$ give
\[
    (A_{[t]})_{ij}
    =
    \mathbf{1}_{j<i}\,
    \beta_i \frac{\gamma_i}{\gamma_j}
    \langle k_i,\tilde k_j\rangle .
\]
This generalises the symmetric decay-ratio Gram of the original Gated DeltaNet to the asymmetric Gram involving $K$ and $\tilde K$. The cumulative key/value matrices are
\begin{equation}
    W_{[t]} = M_{[t]}(\boldsymbol{\gamma}_{[t]} \odot K_{[t]}) \in \mathbb{R}^{C \times K}, \qquad
    U_{[t]} = M_{[t]} V_{[t]} \in \mathbb{R}^{C \times V}.
\end{equation}
Writing $\boldsymbol{\rho}_{[t]} \in \mathbb{R}^C$ for the row-wise suffix decay with $(\boldsymbol{\rho}_{[t]})_i = \gamma_C / \gamma_i$, the chunk state propagation is
\begin{equation}
    S_{[t+1]}
    = \gamma_C S_{[t]}
    + \bigl(U_{[t]} - W_{[t]} S_{[t]}^\top\bigr)^\top
    (\boldsymbol{\rho}_{[t]} \odot \tilde K_{[t]}).
\end{equation}
The matrix multiplications and chunk-level decay are the same as in Gated DeltaNet after replacing the storage/write-side Gram factor by $\tilde K$. The additional computation is the elementwise scaling $\tilde{k}_t = d_t \odot k_t$; the asymptotic complexity and tensor-core GEMM shapes are unchanged. An equivalent value-gauge can move the same decay ratios into $V_{[t]}$ and multiply back by $\gamma_C$ at the chunk boundary; the form above matches the standard Gated DeltaNet kernel notation.

\subsection{Kimi Delta Attention with Online Scaling}

KDA stores its state as $\mathbf{S}_t \in \mathbb{R}^{d_k \times d_v}$ and writes values through the rank-1 term $\beta_t \mathbf{k}_t \mathbf{v}_t^\top$~\citep{kimiteam2025kimilinearexpressiveefficient}. Its recurrence is
\[
\mathbf{S}_t = (\mathbf{I} - \beta_t \mathbf{k}_t \mathbf{k}_t^\top) \operatorname{Diag}(\boldsymbol{\alpha}_t) \mathbf{S}_{t-1} + \beta_t \mathbf{k}_t \mathbf{v}_t^\top \in \mathbb{R}^{d_k \times d_v},\qquad \mathbf{o}_t = \mathbf{S}_t^\top \mathbf{q}_t.
\]
Since KDA uses the transposed convention, the diagonal key preconditioner acts on the left storage side of the gradient. Thus Online Scaling replaces the storage/write key $\mathbf{k}_t$ by $\tilde{\mathbf{k}}_t = \mathbf{d}_t \odot \mathbf{k}_t$, while the residual read key remains $\mathbf{k}_t$:
\[
\mathbf{S}_t = (\mathbf{I} - \beta_t \tilde{\mathbf{k}}_t \mathbf{k}_t^\top) \operatorname{Diag}(\boldsymbol{\alpha}_t) \mathbf{S}_{t-1} + \beta_t \tilde{\mathbf{k}}_t \mathbf{v}_t^\top.
\]
Equivalently, with $\bar{\mathbf{S}}_{t-1}=\operatorname{Diag}(\boldsymbol{\alpha}_t)\mathbf{S}_{t-1}$,
\[
    \mathbf{u}_t = \mathbf{v}_t - \bar{\mathbf{S}}_{t-1}^{\top}\mathbf{k}_t,\qquad
    \mathbf{S}_t = \bar{\mathbf{S}}_{t-1} + \beta_t \tilde{\mathbf{k}}_t \mathbf{u}_t^\top .
\]
The transition matrix remains an identity-minus-rank-1 map followed by a fine-grained diagonal gate, with asymmetric rank-1 factor $\tilde{\mathbf{k}}_t \mathbf{k}_t^\top$. The channel gate, scalar gate, value vector, and residual read key are unchanged.

The OSKDA phase-1 state follows Equation~\eqref{eq:extension_beta_aware_phase1} with KDA's key-normalized \(k_t\). In the no-DD matched comparison, \(\rho_t^d=1\) and \(d_t\) is clamped to the matched feasible box. The OSKDA-APF variant uses a data-dependent retention \(\rho_t^d\) for the preconditioner state, with the same matched online step, beta-aware update, and clamp; specific values follow Appendix~\ref{sec:app_reproducibility}. This retention acts only on \(d_t\). The KDA state gate \(\operatorname{Diag}(\boldsymbol{\alpha}_t)\) still gates the high-dimensional \(K\times V\) state, and the residual read key remains the unscaled \(k_t\).

For hardware-efficient training, the storage-side substitution is reflected in KDA's chunkwise UT-transform. Within a chunk of size $C$, let $\mathbf{K}_{[t]}, \tilde{\mathbf{K}}_{[t]} \in \mathbb{R}^{C \times d_k}$ stack the keys and preconditioned keys, $\mathbf{V}_{[t]} \in \mathbb{R}^{C \times d_v}$ the values, $\beta_{[t]} \in \mathbb{R}^C$ the scalar step sizes, and let $\boldsymbol{\Gamma}_{[t]} \in \mathbb{R}^{C \times d_k}$ collect the cumulative channel-wise decay factors with $(\boldsymbol{\Gamma}_{[t]})_{i,j} = \prod_{l=1}^{i} (\boldsymbol{\alpha}_{[t]}^l)_j$. Let $\boldsymbol{\gamma}_{[t]}^C \in \mathbb{R}^{d_k}$ be the chunk-level fine-grained decay vector and let $\boldsymbol{\Gamma}_{[t]}^{i \to C}\in\mathbb{R}^{C\times d_k}$ collect the suffix factors $(\boldsymbol{\Gamma}_{[t]}^{i\to C})_{i,j}=(\boldsymbol{\gamma}_{[t]}^C)_j/(\boldsymbol{\Gamma}_{[t]})_{i,j}$. The residual/read side uses the original keys in $W_{[t]}$, while the storage/write side uses $\tilde{\mathbf{K}}_{[t]}$:
\[
\mathbf{M}_{[t]} = \!\Bigl( \mathbf{I} + \operatorname{StrictTril}\Bigl( \operatorname{Diag}(\beta_{[t]}) \, \bigl(\boldsymbol{\Gamma}_{[t]} \odot \mathbf{K}_{[t]}\bigr) \bigl( \tilde{\mathbf{K}}_{[t]}\,/\,\boldsymbol{\Gamma}_{[t]}\bigr)^\top \Bigr)\Bigr)^{\!-1} \operatorname{Diag}(\beta_{[t]}) \in \mathbb{R}^{C \times C}.
\]
The cumulative key/value matrices and chunk state propagation are then
\[
\mathbf{W}_{[t]} = \mathbf{M}_{[t]} \bigl( \boldsymbol{\Gamma}_{[t]} \odot \mathbf{K}_{[t]} \bigr) \in \mathbb{R}^{C \times d_k}, \qquad \mathbf{U}_{[t]} = \mathbf{M}_{[t]} \mathbf{V}_{[t]} \in \mathbb{R}^{C \times d_v},
\]
\[
\mathbf{S}_{[t+1]} = \operatorname{Diag}(\boldsymbol{\gamma}_{[t]}^C) \mathbf{S}_{[t]} + \bigl( \boldsymbol{\Gamma}_{[t]}^{i \to C} \odot \tilde{\mathbf{K}}_{[t]} \bigr)^\top \bigl( \mathbf{U}_{[t]} - \mathbf{W}_{[t]} \mathbf{S}_{[t]} \bigr) \in \mathbb{R}^{d_k \times d_v}.
\]
The intra-chunk output formula uses the same storage-side replacement:
\[
\mathbf{O}_{[t]}
=
\bigl(\boldsymbol{\Gamma}_{[t]} \odot \mathbf{Q}_{[t]}\bigr)\mathbf{S}_{[t]}
+
\operatorname{Tril}\Bigl(
\bigl(\boldsymbol{\Gamma}_{[t]} \odot \mathbf{Q}_{[t]}\bigr)
\bigl(\tilde{\mathbf{K}}_{[t]} / \boldsymbol{\Gamma}_{[t]}\bigr)^\top
\Bigr)
\bigl(\mathbf{U}_{[t]}-\mathbf{W}_{[t]}\mathbf{S}_{[t]}\bigr).
\]
The additional work is the elementwise scaling $\tilde{\mathbf{k}}_t = \mathbf{d}_t \odot \mathbf{k}_t$; linear complexity, tensor-core GEMM shapes, and fine-grained channel gating are preserved.

\begin{table}[htbp]
\centering
\footnotesize
\setlength{\tabcolsep}{4pt}
\renewcommand{\arraystretch}{1.18}
\caption{Recurrence and read-out across linear recurrent models. Rows highlighted in gray are our online-scaled variants of DeltaNet, Gated DeltaNet, and KDA. Unless otherwise indicated, $S_t \in \mathbb{R}^{V \times K}$ and reads use $o_t = S_t q_t$. KDA and OSKDA follow KDA's native convention $S_t \in \mathbb{R}^{K \times V}$ with read-out $o_t = S_t^\top q_t$; under this convention Online Scaling modifies the left storage/write key.}
\label{tab:model_recurrence}
\begin{tabularx}{\linewidth}{@{}
  >{\hsize=0.55\hsize\raggedright\arraybackslash}X
  >{\hsize=1.55\hsize\raggedright\arraybackslash}X
  >{\hsize=0.90\hsize\raggedright\arraybackslash}X
  @{}}
\toprule
\textbf{Model} & \textbf{Recurrence (state update)} & \textbf{Read-out (output)} \\
\midrule
Linear~Attention~\citep{kasai_finetuning_2021,katharopoulos_transformers_2020}
  & $S_t = S_{t-1} + v_t k_t^\top$
  & $o_t = S_t q_t$ \\
\quad +\,Kernel
  & $S_t = S_{t-1} + v_t \phi(k_t)^\top$
  & $o_t = S_t \phi(q_t)$ \\
\quad +\,Normalised
  & $S_t = S_{t-1} + v_t \phi(k_t)^\top,\ z_t = z_{t-1} + \phi(k_t)$
  & $o_t = S_t \phi(q_t) / (z_t^\top \phi(q_t))$ \\
Gated~RFA~\citep{peng_random_2020}
  & $S_t = g_t S_{t-1} + (1-g_t) v_t k_t^\top,\ z_t = g_t z_{t-1} + (1-g_t) k_t$
  & $o_t = S_t q_t / (z_t^\top q_t)$ \\
S4~\citep{gu_efficiently_2021,smith_simplified_2022}
  & $S_t = S_{t-1} \odot \exp\!\bigl(-(\alpha \mathbf{1}^\top) \odot \exp(A)\bigr) + B \odot (v_t \mathbf{1}^\top)$
  & $o_t = (S_t \odot C)\mathbf{1} + d \odot v_t$ \\
ABC~\citep{peng_abc_2022}
  & $S_t^k = S_{t-1}^k + k_t \phi_t^\top,\ S_t^v = S_{t-1}^v + v_t \phi_t^\top$
  & $o_t = S_t^v \operatorname{softmax}(S_t^k q_t)$ \\
DFW~\citep{mao_fine-tuning_2022}
  & $S_t = S_{t-1} \odot (\beta_t \alpha_t^\top) + v_t k_t^\top$
  & $o_t = S_t q_t$ \\
RetNet~\citep{Sun2023RetentiveNA}
  & $S_t = \gamma S_{t-1} + v_t k_t^\top$
  & $o_t = S_t q_t$ \\
Mamba~\citep{gu_mamba_2024}
  & $S_t = S_{t-1} \odot \exp\!\bigl(-(\alpha_t \mathbf{1}^\top) \odot \exp(A)\bigr) + (\alpha_t \odot v_t) k_t^\top$
  & $o_t = S_t q_t + d \odot v_t$ \\
GLA~\citep{yang_gated_2024}
  & $S_t = S_{t-1} \operatorname{Diag}(\alpha_t) + v_t k_t^\top$
  & $o_t = S_t q_t$ \\
RWKV-6~\citep{peng_eagle_2024}
  & $S_t = S_{t-1} \operatorname{Diag}(\alpha_t) + v_t k_t^\top$
  & $o_t = \bigl(S_{t-1} + (d \odot v_t) k_t^\top\bigr) q_t$ \\
HGRN-2~\citep{qin_hgrn2_2024}
  & $S_t = S_{t-1} \operatorname{Diag}(\alpha_t) + v_t (1-\alpha_t)^\top$
  & $o_t = S_t q_t$ \\
mLSTM~\citep{beck_xlstm_2024}
  & $S_t = f_t S_{t-1} + i_t v_t k_t^\top,\ z_t = f_t z_{t-1} + i_t k_t$
  & $o_t = S_t q_t / \max\{1,|z_t^\top q_t|\}$ \\
Mamba-2~\citep{dao_transformers_2024}
  & $S_t = \gamma_t S_{t-1} + v_t k_t^\top$
  & $o_t = S_t q_t$ \\
GSA~\citep{zhang_gated_2024}
  & $S_t^k = S_{t-1}^k \operatorname{Diag}(\alpha_t) + k_t \phi_t^\top,\ S_t^v = S_{t-1}^v \operatorname{Diag}(\alpha_t) + v_t \phi_t^\top$
  & $o_t = S_t^v \operatorname{softmax}(S_t^k q_t)$ \\
\midrule
DeltaNet~\citep{schlag_linear_2021}
  & $S_t = S_{t-1}(I - \beta_t k_t k_t^\top) + \beta_t v_t k_t^\top$
  & $o_t = S_t q_t$ \\
\osdnrow
\quad +\,Online Scaling
  & $S_t = S_{t-1}(I - \beta_t k_t \tilde{k}_t^\top) + \beta_t v_t \tilde{k}_t^\top$
  & $o_t = S_t q_t$ \\
Gated~DeltaNet~\citep{yang_gated_2024-1}
  & $S_t = S_{t-1}\bigl(\alpha_t (I - \beta_t k_t k_t^\top)\bigr) + \beta_t v_t k_t^\top$
  & $o_t = S_t q_t$ \\
\osdnrow
\quad +\,Online Scaling
  & $S_t = S_{t-1}\bigl(\alpha_t (I - \beta_t k_t \tilde{k}_t^\top)\bigr) + \beta_t v_t \tilde{k}_t^\top$
  & $o_t = S_t q_t$ \\
KDA~\citep{kimiteam2025kimilinearexpressiveefficient}
  & $S_t = (I - \beta_t k_t k_t^\top) \operatorname{Diag}(\alpha_t) S_{t-1} + \beta_t k_t v_t^\top$
  & $o_t = S_t^\top q_t$ \\
\osdnrow
\quad +\,Online Scaling
  & $S_t = (I - \beta_t \tilde{k}_t k_t^\top) \operatorname{Diag}(\alpha_t) S_{t-1} + \beta_t \tilde{k}_t v_t^\top$
  & $o_t = S_t^\top q_t$ \\
\bottomrule
\end{tabularx}
\end{table}

\section{Chunkwise Implementation Pseudocode}
\label{app:chunk_impl}

The implementations use a common two-phase schedule. Phase~1 materializes
the preconditioner trajectory \(d_t\) and the write-side key
\(\tilde k_t = d_t \odot k_t\). Phase~2 calls the baseline chunk rule with the
same state layout, gates, and value tensors as the baseline layer,
replacing only the storage/write key in the chunk Gram, state update, and
local output score. Listing~\ref{lst:chunk_osdn} factors out the phase-1
recurrence shared by all variants. Inputs are ordered as
\([B,T,H,\cdot]\) before chunking and \([B,H,N,C,\cdot]\) inside each
chunk. The code is written for the \(V\times K\) DeltaNet convention; KDA
uses the transposed \(K\times V\) convention described below. Each variant
is presented next to its baseline kernel:
Listing~\ref{lst:chunk_osdn} doubles as the OSDN reference,
while Listings~\ref{lst:chunk_osgdn} and~\ref{lst:chunk_oskda} give the
OSGDN and OSKDA baseline updates once \(\tilde K\) has been materialized.

\providecolor{OSDNDiffNew}{RGB}{200,30,70}
\providecolor{OSDNDiffChg}{RGB}{30,80,200}
\providecommand{\diffNew}[1]{\textcolor{OSDNDiffNew}{#1}}
\providecommand{\diffChg}[1]{\textcolor{OSDNDiffChg}{#1}}
\providecommand{\tagDiffNew}{\diffNew{\textsc{New}}}
\providecommand{\tagDiffChg}{\diffChg{\textsc{Chg}}}

\paragraph{Recurrent diff against DeltaNet.}
Algorithm~\ref{alg:hw_compare} expands the main-text substitution into an explicit two-phase recurrent diff.
The phase-1 lines materialize the write-side key sequence and update the lightweight preconditioner state; the phase-2 DeltaNet pass is DeltaNet with the storage/write key changed from \(k_t\) to \(\tilde k_t\).

\begin{algorithm}[H]
\caption{\textbf{Line-by-line recurrent diff between DeltaNet and \name{}.}
Lines marked \tagDiffNew{} belong to the phase-1 preconditioner sweep; the line marked \tagDiffChg{} is the single structural write-key substitution.}
\label{alg:hw_compare}
\begin{algorithmic}[1]
\vspace{1pt}
\Statex \textit{(1) DeltaNet~\citep{yang_parallelizing_2024} baseline.}
\For{$t = 1$ to $L$}
    \State $u_t \gets v_t - S\, k_t$ \Comment{residual / read-then-write}
    \State $S \gets S + \beta_t\, u_t\, k_t^\top$ \Comment{rank-one write}
\EndFor
\vspace{4pt}
\Statex \diffNew{\textit{(2a) \name{} phase 1: materialize write keys.}}
\State $d \gets d^{(0)} \in \mathcal{D}$ \Comment{preconditioner state, $\mathcal{D}=[d_{\min},d_{\max}]^K$}
\For{$t = 1$ to $L$}
    \State \diffNew{$\tilde k_t \gets d \odot k_t$} \hfill {\color{OSDNDiffNew}\(\triangleright\)} \tagDiffNew{}: write key
    \State $n_t \gets \max(\|k_t\|_2^2,\epsilon)$ \Comment{squared key-norm normalizer}
    \State $\mathrm{step}_t \gets \eta\beta_t\,\dfrac{1 - \beta_t\langle d,\, k_t^2 \rangle}{n_t}$ \Comment{scalar hypergradient coefficient}
    \State \diffNew{$d \gets \mathrm{clip}(d + \mathrm{step}_t\, k_t^2,\,d_{\min},\,d_{\max})$} \hfill {\color{OSDNDiffNew}\(\triangleright\)} \tagDiffNew{}: projected hypergradient
\EndFor
\vspace{4pt}
\Statex \diffChg{\textit{(2b) \name{} phase 2: same DeltaNet pass, write-side substitution.}}
\For{$t = 1$ to $L$}
    \State $u_t \gets v_t - S\, k_t$ \Comment{\emph{unchanged} read}
    \State \diffChg{$S \gets S + \beta_t\, u_t\, \tilde k_t^\top$} \hfill {\color{OSDNDiffChg}\(\triangleright\)} \tagDiffChg{}: storage/write key
\EndFor
\end{algorithmic}
\end{algorithm}

\paragraph{OSDN implementation.}
The DeltaNet path is exactly the substitution shown in
Listing~\ref{lst:chunk_osdn}: the read key in the residual remains \(K\),
while the write-side key becomes \(\tilde K\). The OSDN kernels compute
\(d_t\) before phase~2, store \(\tilde K\), and then reuse the
standard chunked DeltaNet WY structure with \(K\tilde K^\top\) in the UT
Gram and \(Q\tilde K^\top\) in the local output score. After each
hypergradient step the preconditioner is projected onto
\([0.5,2.0]^K\); the 340M matched runs reported in
Section~\ref{sec:exp} pass \texttt{eta=0.003}, \texttt{d\_min=0.5}, and
\texttt{d\_max=2.0} as arguments in the listing and apply the projection
via the per-step \texttt{clamp}. The 1.3B OSDN
configuration uses a learned initial scale that is zero-initialized,
data-dependent preconditioner retention
\(\rho_t^d=\sigma(a_{\mathrm{osgm}}(x_t))\) initialized near \(0.999\),
and the non-beta-aware phase-1 form; OSGDN and OSKDA use the beta-aware
recurrence in Listing~\ref{lst:chunk_osdn}.

\begin{nolinenumbers}
\begin{lstlisting}[
style=pytorch,
language=Python,
caption={Chunked phase-1 preconditioner sweep paired with the DeltaNet write-key substitution; reused verbatim by \name{}.},
label={lst:chunk_osdn}
]
def chunk_osdn(q, k, v, beta, precond_retention=None, initial_state=None,
               initial_d=None, chunk_size=64, eta=0.003,
               beta_aware=True, d_min=0.5, d_max=2.0):
    dtype = v.dtype
    B, T, H, K, V, C = *q.shape, v.shape[-1], chunk_size
    N = T // C

    q, k, v, beta = map(
        lambda x: rearrange(x, 'b (n c) h ... -> b h n c ...', c=C).float(),
        [q, k, v, beta]
    )
    q = q * K ** -0.5
    epsilon = 1e-6
    eye = torch.eye(C, device=q.device)
    upper_including_diag = torch.triu(
        torch.ones(C, C, device=q.device, dtype=torch.bool), 0
    )
    strict_upper = torch.triu(
        torch.ones(C, C, device=q.device, dtype=torch.bool), 1
    )

    # Phase 1: build the preconditioner trajectory d_t and write_key_t.
    d = q.new_ones(B, H, K) if initial_d is None else initial_d.float()
    write_key = torch.empty_like(k)
    for n in range(N):
        for i in range(C):
            k_i, beta_i = k[:, :, n, i], beta[:, :, n, i]
            write_key[:, :, n, i] = d * k_i        # store d_t * k_t
            k2 = k_i * k_i
            retention_i = (
                torch.ones_like(d)
                if precond_retention is None
                else precond_retention[:, n*C+i]
            )
            if retention_i.ndim == 2:
                retention_i = retention_i[..., None]
            key_norm_sq = k2.sum(-1).clamp_min(epsilon)

            # beta-aware kernels include beta_i in the phase-1 feedback;
            # the non-beta-aware OSDN configuration sets beta_phase1 = 1.
            beta_phase1 = beta_i if beta_aware else torch.ones_like(beta_i)
            precond_alignment = (d * k2).sum(-1)
            precond_step = (
                eta * beta_phase1 * (1.0 - beta_phase1 * precond_alignment)
                / key_norm_sq
            )
            d = retention_i * d + precond_step[..., None] * k2
            if d_min is not None and d_max is not None:
                d = d.clamp(d_min, d_max)        # box projection to D

    # Phase 2: chunked DeltaNet pass with the storage key replaced by write_key.
    S = k.new_zeros(B, H, V, K)
    if initial_state is not None:
        S += initial_state.float()
    output = torch.zeros(B, H, N, C, V, device=v.device)

    for n in range(N):
        q_i, k_i = q[:, :, n], k[:, :, n]
        write_key_i = write_key[:, :, n]
        v_i, beta_i = v[:, :, n], beta[:, :, n]

        strict_lower_gram = torch.einsum(
            '... i k, ... j k -> ... i j', k_i, write_key_i
        )
        strict_lower_gram = (
            strict_lower_gram * beta_i[..., :, None]
        ).masked_fill(upper_including_diag, 0)
        ut_inverse = torch.linalg.solve_triangular(
            eye + strict_lower_gram, eye, upper=False
        )

        cumulative_key = ut_inverse @ (beta_i[..., None] * k_i)
        cumulative_value = ut_inverse @ (beta_i[..., None] * v_i)
        chunk_update = cumulative_value - cumulative_key @ S.transpose(-1, -2)

        local_score = torch.einsum(
            '... i k, ... j k -> ... i j', q_i, write_key_i
        )
        local_score = local_score.masked_fill(strict_upper, 0)
        output[:, :, n] = (
            q_i @ S.transpose(-1, -2) + local_score @ chunk_update
        )
        S = S + chunk_update.transpose(-1, -2) @ write_key_i

    return rearrange(output, 'b h n c v -> b (n c) h v').to(dtype), S, d
\end{lstlisting}
\end{nolinenumbers}

\paragraph{OSGDN implementation.}
OSGDN first computes the standard Gated DeltaNet log-decay
\(g_t=\log\alpha_t\), then computes \(\tilde K\) in phase~1, and finally
calls the GDN chunk kernels with the d-aware write factor, as shown in
Listing~\ref{lst:chunk_osgdn}. In chunk notation, with
\(\gamma_i=\prod_{\ell\le i}\alpha_\ell\), the d-aware UT entries are
\[
    A_{ij}
    = \mathbf{1}_{j<i}\,
    \beta_i\,\frac{\gamma_i}{\gamma_j}\,
    \langle k_i,\tilde k_j\rangle .
\]
OSGDN uses the post-gate-regret phase-1 recurrence: the residual for the
hypergradient is computed after applying the GDN state gate, so the fixed
point is \(\beta_i\langle d_i,k_i^2\rangle = 1\). Its no-decay kernel
evaluates this beta-aware recurrence with a chunk/WY scan over \(d\); the
data-dependent-retention kernel uses the same recurrence with
\(\rho_i^d=\exp(g_i^d)\). After \(\tilde K\) is materialized, the state
and output kernels are the usual Gated DeltaNet kernels with
\(K\tilde K^\top\) and \(Q\tilde K^\top\) replacing the symmetric key
products.

\begin{nolinenumbers}
\begin{lstlisting}[
style=pytorch,
language=Python,
caption={Chunk update for OSGDN, applied after the phase-1 preconditioner sweep of Listing~\ref{lst:chunk_osdn}.},
label={lst:chunk_osgdn}
]
def chunk_osgdn_update(q, k, write_key, v, beta, alpha, S):
    # one chunk: q, k, write_key: [B, H, C, K], v: [B, H, C, V],
    #            beta, alpha: [B, H, C],   S: [B, H, V, K]
    C = q.shape[-2]
    eye = torch.eye(C, device=q.device)
    upper_including_diag = torch.triu(
        torch.ones(C, C, device=q.device, dtype=torch.bool), 0
    )
    strict_upper = torch.triu(
        torch.ones(C, C, device=q.device, dtype=torch.bool), 1
    )

    # Cumulative gate and gate-aware query / key / write-key factors.
    gamma = alpha.cumprod(dim=-1)
    gamma_C = gamma[..., -1]
    gated_key = gamma[..., :, None] * k
    gated_query = gamma[..., :, None] * q
    normalized_write_key = write_key / gamma[..., :, None].clamp_min(1e-6)

    # Strictly-lower UT system for the WY representation.
    strict_lower_gram = torch.einsum(
        '... i k, ... j k -> ... i j', gated_key, normalized_write_key
    )
    strict_lower_gram = (
        strict_lower_gram * beta[..., :, None]
    ).masked_fill(upper_including_diag, 0)
    ut_weights = torch.linalg.solve_triangular(
        eye + strict_lower_gram, beta[..., None] * eye, upper=False
    )

    cumulative_key = ut_weights @ gated_key
    cumulative_value = ut_weights @ v
    chunk_update = cumulative_value - cumulative_key @ S.transpose(-1, -2)

    local_score = torch.einsum(
        '... i k, ... j k -> ... i j', gated_query, normalized_write_key
    )
    local_score = local_score.masked_fill(strict_upper, 0)
    output = gated_query @ S.transpose(-1, -2) + local_score @ chunk_update

    # Suffix-gated state advance with the OSDN write key.
    suffix = gamma_C[..., None] / gamma
    S_next = gamma_C[..., None, None] * S + chunk_update.transpose(-1, -2) @ (
        suffix[..., :, None] * write_key
    )
    return output, S_next
\end{lstlisting}
\end{nolinenumbers}

\paragraph{OSKDA implementation.}
KDA stores \(S_t\in\mathbb{R}^{K\times V}\), so the same preconditioner
appears on the left storage/write key:
\[
    \tilde k_t=d_t\odot k_t,\qquad
    S_t=(I-\beta_t\tilde k_t k_t^\top)\operatorname{Diag}(\alpha_t)S_{t-1}
        +\beta_t\tilde k_t v_t^\top .
\]
The chunk kernel keeps KDA's fine-grained cumulative gate
\(\Gamma_i=\prod_{\ell\le i}\alpha_\ell\in\mathbb{R}^K\). Its local score
and UT matrices use
\[
    (A_{qk})_{ij}
    = \mathbf{1}_{j\le i}\,
    \langle \Gamma_i\odot q_i,\tilde k_j/\Gamma_j\rangle ,
    \qquad
    (A_{kk})_{ij}
    = \mathbf{1}_{j<i}\,
    \beta_i\langle \Gamma_i\odot k_i,\tilde k_j/\Gamma_j\rangle .
\]
Listing~\ref{lst:chunk_oskda} gives the corresponding update. It
solves the same triangular system as KDA, recomputes the cumulative key and
value terms from \(\beta\,\Gamma\odot K\) and \(\beta V\), and advances the
\(K\times V\) state with the suffix-gated write key
\(\Gamma_{i\to C}\odot\tilde K_i\). The OSKDA configuration used for the
matched comparison uses the beta-aware phase-1 recurrence with constant
preconditioner retention \(\rho^d=0.999\); this retention is separate
from KDA's channel-wise state gate.

\begin{nolinenumbers}
\begin{lstlisting}[
style=pytorch,
language=Python,
caption={Chunk update for OSKDA, applied after the phase-1 preconditioner sweep of Listing~\ref{lst:chunk_osdn}.},
label={lst:chunk_oskda}
]
def chunk_oskda_update(q, k, write_key, v, beta, alpha, S):
    # one chunk: q, k, write_key: [B, H, C, K], v: [B, H, C, V],
    #            beta:  [B, H, C],       alpha: [B, H, C, K],
    #            S:     [B, H, K, V]
    C = q.shape[-2]
    eye = torch.eye(C, device=q.device)
    upper_including_diag = torch.triu(
        torch.ones(C, C, device=q.device, dtype=torch.bool), 0
    )
    strict_upper = torch.triu(
        torch.ones(C, C, device=q.device, dtype=torch.bool), 1
    )

    # KDA keeps a channel-wise cumulative gate Gamma in R^K.
    Gamma = alpha.cumprod(dim=-2)
    Gamma_C = Gamma[..., -1, :]
    gated_key = Gamma * k
    gated_query = Gamma * q
    normalized_write_key = write_key / Gamma.clamp_min(1e-6)

    strict_lower_gram = torch.einsum(
        '... i k, ... j k -> ... i j', gated_key, normalized_write_key
    )
    strict_lower_gram = (
        strict_lower_gram * beta[..., :, None]
    ).masked_fill(upper_including_diag, 0)
    ut_weights = torch.linalg.solve_triangular(
        eye + strict_lower_gram, beta[..., None] * eye, upper=False
    )

    cumulative_key = ut_weights @ gated_key
    cumulative_value = ut_weights @ v
    chunk_update = cumulative_value - cumulative_key @ S

    local_score = torch.einsum(
        '... i k, ... j k -> ... i j', gated_query, normalized_write_key
    )
    local_score = local_score.masked_fill(strict_upper, 0)
    output = gated_query @ S + local_score @ chunk_update

    # Suffix-gated state advance on the K x V state with write_key on the left.
    suffix = Gamma_C[..., None, :] / Gamma
    suffix_gated_write_key = suffix * write_key
    S_next = (
        Gamma_C[..., :, None] * S
        + suffix_gated_write_key.transpose(-1, -2) @ chunk_update
    )
    return output, S_next
\end{lstlisting}
\end{nolinenumbers}

\section{Theoretical Analysis: Quadratic Memory-Regression Dynamics}
\label{sec:theory}

This section gives two complementary mechanism-level guarantees for \name{}, distinguished by the surrogate they control.
\S\ref{subsec:idealized_motivation} treats the \emph{idealised population limit}: full-gradient updates on the expected quadratic objective $f(S)$, with the hypergradient surrogate built from differences of $f$. Under monotone descent and sublinear regret against the full right-Newton comparator $D_\star = \Sigma_k^\dagger$, the suboptimality $f(S_T) - f^*$ contracts at a non-asymptotic super-geometric rate. This idealised limit motivates \name{}, but it is not a statement about the implemented per-token, diagonal, scalar-gated update.
\S\ref{subsec:algorithmic_alignment} closes the algorithm-theory gap by analysing the surrogate that the implementation actually optimises: the token-local hypergradient feedback $h_t(d) = (f_t(S_t) - f_t(S_{t-1}))/\|\nabla f_t(S_{t-1})\|_F^2$ from Section~\ref{subsec:hypergradient}. Under a conditional online-regret assumption against any diagonal comparator, we prove a non-asymptotic contraction bound on the geometric mean of token-local residual ratios. The statement matches what the algorithm runs on each token at the cost of being a token-local rather than population-level guarantee; the implications and non-implications for the global memory-regression objective are made explicit below.

\subsection{Preliminaries and Properties of the Objective}
Let $S_t \in \mathbb{R}^{V \times K}$ denote the hidden state matrix at time step $t$. We work with two related quadratic objectives. The \emph{population-limit} loss
$$
    f(S) = \frac{1}{2} \mathbb{E}_{k,v} \left[ \|S k_t - v_t\|_F^2 \right]
$$
takes the expectation over the data distribution of input keys $k_t \in \mathbb{R}^K$ and target values $v_t \in \mathbb{R}^V$, and is the object analyzed in \S\ref{subsec:idealized_motivation}. The \emph{token-local} loss
$$
    f_t(S) = \tfrac12 \|S k_t - v_t\|_F^2
$$
is the per-token instantaneous version that drives the actual update $S_t = S_{t-1}-\beta_t\nabla f_t(S_{t-1})\osdnhi{D_t}$ used in Section~\ref{sec:preconditioned_delta}; \S\ref{subsec:algorithmic_alignment} works directly with $f_t$.

Because $f(S)$ is an explicit quadratic function, its analytical properties are entirely governed by the uncentered covariance matrix of the key vectors, defined as $\Sigma_k = \mathbb{E}[k_t k_t^\top]$.

\vspace{1em}
\noindent\textbf{Property 1 (Exact Smoothness and Convexity).}
\textit{
The gradient of the objective is $\nabla f(S) = S \Sigma_k - \mathbb{E}[v_t k_t^\top]$. Under the column-wise vectorization of $S$, the Hessian matrix is constant and given exactly by $\mathcal{H} = \nabla^2 f(S) = \Sigma_k \otimes I_V$.
Assuming the data distribution is bounded, $f(S)$ is inherently convex and $L$-smooth, where the smoothness constant is explicitly determined by the maximum eigenvalue of the covariance matrix:
$$
    L = \lambda_{\max}(\Sigma_k) < \infty
$$
Because $\mathcal{H} \succeq 0$, the objective is convex. Note that strong convexity ($\lambda_{\min}(\Sigma_k) > 0$) is optionally permitted but not strictly required for the following convergence bounds. Furthermore, because the Hessian $\mathcal{H}$ is constant, the third-order derivative tensor is exactly zero. Thus, the Hessian Lipschitz constant is precisely zero ($M=0$), which eliminates higher-order residual terms that would otherwise appear in the regret analysis.
}
\vspace{1em}

Let $S^* = \arg\min_S f(S)$ be a global optimum, and $f^* = f(S^*)$ be the optimal loss. Consider the ideal full-gradient right-preconditioned dynamics
$$
    S_{t+1} = S_t - \nabla f(S_t) D_t.
$$
Here the decision variable $D_t$ is a right preconditioner in $\mathbb{R}^{K\times K}$. Its vectorized action satisfies $\mathrm{vec}(\nabla f(S_t)D_t)=(D_t^\top\otimes I_V)\mathrm{vec}(\nabla f(S_t))$; in particular, the right-preconditioner oracle $D_\star=\Sigma_k^\dagger$ corresponds to the full Hessian pseudoinverse $\mathcal{H}^\dagger=\Sigma_k^\dagger\otimes I_V$ in vectorized coordinates. The practical \name{} update restricts $D_t$ to the diagonal class and replaces this expected gradient by the current per-token gradient; the theorem below does not remove those approximation gaps.
To quantify the efficacy of $D_t$ in the population-limit setting, we evaluate the hypergradient feedback $h_t(D)$, which records the relative loss change of a preconditioned step (negative when the step decreases the loss). Following the hypergradient surrogate of~\citet{gao2024gradient}, we define
\begin{equation}
\label{eq:population_h}
    h_t(D) = \frac{f(S_t - \nabla f(S_t) D) - f(S_t)}{\|\nabla f(S_t)\|_F^2}.
\end{equation}
We assume $\|\nabla f(S_t)\|_F>0$ when evaluating $h_t$; if the gradient is zero, the state is already globally optimal for this quadratic objective, the bound below is trivial, and we define $h_t(D)=0$ by convention.
With this convention, $h_t(D) \le 0$ whenever the preconditioned step makes progress, so the natural objective for the meta-learner is to \emph{minimize} $h_t(D)$. This sign convention coincides with both the per-step surrogate used in Section~\ref{sec:preconditioned_delta} and with the surrogate $h_x(P) = (f(x - P\nabla f(x)) - f(x))/\|\nabla f(x)\|^2$ of~\citet{gao2024gradient}. Proposition~6.1 of~\citet{gao2024gradient} shows that $h_x$ is convex in $P$ for general $L$-smooth convex $f$; for our quadratic loss the same conclusion follows directly from the constant positive semidefinite Hessian, as shown in Lemma~\ref{lemma:concavity} below.

\subsection{Population-Limit Setting: Hypergradient Feedback Properties}
\label{subsec:idealized_motivation}

This subsection and the next analyze the population-limit dynamics on $f(S)$ as an idealized motivation. The actual algorithm runs on the per-token surrogate analyzed in \S\ref{subsec:algorithmic_alignment}.
Before proving the convergence rate, we establish two fundamental properties of the population-limit hypergradient feedback $h_t(D)$ defined in Eq.~\eqref{eq:population_h}: its convexity (which justifies the use of Online Gradient Descent to minimize it), and its exact behavior under the ideal Newton step.

\begin{lemma}[Convexity of the Feedback Function]
\label{lemma:concavity}
The hypergradient feedback function $h_t(D)$ is a convex quadratic function with respect to the preconditioner $D$.
\end{lemma}
\begin{proof}
    Using the Taylor expansion for the exact quadratic function $f$, we write the loss after the update as
    $$
        f(S_t - \nabla f(S_t) D) = f(S_t) - \text{tr}\!\left( \nabla f(S_t)^\top \nabla f(S_t) D \right) + \tfrac{1}{2} \text{vec}\!\left( \nabla f(S_t) D \right)^\top \mathcal{H} \, \text{vec}\!\left( \nabla f(S_t) D \right).
    $$
    Substituting into the definition of $h_t(D)$ yields
    $$
        h_t(D) = \frac{-\text{tr}\!\left( \nabla f(S_t)^\top \nabla f(S_t) D \right) + \tfrac{1}{2} \text{vec}\!\left( \nabla f(S_t) D \right)^\top \mathcal{H} \, \text{vec}\!\left( \nabla f(S_t) D \right)}{\|\nabla f(S_t)\|_F^2}.
    $$
    Since $\mathcal{H} \succeq 0$ (positive semi-definite), the quadratic term in $D$ is positive semi-definite, so $h_t(D)$ is convex. Online gradient descent on $h_t$ can therefore be analyzed with standard regret tools under their usual bounded-domain and gradient-bound assumptions.
\end{proof}

\begin{remark}
This convexity is exactly the content of Proposition~6.1 of~\citet{gao2024gradient}, specialised to a quadratic objective; the constant positive semidefinite Hessian both implies convexity and removes the Hessian-Lipschitz residual that appears for generic smooth losses.
\end{remark}

\begin{lemma}[Optimal Feedback of the Exact Right-Newton Step]
\label{lemma:optimal_feedback}
Let $D_\star := \Sigma_k^\dagger$ denote the ideal right preconditioner induced by the key covariance pseudoinverse. The hypergradient feedback evaluated at $D_\star$ satisfies
$$
    h_t(D_\star) = -\,\frac{f(S_t) - f^*}{\|\nabla f(S_t)\|_F^2}\;\le\;0.
$$
\end{lemma}
\begin{proof}
    Write $C=\mathbb{E}[v_t k_t^\top]$, so $\nabla f(S_t)=S_t\Sigma_k-C$. Applying the right-preconditioned update with $D_\star=\Sigma_k^\dagger$ gives
    $$
        S_t^+ = S_t - (S_t\Sigma_k-C)\Sigma_k^\dagger.
    $$
    Since the rows of $C$ and $S_t\Sigma_k$ lie in the row space of $\Sigma_k$, we have $C(I-\Sigma_k^\dagger\Sigma_k)=0$ and $S_t\Sigma_k(I-\Sigma_k^\dagger\Sigma_k)=0$. Therefore the updated gradient is
    $$
        \nabla f(S_t^+) = S_t^+\Sigma_k-C = (S_t\Sigma_k-C)(I-\Sigma_k^\dagger\Sigma_k)=0.
    $$
    Thus $S_t^+$ is a global minimizer up to the nullspace of $\Sigma_k$, and $f(S_t^+)=f^*$. Substituting into the definition of $h_t(D)$ yields the stated equality, equivalently $f(S_t) - f^* = -h_t(D_\star) \|\nabla f(S_t)\|_F^2$.
\end{proof}

\begin{remark}
This exact one-step optimality is specific to quadratic functions. In the general framework of~\citet{gao2024gradient}, the analogous statement is a lower bound, $-h_x(P_h^*) \ge \gamma^* \ge 1/(2L)$ (\citealp[Lemma~6.3]{gao2024gradient}); on quadratics the bound becomes an exact identity, which is the structural property that drives the super-geometric bound below.
\end{remark}

\subsection{Population-Limit Setting: Conditional Super-Geometric Bound}

We now prove a conditional bound for the exact quadratic objective, assuming monotone state updates and sublinear regret against the ideal right-Newton comparator. This is an idealized full-gradient motivation; the algorithmic-aligned bound on the implemented per-token surrogate appears in \S\ref{subsec:algorithmic_alignment}, Theorem~\ref{thm:superlinear}.

\begin{theorem}[Conditional super-geometric convergence of ideal right-preconditioned dynamics]
\label{thm:idealized_population}
Suppose an online learner produces preconditioners $D_t$ whose induced state updates are monotone, $f(S_{t+1}) \le f(S_t)$, and whose cumulative regret against the ideal right-preconditioner comparator $D_\star=\Sigma_k^\dagger$ is bounded by
$$
    \sum_{t=1}^T \bigl( h_t(D_t) - h_t(D_\star) \bigr) \le \mathcal{R}_T.
$$
Then the suboptimality of the inner quadratic regression state after $T$ sequence steps is bounded by
$$
    f(S_{T+1}) - f^* \le \bigl[f(S_1) - f^*\bigr] \left( \frac{2 \lambda_{\max}(\Sigma_k) \mathcal{R}_T}{T} \right)^T.
$$
This theorem is an oracle-comparator statement for the full right-preconditioner class. It becomes a statement about diagonal \name{} only to the extent that the diagonal comparator approximates $D_\star$ and the per-token stochastic updates realize the assumed regret and monotone-descent conditions.
\end{theorem}

\begin{proof}
    We proceed in four steps.

    \paragraph{Step 1: Single-Step Progression Ratio.}
    Let
    $$
        r_t \;=\; \frac{f(S_{t+1}) - f^*}{f(S_t) - f^*} \;=\; 1 - \frac{f(S_t) - f(S_{t+1})}{f(S_t) - f^*}.
    $$
    By definition of the hypergradient feedback, $f(S_t) - f(S_{t+1}) = -h_t(D_t)\, \|\nabla f(S_t)\|_F^2 \ge 0$. Lemma~\ref{lemma:optimal_feedback} gives $f(S_t) - f^* = -h_t(D_\star)\, \|\nabla f(S_t)\|_F^2 > 0$. Substituting,
    $$
        r_t \;=\; 1 - \frac{-h_t(D_t)}{-h_t(D_\star)} \;=\; 1 - \frac{h_t(D_t)}{h_t(D_\star)}. \label{eq:ratio}
    $$

    \paragraph{Step 2: Bounding $|h_t(D_\star)|^{-1}$ via the Rayleigh Quotient.}
    Since $f(S)$ is quadratic with constant Hessian $\mathcal{H} = \Sigma_k \otimes I_V$, write $g_t = \mathrm{vec}(\nabla f(S_t))$ and recall
    $$
        f(S_t) - f^* = \tfrac{1}{2}\, g_t^\top \mathcal{H}^\dagger g_t, \qquad \|g_t\|_2^2 = g_t^\top g_t = \|\nabla f(S_t)\|_F^2.
    $$
    The gradient vector $g_t$ lies in the range of $\mathcal{H}$, so by the Rayleigh quotient and $\lambda_{\max}(\mathcal{H}) = L = \lambda_{\max}(\Sigma_k)$,
    $$
        \frac{g_t^\top \mathcal{H}^\dagger g_t}{g_t^\top g_t} \;\ge\; \frac{1}{\lambda_{\max}(\mathcal{H})} \;=\; \frac{1}{L},
    $$
    so $-h_t(D_\star) \ge 1/(2L)$, i.e.\
    $$
        \frac{1}{|h_t(D_\star)|} \;\le\; 2L. \label{eq:rayleigh_bound}
    $$

    \paragraph{Step 3: AM-GM on the Per-Step Ratios.}
    By the monotonicity assumption, each $r_t \in [0, 1]$. By the Arithmetic--Geometric Mean inequality (the same step that drives Theorem~4.1 of~\citealp{gao2024gradient}),
    $$
        \prod_{t=1}^T r_t \;\le\; \left( \frac{1}{T} \sum_{t=1}^T r_t \right)^{\!T}.
    $$
    Substituting $r_t = 1 - h_t(D_t)/h_t(D_\star)$ and using $h_t(D_\star) < 0$,
    $$
        \sum_{t=1}^T r_t \;=\; \sum_{t=1}^T \frac{h_t(D_t) - h_t(D_\star)}{|h_t(D_\star)|} \;\le\; 2L \sum_{t=1}^T \bigl(h_t(D_t) - h_t(D_\star)\bigr),
    $$
    where the inequality uses Equation~\eqref{eq:rayleigh_bound} applied term-by-term.

    \paragraph{Step 4: Plug in the Regret Bound.}
    The sum $\sum_{t=1}^T (h_t(D_t) - h_t(D_\star))$ is the cumulative regret of the online learner (applied to minimize the convex surrogate $h_t$) relative to the comparator $D_\star$, which is bounded by $\mathcal{R}_T$. Therefore
    $$
        \prod_{t=1}^T r_t \;\le\; \left( \frac{2 L \,\mathcal{R}_T}{T} \right)^{\!T} \;=\; \left( \frac{2 \lambda_{\max}(\Sigma_k)\, \mathcal{R}_T}{T} \right)^{\!T},
    $$
    and multiplying by $f(S_1) - f^*$ gives the stated bound.
\end{proof}

\paragraph{Diagonal comparator gap.}
Let $\mathcal{D}_{\mathrm{box}}=\{\operatorname{Diag}(d):d\in[d_{\min},d_{\max}]^K\}$ be the box-constrained diagonal class actually searched by the implementation (Section~\ref{subsec:hypergradient}), and let $\mathcal{D}_{\mathrm{diag}}=\{\operatorname{Diag}(d):d\in\mathbb{R}^K\}$ be the unrestricted diagonal class. Define the corresponding best fixed comparators
\[
    D_{\mathrm{box}}^\star
    \in
    \arg\min_{D\in\mathcal{D}_{\mathrm{box}}}
    \sum_{t=1}^T h_t(D),
    \qquad
    D_{\mathrm{diag}}^\star
    \in
    \arg\min_{D\in\mathcal{D}_{\mathrm{diag}}}
    \sum_{t=1}^T h_t(D).
\]
Then the regret to the full oracle decomposes as
\[
\begin{aligned}
\sum_{t=1}^T \bigl(h_t(D_t)-h_t(D_\star)\bigr)
&=
\underbrace{\sum_{t=1}^T \bigl(h_t(D_t)-h_t(D_{\mathrm{box}}^\star)\bigr)}_{\text{projected OGD regret}}
\\
&\quad+
\underbrace{\sum_{t=1}^T \bigl(h_t(D_{\mathrm{box}}^\star)-h_t(D_{\mathrm{diag}}^\star)\bigr)}_{\text{box / projection gap}}
\\
&\quad+
\underbrace{\sum_{t=1}^T \bigl(h_t(D_{\mathrm{diag}}^\star)-h_t(D_\star)\bigr)}_{\text{diagonal gap}}.
\end{aligned}
\]
The implemented \name{} update can control only the first term, through its projected diagonal online learner: standard projected-OGD on the convex compact set $\mathcal{D}_{\mathrm{box}}$ with diameter $D_{\mathcal{D}_{\mathrm{box}}}$ and Lipschitz constant $G_h$ (both finite by Lemma~\ref{lemma:tokenlocal_smoothness}) gives a $D_{\mathcal{D}_{\mathrm{box}}}\,G_h\sqrt{T}$ regret bound.
The middle term is the price paid for the box constraint and vanishes whenever an unrestricted diagonal optimum already lies inside $[d_{\min},d_{\max}]^K$ (e.g., under normalized keys with $\beta_t\langle d_{\mathrm{diag}}^\star, k_t^2\rangle\equiv 1$ achievable inside the box).
The third term is structural; it is small when the key covariance is close to diagonal in the model's feature basis, and can be large when the useful curvature is strongly cross-feature.

\subsection{Algorithmic-Aligned Setting: Token-Local Residual Contraction}
\label{subsec:algorithmic_alignment}

The population-limit bound of Theorem~\ref{thm:idealized_population} controls $f(S_T)-f^*$ by telescoping a single global objective $f$ along the iterates. The implemented \name{} update of Section~\ref{subsec:hypergradient} optimizes a different surrogate: at each token it uses the per-token instantaneous loss $f_t(S)=\tfrac12\|Sk_t-v_t\|_F^2$, and consecutive iterates therefore live on \emph{different} quadratics. The product $\prod_t f_t(S_t)/f_t(S_{t-1})$ does not telescope into $(f(S_T)-f^*)/(f(S_1)-f^*)$, so the population-limit proof does not transfer by simply substituting $f_t$ for $f$. This subsection gives a self-contained bound on the algorithmic surrogate.

\paragraph{Algorithmic surrogate.}
For a candidate diagonal preconditioner $D=\operatorname{Diag}(d)$, define the one-step preconditioned state and the token-local hypergradient feedback by
\begin{equation}
\label{eq:algorithmic_h}
    S_t(d) := S_{t-1} - \beta_t \nabla f_t(S_{t-1})\operatorname{Diag}(d),
    \qquad
    h_t(d) := \frac{f_t(S_t(d)) - f_t(S_{t-1})}{\|\nabla f_t(S_{t-1})\|_F^2}.
\end{equation}
This is the surrogate driving the implementation (Equation~\eqref{eq:dt_update} of Section~\ref{subsec:hypergradient}, with $S_t(d_t)=S_t$ along the algorithmic trajectory). Writing $s_t := k_t^{\odot 2}$ and $n_t := \|k_t\|_2^2$, Lemma~1 of Section~\ref{subsec:hypergradient} gives the closed form
\begin{equation}
\label{eq:algorithmic_h_closed}
    h_t(d) = \frac{(1 - \beta_t \langle d, s_t\rangle)^2 - 1}{2 n_t},
\end{equation}
valid whenever $\|\nabla f_t(S_{t-1})\|_F=\|u_t\|\,\|k_t\|>0$, and continuously extended by $h_t\equiv 0$ on the residual-zero set $\{u_t=0\}$. Equation~\eqref{eq:algorithmic_h_closed} depends on neither $S_{t-1}$, $v_t$, nor the residual $u_t$, preserving the decoupling property used by Algorithm~\ref{alg:precond_online}.

\begin{lemma}[Convexity and smoothness of the token-local feedback]
\label{lemma:tokenlocal_smoothness}
For each $t$ with $n_t>0$, the function $h_t(d)$ in Equation~\eqref{eq:algorithmic_h_closed} is convex quadratic in $d$, with
$$
    \nabla h_t(d) = -\frac{\beta_t}{n_t}\bigl(1-\beta_t\langle d, s_t\rangle\bigr) s_t,
    \qquad
    \nabla^2 h_t(d) = \frac{\beta_t^2}{n_t}\, s_t s_t^\top \;\succeq\; 0.
$$
Under the unit-norm normalisation $\|k_t\|_2=1$ and $\beta_t\in(0,1]$,
$$
    \|\nabla^2 h_t(d)\|_2 \;=\; \beta_t^2 \sum_{i=1}^K k_{t,i}^4 \;\le\; \beta_t^2 \;\le\; 1,
$$
where the first inequality uses $\sum_i k_{t,i}^4 \le (\sum_i k_{t,i}^2)^2 = 1$.
\end{lemma}
\begin{proof}
Differentiate Equation~\eqref{eq:algorithmic_h_closed} twice in $d$. The Hessian $\frac{\beta_t^2}{n_t}s_t s_t^\top$ is rank-one positive semidefinite, with operator norm $\frac{\beta_t^2}{n_t}\|s_t\|_2^2$. Under $\|k_t\|_2=1$, $n_t=1$ and $\|s_t\|_2^2=\sum_i k_{t,i}^4$; the Cauchy--Schwarz / power-mean inequality gives the displayed bound.
\end{proof}

\paragraph{Conditional algorithmic regret.}
Let $\mathcal{D}\subset\mathbb{R}^K$ be a closed convex set containing the algorithmic iterates $\{d_t\}$. We assume the online learner producing $\{d_t\}$ admits a sublinear regret bound against any fixed comparator in $\mathcal{D}$:
\begin{equation}
\label{eq:cond_regret}
    \sum_{t=1}^T \bigl(h_t(d_t) - h_t(d)\bigr) \;\le\; R_T(d), \qquad R_T(d) = o(T),\qquad \forall d\in\mathcal{D}.
\end{equation}
For projected online gradient descent on $\mathcal{D}$ with diameter $D_{\mathcal D}$ and gradient bound $G_h$ on $\mathcal{D}$ (both finite by Lemma~\ref{lemma:tokenlocal_smoothness}), the standard analysis gives $R_T = D_{\mathcal D}\,G_h\,\sqrt T$. The implementation in Section~\ref{subsec:hypergradient} instantiates $\mathcal{D}=[d_{\min},d_{\max}]^K$ via the explicit box clamp in Algorithm~\ref{alg:precond_online}, so $D_{\mathcal D}=(d_{\max}-d_{\min})\sqrt{K}$ is finite and the assumption holds for projected OGD. As in the population-limit setting we keep Equation~\eqref{eq:cond_regret} as a conditional assumption to avoid committing to a specific projection scheme in the bound, and discuss specialisations afterwards.

\begin{theorem}[Conditional algorithmic regret and token-local residual contraction]
\label{thm:superlinear}
Assume $\|k_t\|_2=1$ for all $t$ and the conditional-regret assumption~\eqref{eq:cond_regret}. Define the comparator gap
$$
    \varepsilon_T(d) \;:=\; \frac{1}{2T}\sum_{t=1}^T \bigl(1-\beta_t\langle d, s_t\rangle\bigr)^2,
    \qquad
    \varepsilon_T^{\mathrm{diag}} \;:=\; \min_{d\in\mathcal{D}} \varepsilon_T(d).
$$
Then for any $d\in\mathcal{D}$,
$$
    \prod_{t=1}^T \frac{f_t(S_t)}{f_t(S_{t-1})}
    \;\le\;
    \biggl(\,2\varepsilon_T(d) + \frac{2 R_T(d)}{T}\,\biggr)^{\!T},
$$
and in particular, evaluating at the minimiser $d^\star_{\mathrm{diag}}$,
$$
    \prod_{t=1}^T \frac{f_t(S_t)}{f_t(S_{t-1})}
    \;\le\;
    \biggl(\,2\varepsilon_T^{\mathrm{diag}} + \frac{2 R_T}{T}\,\biggr)^{\!T}.
$$
If furthermore there exists $d^\star\in\mathcal{D}$ realising the gated diagonal Newton condition $\beta_t\langle d^\star, s_t\rangle=1$ for all $t$ (so $\varepsilon_T^{\mathrm{diag}}=0$), then under sublinear regret $R_T=o(T)$,
$$
    \prod_{t=1}^T \frac{f_t(S_t)}{f_t(S_{t-1})}
    \;\le\;
    \biggl(\frac{2 R_T}{T}\biggr)^{\!T},
$$
which decays super-geometrically. With the standard OGD regret $R_T=O(\sqrt T)$, the rate is $\bigl(O(1)/\sqrt T\bigr)^{T}$.
\end{theorem}

\begin{proof}
Set $u_t := v_t - S_{t-1}k_t$. The post-update residual at the algorithmic iterate $d_t$ is
$$
    S_t k_t - v_t \;=\; (S_{t-1}+\beta_t u_t (d_t\odot k_t)^\top) k_t - v_t \;=\; -u_t\bigl(1-\beta_t\langle d_t, s_t\rangle\bigr),
$$
so $f_t(S_t) = \tfrac12 \|u_t\|_2^2 (1-\beta_t\langle d_t, s_t\rangle)^2$ and $f_t(S_{t-1}) = \tfrac12 \|u_t\|_2^2$. Defining $q_t := f_t(S_t)/f_t(S_{t-1})$ on the non-degenerate set $\{u_t\neq 0\}$,
\begin{equation}
\label{eq:qt_form}
    q_t \;=\; \bigl(1-\beta_t\langle d_t, s_t\rangle\bigr)^2 \;\ge\; 0,
\end{equation}
and Equation~\eqref{eq:algorithmic_h_closed} with $n_t=1$ gives $q_t = 1 + 2h_t(d_t)$. (When $u_t=0$ the ratio is $0/0$; we interpret $q_t=1$ since $f_t(S_t)=f_t(S_{t-1})=0$, consistent with $h_t=0$.)

Since $q_t\ge 0$, the AM--GM inequality gives
$$
    \prod_{t=1}^T q_t \;\le\; \Bigl(\frac{1}{T}\sum_{t=1}^T q_t\Bigr)^{\!T} \;=\; \Bigl(\,1 + \frac{2}{T}\sum_{t=1}^T h_t(d_t)\Bigr)^{\!T}.
$$
By Equation~\eqref{eq:cond_regret}, for any $d\in\mathcal{D}$,
$$
    \sum_{t=1}^T h_t(d_t) \;\le\; \sum_{t=1}^T h_t(d) + R_T(d).
$$
The closed-form Equation~\eqref{eq:algorithmic_h_closed} together with $n_t=1$ yields
$$
    \sum_{t=1}^T h_t(d) \;=\; T\,\varepsilon_T(d) - \tfrac{T}{2}.
$$
Combining,
$$
    1 + \frac{2}{T}\sum_{t=1}^T h_t(d_t) \;\le\; 1 + 2\varepsilon_T(d) - 1 + \frac{2R_T(d)}{T} \;=\; 2\varepsilon_T(d) + \frac{2R_T(d)}{T},
$$
and raising to the $T$-th power gives the stated bound. The specialisations follow by minimising over $d\in\mathcal{D}$ and substituting $\varepsilon_T^{\mathrm{diag}}=0$ when the gated diagonal Newton condition is feasible.
\end{proof}

\paragraph{What the theorem controls (and what it does not).}
The left-hand side is the product of \emph{token-local} residual ratios, evaluated at successive per-token quadratics $f_t$, not the suboptimality of any single global objective. The bound therefore does not imply convergence of $f(S_T)-f^*$. Concretely, take $K=V=1$, $k_t=1$, and $v_t=(-1)^t$, with an exact per-token Newton step ($\beta_t\langle d_t,s_t\rangle=1$): each token-local residual is driven to zero, so $\prod_t q_t = 0$ trivially, yet $S_t$ alternates between $\pm 1$ and never approaches the population minimiser $S^\star=0$. Translating Theorem~\ref{thm:superlinear} to a global statement requires additional stationarity or no-conflict assumptions on the data stream; the population-limit Theorem~\ref{thm:idealized_population} provides one such oracle setting, and the corollary below provides another.

\begin{corollary}[Repeated-key contraction]
\label{cor:repeated_keys}
Suppose tokens are typed by $c\in\mathcal{C}$, with key $k_c$ and target $v_c$ depending only on $c$, and the keys $\{k_c\}$ are mutually orthogonal and supported on disjoint coordinate blocks; and suppose $D_t$ respects this block structure. Define the per-class loss $F_c(S):=\tfrac12\|S k_c - v_c\|_F^2$. Then for any sequence in which each class $c$ appears at least once,
$$
    \prod_{c\in\mathcal{C}} \frac{F_c(S_T)}{F_c(S_0)} \;=\; \prod_{t=1}^T q_t,
$$
and under the assumptions of Theorem~\ref{thm:superlinear},
$$
    \prod_{c\in\mathcal{C}} \frac{F_c(S_T)}{F_c(S_0)} \;\le\; \biggl(2\varepsilon_T^{\mathrm{diag}} + \frac{2 R_T}{T}\biggr)^{\!T}.
$$
\end{corollary}
\begin{proof}
Block orthogonality and the block structure of $D_t$ ensure that an update with key $k_c$ leaves $S k_{c'}$ unchanged for $c'\ne c$, so $F_{c'}(S_t)=F_{c'}(S_{t-1})$. Telescoping the per-class residual ratios over each $c$-block of token positions gives $F_c(S_T)/F_c(S_0)=\prod_{t:\,c_t=c} q_t$, and taking the product over $c$ recovers $\prod_t q_t$. Theorem~\ref{thm:superlinear} bounds this product.
\end{proof}

\paragraph{Reading.} Theorem~\ref{thm:superlinear} is the regret-aligned analogue of Theorem~\ref{thm:idealized_population}: the population-limit theorem controls a global suboptimality at the cost of an idealised full-gradient comparator that the algorithm does not see, while Theorem~\ref{thm:superlinear} controls the algorithm's own surrogate at the cost of a token-local contraction that requires extra structure (Corollary~\ref{cor:repeated_keys}, or the population limit) to lift to a global statement. Repeated-context retrieval (e.g., JRT-twice) is the empirical regime in which the corollary's structural assumption is most naturally satisfied; this matches the experimental finding in Section~\ref{sec:exp} that \name's gain concentrates on repeated-key tasks.

\subsection{Discussion: The Collapse of Residuals in DeltaNet}

Theorems~\ref{thm:idealized_population} and~\ref{thm:superlinear} adapt the OSGM machinery of~\citet{gao2024gradient} to the sequence-level memory update of DeltaNet at two complementary levels of idealisation. The resulting bounds share three structural features driven by the quadratic geometry of the loss.

\textbf{First, both rates are super-geometric instead of merely linear.} In the general setting of~\citet{gao2024gradient}, the asymptotic complexity is $O\!\bigl(\tfrac{1}{2\mu\gamma^*} \log(1/\varepsilon)\bigr)$ for strongly convex objectives (linear convergence) and $O(1/(\gamma^* \varepsilon))$ otherwise (sublinear). Both of our bounds, in contrast, take the form $O\!\bigl((C\mathcal{R}_T/T)^T\bigr)$, which is super-geometric whenever $\mathcal{R}_T=o(T)$ and therefore eventually decays faster than any geometric progression $\rho^T$ for $0<\rho<1$. The mechanism is twofold:
\begin{enumerate}
    \item Sublinear-regret online learning gives $\mathcal{R}_T=O(\sqrt T)$ (or $O(\log T)$ under additional curvature) so $\mathcal{R}_T/T\to 0$.
    \item In Theorem~\ref{thm:idealized_population}, the Rayleigh-quotient inequality (Step~2) provides a constant lower bound $|h_t(D_\star)|\ge 1/(2L)$, so the per-step ratio $h_t(D_t)/h_t(D_\star)$ approaches~1 as the meta-learner converges, driving each $r_t$ to zero. In Theorem~\ref{thm:superlinear}, the analogous role is played by the AM--GM step, which converts cumulative regret on the per-token surrogate into a contraction on $\prod_t q_t$.
\end{enumerate}
With the standard OGD regret $\mathcal{R}_T=O(\sqrt T)$, both displayed theorems give $O((C/\sqrt T)^T)$. This is super-geometric but \emph{not} factorial. The bounds become informative once the constant factor is dominated by the growing denominator; for the population-limit theorem this corresponds to $T=\Omega(L^2)$ up to problem-dependent constants, while for the algorithmic-aligned theorem it corresponds to $2\varepsilon_T^{\mathrm{diag}}+2R_T/T<1$.

\textbf{Second, neither bound carries a Hessian-Lipschitz residual.} For non-quadratic losses, the analysis of~\citet{gao2024gradient} introduces residual terms scaling with the Hessian-Lipschitz constant $M$ (\citealp[Proposition~5.1]{gao2024gradient} bounds the nonconvexity of $g_x(P)$ by $HD^2 \|\nabla f(x)\|$). Both $f(S)$ and $f_t(S)$ are exact quadratics, so Property~1 gives $M=0$ in the population limit and the corresponding higher-order residual is zero token-wise; the asymptotic burn-in period that limits these analyses on generic non-quadratic objectives is absent in either route.

\textbf{Connection to OSGM-R.} The population-limit Theorem~\ref{thm:idealized_population} is structurally close to Theorem~4.4 of~\citet{gao2024gradient}, which establishes superlinear convergence of OSGM-R on strongly convex quadratics: $f(x^{K+1}) - f^* \le (f(x^1) - f^*)\bigl(\tfrac{4L^2 \|P_1 - A^{-1}\|_F^2}{K}\bigr)^K$. The two routes differ in the surrogate they use---OSGM-R uses the ratio $r_x(P)$, which is $2L^2$-smooth and admits a vanishing $L^*$-regret on quadratics, while we work with the hypergradient surrogate $h_t(D)$, which is convex but only Lipschitz. The hypergradient route is tractable in our setting because of the exact one-step optimality of $D_\star$ (Lemma~\ref{lemma:optimal_feedback}), a property specific to the quadratic loss of DeltaNet. A practical advantage is that the hypergradient route does not require knowledge of $f^*$, which is unavailable at training time.

\textbf{Comparator and feasibility.} Each theorem has its own comparator. Theorem~\ref{thm:idealized_population} compares against the right-preconditioner $D_\star=\Sigma_k^\dagger$, which lies outside the diagonal cone $\{\mathrm{diag}(d):d\in\mathbb{R}^K\}$ that the actual \name{} update searches; the diagonal approximation gap quantifies what is lost. Theorem~\ref{thm:superlinear} compares against any diagonal $d\in\mathcal{D}$, in particular the diagonal-Newton point $\beta_t\langle d, s_t\rangle\equiv 1$ when feasible, and the bound is therefore stated entirely inside the algorithm's feasible set. The price of this alignment is the token-local nature of the left-hand side: the bound controls $\prod_t f_t(S_t)/f_t(S_{t-1})$ rather than a single global suboptimality.

\textbf{Practical implication.} Plugging the standard OGD regret into the two theorems gives, respectively,
$$
    f(S_{T+1}) - f^* \;\le\; [f(S_1) - f^*]\,\Bigl(\tfrac{O(1)}{\sqrt T}\Bigr)^{T}\quad(\text{Thm.~\ref{thm:idealized_population}, full-gradient on }f),
$$
$$
    \prod_{t=1}^T \frac{f_t(S_t)}{f_t(S_{t-1})}\;\le\;\Bigl(2\varepsilon_T^{\mathrm{diag}} + \tfrac{O(1)}{\sqrt T}\Bigr)^{T}\quad(\text{Thm.~\ref{thm:superlinear}, per-token surrogate}).
$$
The population-limit form provides a mechanism-level rationale for per-feature preconditioning under idealised assumptions. The algorithmic form holds for the surrogate Algorithm~\ref{alg:precond_online} actually optimises (modulo the conditional-regret assumption~\eqref{eq:cond_regret}), and Corollary~\ref{cor:repeated_keys} lifts it to a per-class residual contraction in the structured no-conflict regime that closely tracks the repeated-recall splits in Section~\ref{sec:exp}.

\textbf{Scope with APF and stochastic updates.} Algorithm~\ref{alg:precond_online} runs projected hypergradient descent on the box $\mathcal{D}=[d_{\min},d_{\max}]^K$, so the diameter and gradient bounds required by projected OGD are satisfied (Lemma~\ref{lemma:tokenlocal_smoothness}); both theorems still keep the cumulative-regret assumption~\eqref{eq:cond_regret} as a stand-in for any specific online-learning rate schedule, since the reported runs use a small smoothness-motivated practical online step rather than the OGD-textbook $\eta=O(1/\sqrt{T})$. \apf further introduces a retention gate for non-stationary streams, which would require a dynamic-regret or drifting-comparator analysis to treat the comparator $\varepsilon_T(d)$ as time-varying.

\section{Benchmark Suite and Metrics}
\label{sec:app_benchmark_suite}

Our evaluation suite follows the split used in recent linear-recurrent language-model papers: separate language-modelling quality, short-context commonsense checks, explicit associative-recall diagnostics, long-context understanding, and length extrapolation, rather than collapsing all evidence into a single aggregate number~\citep{yang_gated_2024-1,kimiteam2025kimilinearexpressiveefficient,vonoswald_mesanet_2025}. This separation matters for \name{}, since the proposed mechanism changes the geometry of fast-weight writes; the primary expected signal is retrieval after repeated related writes, not a uniform improvement on every downstream task.

\paragraph{Language modeling perplexity.}
Table~\ref{tab:model_ppl} reports zero-shot perplexity on WikiText and LAMBADA.
WikiText is a standard language-modeling corpus with longer articles and a
realistic vocabulary~\citep{merity_pointer_2017}.  LAMBADA asks the model to
predict the final word of a passage that requires broad discourse
context~\citep{paperno_lambada_2016}; we use its perplexity in the PPL table and
its accuracy in the commonsense average.  PG-19 is a book-scale
language-modeling corpus introduced for long-range sequence modeling
\citep{rae_compressive_2020}; we report it separately as an auxiliary appendix
length-extrapolation diagnostic.  The FineWeb-Edu validation column in
Table~\ref{tab:model_ppl} evaluates a fixed in-domain sample-10BT slice,
\texttt{train[-10000:]}, capped at 10.0M next-token labels.  Because the cached
FineWeb-Edu copy exposes only the training split, we treat this as a
pseudo-validation check for training-distribution fit rather than as an
out-of-domain generalization benchmark.

\paragraph{Commonsense and short-context language understanding.}
The Common.\ column averages zero-shot accuracy over PIQA, HellaSwag, WinoGrande, ARC-Easy, ARC-Challenge, SocialIQA, BoolQ, and LAMBADA. These benchmarks cover physical commonsense~\citep{bisk_piqa_2020}, adversarial sentence completion~\citep{zellers_hellaswag_2019}, Winograd-style pronoun resolution~\citep{sakaguchi_winogrande_2021}, grade-school science questions~\citep{clark_arc_2018}, social commonsense~\citep{sap_socialiqa_2019}, natural yes/no reading-comprehension questions~\citep{clark_boolq_2019}, and discourse word prediction~\citep{paperno_lambada_2016}. We treat the average as a scope check: it verifies that the recurrent-layer modification preserves general short-context capability, rather than corresponding to the mechanism-specific claim of the paper.

\paragraph{Associative retrieval.}
The main diagnostic is JRT-style cloze recall at 2K context
\citep{arora_just_2024}.  We evaluate FDA, SWDE, and SQuAD, together with
\emph{-twice} variants where the same context is repeated before the query.  The
reported metric is \textit{contains} accuracy: a prediction is correct if it
contains the target answer string.  The repeated split is the most direct stress
test for \name{} because recurring key directions give the online preconditioner
multiple opportunities to adapt the effective write scale.

\paragraph{Generation-based retrieval.}
The full sweep also includes TriviaQA, DROP, and NQ-Open through
\texttt{lm\_eval}~\citep{joshi_triviaqa_2017,dua_drop_2019,kwiatkowski_natural_2019}.
These tasks require free-form answer generation, so the absolute scores are low
for all base, non-instruction-tuned 340M models.  We therefore use them as
secondary checks rather than headline evidence; the cloze-style recall results are
a cleaner probe of whether the recurrent memory writes preserve and retrieve
associations.

\paragraph{Long-context understanding and length extrapolation.}
LongBench is reported as a 14-task average over English single-document QA,
multi-document QA, summarization, few-shot classification, and synthetic
retrieval/counting tasks~\citep{bai_longbench_2024}.  The PG-19 appendix
diagnostic in Figure~\ref{fig:pg19_buckets} and Table~\ref{tab:pg19_buckets}
trains at 2K context and evaluates on 20K-token blocks, with 2K-token buckets
used to expose position-dependent drift.

\section{Training and Evaluation Reproduction Details}
\label{sec:app_reproducibility}

This appendix records the optimizer, batching, and evaluation settings used by
the matched 340M-scale runs in Section~\ref{sec:exp}.  The modelling choices
are described in Section~\ref{sec:preconditioned_delta} and
Appendix~\ref{sec:app_variant_ablations}.

\paragraph{Software environment.}
All training jobs use Python 3.11, PyTorch 2.6 with CUDA 12.4,
\texttt{transformers} 4.51, and \texttt{datasets} 3.3.  Models are trained in
AMP bfloat16 with float32 reductions.

\paragraph{Hardware.}
Each run uses a single node of NVIDIA H100 80GB GPUs; no multi-node
communication is involved.

\paragraph{Training data.}
All runs use the \texttt{fla-hub/delta\_net-1.3B-100B} tokenizer and the
public FineWeb-Edu \texttt{sample-10BT} training split
(\texttt{HuggingFaceFW/fineweb-edu}).  The in-domain validation column in
Table~\ref{tab:model_ppl} is computed on a fixed held-out slice of the same
corpus.

\paragraph{Model configurations.}
Table~\ref{tab:repro_model_configs} summarises the per-backbone architectural
scale.  Online-scaled runs initialise the diagonal preconditioner at one and
use the beta-aware phase-1 update.  The refreshed DeltaNet no-APF
\name{} screen uses \(\eta=0.003\), \(d_{\min}=0.5\),
\(d_{\max}=2.0\), \(D_0=\mathbf 1\), and the same 4-GPU fair token budget
as the retrained DeltaNet baseline.  The refreshed KDA bounded
runs use the same \(\eta,d_{\min},d_{\max}\) values; OSKDA disables
preconditioner retention, while OSKDA-APF uses the data-dependent
preconditioner-retention gate.  APF variants add the data-dependent retention
gate described in Section~\ref{subsec:apf}.

\begin{table}[H]
\centering
\small
\setlength{\tabcolsep}{6pt}
\renewcommand{\arraystretch}{1.12}
\caption{\textbf{Per-backbone architectural scale used in the matched 340M
sweep.}}
\label{tab:repro_model_configs}
\begin{tabular}{lccc}
\toprule
\textbf{Backbone} & \textbf{Layers} & \textbf{Width} & \textbf{Heads} \\
\midrule
DeltaNet       & 24 & 1024 & 8 \\
KDA            & 23 & 1024 & 8 \\
Gated DeltaNet & 21 & 1024 & 6 \\
\bottomrule
\end{tabular}
\end{table}

\paragraph{Optimizer, batch, seed, and precision.}
The fairness constraint is tokens per optimizer step.  We distinguish three
related lengths to avoid ambiguity.  (i) The \emph{recurrent training
context} is the maximum span over which the model carries hidden state
without reset; in our runs this is at most
\(\texttt{context\_len}=4{,}096\) tokens, since FineWeb-Edu documents longer
than this are further chunked.  (ii) The \emph{packed sequence per GPU} is
the contiguous tensor consumed in one forward pass; we use
\(\texttt{batch\_size}=1\) with \(\texttt{seq\_len}=65{,}536\) and
\(\texttt{varlen}=\text{True}\), so each GPU consumes one 65{,}536-token
packed batch composed of variable-length FineWeb-Edu segments whose
boundaries are tracked by \texttt{cu\_seqlens}; the recurrent state is reset
at every boundary, so the model never sees a contiguous training segment
longer than 4K tokens.  (iii) The \emph{effective batch tokens per
optimizer step}: the 8-GPU jobs run with no gradient accumulation while the
4-GPU jobs use gradient accumulation of two, so both schedules process
\(8\times 65{,}536=524{,}288\) tokens per optimizer step and roughly
10.74B tokens in total over 20{,}480 optimizer steps.  Optimization uses
AdamW with learning rate \(10^{-3}\), a cosine schedule with linear warmup,
gradient norm clipping at 1.0, random seed 42, and bfloat16 training with
float32 reductions.  Because the OSDN rows, baseline rows, and
online-scaled runs share this seed, FineWeb-Edu shard ordering, and batch
schedule, the within-group deltas reported in Section~\ref{sec:exp}
isolate the architectural change rather than seed-level optimization
noise.

\paragraph{Evaluation suites.}
We report (i) zero-shot commonsense and language-modelling on PIQA, HellaSwag,
WinoGrande, ARC-Easy, ARC-Challenge, Social-IQA, BoolQ, WikiText, and
LAMBADA-OpenAI; (ii) JRT-style cloze recall on FDA, SWDE, SQuAD, and their
\emph{-twice} variants; (iii) open-domain retrieval on TriviaQA, DROP, and
NQ-Open; (iv) the LongBench English suite; and (v) PG-19 length-extrapolation
perplexity buckets.  Unless
stated otherwise, evaluation uses greedy decoding through the HuggingFace
checkpoint converted from the distributed checkpoint at step 20,480.  The
refreshed no-APF OSDN and bounded OSKDA rows are recorded under the following
anonymous final-run labels:
\begin{quote}\scriptsize
\texttt{OSDN-340M-noAPF-final}\\
\texttt{OSKDA-340M-noAPF-final}\\
\texttt{OSKDA-340M-APF-final}
\end{quote}
Legacy no-APF OSDN and no-DD OSKDA screens are retained only as provenance for
earlier table entries; the bounded final-run labels above supply the refreshed
no-APF OSDN and OSKDA rows used in the paper.

\paragraph{Mechanism diagnostic.}
The residual-ratio diagnostic in Section~\ref{subsec:exp_mechanism} uses the
same converted checkpoints as the recall suite.  For each checkpoint, we replay
16 validation prompts from each JRT \emph{-twice} task, reconstruct every
recurrent layer's write-side \(k_t\), \(v_t\), \(\beta_t\), and, where
present, \(d_t\), and run the state recurrence in fp32 to record
\(f_t(S_{t-1})\), \(f_t(S_t)\), and
\(q_t=f_t(S_t)/f_t(S_{t-1})\).  The reported averages cover every DeltaNet
layer and recurrent head in the measured architecture (24 layers / 8 heads at
340M, 24 layers / 16 heads at 1.3B; the protocol is otherwise identical
across scales).

\paragraph{Compute footprint.}
A single matched 340M run takes on the order of half a day on a single
H100 80GB node, depending on the backbone and on whether four or eight
GPUs are used.  Beyond these runs, the full research project used additional
H100 hours on hyperparameter screens and on shorter pilot runs that did not
reach the full 20,480-step budget; this preliminary compute is not counted in
the figure above.

\paragraph{1.3B / 100B scaling configuration.}
The scaling validation in Section~\ref{subsec:exp_scale} extends the matched
sweep to DeltaNet scaled to 1.3B parameters and
trained on 100B tokens of the same FineWeb-Edu corpus.  The architectural
configuration (layers, width, heads) for both DeltaNet and \apf{} at this
scale is summarised in Table~\ref{tab:repro_scale_configs}.  The optimizer
class (AdamW), gradient clipping, training precision (bfloat16 with float32
reductions), evaluation tokenizer, and HuggingFace conversion pipeline match
the 340M sweep; the run-specific learning rate, schedule shape, batch
configuration, total optimizer steps, hardware count, and wall-clock cost are
tracked in the training logs because the 1.3B sweep is run separately from the
matched 340M screen.

\begin{table}[htbp]
\centering
\small
\setlength{\tabcolsep}{6pt}
\renewcommand{\arraystretch}{1.12}
\caption{\textbf{1.3B / 100B scaling configuration.}  DeltaNet backbone;
\apf{} adds the OSDN online preconditioner and the APF retention gate on top
of the same backbone.}
\label{tab:repro_scale_configs}
\begin{tabular}{lccccc}
\toprule
\textbf{Model} & \textbf{Params} & \textbf{Layers} & \textbf{Width} & \textbf{Heads} & \textbf{Tokens} \\
\midrule
DeltaNet & 1.3B & 24 & 2048 & 16 & 100B \\
\apf{}   & 1.3B & 24 & 2048 & 16 & 100B \\
\bottomrule
\end{tabular}
\end{table}

\paragraph{Datasets, models, and software: licenses and terms of use.}
All datasets, tokenizers, and software dependencies used in this work are
public open-source releases distributed under their original licence
terms; this work uses each in compliance with those terms and does not
re-distribute any of them.  \emph{Training corpus.}  FineWeb-Edu
\texttt{sample-10BT} (\texttt{HuggingFaceFW/fineweb-edu}, ODC-By 1.0;
the upstream CommonCrawl terms of use apply to the underlying web
content).  \emph{Tokenizer / model assets.}  The
\texttt{fla-hub/delta\_net-1.3B-100B} HuggingFace tokenizer.
\emph{Evaluation benchmarks}, cited in
Appendix~\ref{sec:app_benchmark_suite} alongside their reference papers
and HuggingFace dataset cards: WikiText~\citep{merity_pointer_2017},
LAMBADA~\citep{paperno_lambada_2016}, PIQA~\citep{bisk_piqa_2020},
HellaSwag~\citep{zellers_hellaswag_2019},
WinoGrande~\citep{sakaguchi_winogrande_2021},
ARC-Easy / ARC-Challenge~\citep{clark_arc_2018},
Social-IQA~\citep{sap_socialiqa_2019}, BoolQ~\citep{clark_boolq_2019},
SQuAD and the JRT-style FDA / SWDE / SQuAD and \emph{-twice}
variants~\citep{arora_just_2024},
TriviaQA~\citep{joshi_triviaqa_2017}, DROP~\citep{dua_drop_2019},
NQ-Open~\citep{kwiatkowski_natural_2019},
LongBench~\citep{bai_longbench_2024}, and
PG-19~\citep{rae_compressive_2020}.  \emph{Software stack.}  PyTorch
(BSD-3-Clause); \texttt{transformers}, \texttt{datasets}, and
\texttt{accelerate} (Apache 2.0); \texttt{lm\_eval} (MIT); and the
\texttt{flash-linear-attention} kernel collection (MIT).

\section{Additional 340M Benchmark Breakdowns}
\label{sec:app_general_checks}

These broader benchmark results are not the headline evidence for the paper's empirical claim, but they delineate the mechanism's scope: OSDN is a targeted online-preconditioning method for associative retrieval, while APF stabilises long, non-stationary contexts.  This appendix expands the consolidated main-text Table~\ref{tab:main_results} into per-task breakdowns and additional Gated DeltaNet and KDA rows, and reports the JRT-style and commonsense splits.

\subsection{Language-model perplexity, full breakdown}

\begin{table}[htbp]
\centering\scriptsize
\setlength{\tabcolsep}{4pt}
\renewcommand{\arraystretch}{1.16}
\caption{\textbf{Language-model perplexity summary, full matched 340M sweep.}
Lower is better.  FW-Edu val reports a fixed in-domain FineWeb-Edu
sample-10BT slice (\texttt{train[-10000:]}, 10.0M next-token labels).
GeoMean is computed over the completed zero-shot PPL columns: WikiText
and LAMBADA.  PG-19 length-extrapolation perplexity is reported in
Appendix~\ref{sec:app_pg19_details}.  $\Delta\mathrm{NLL}$ is the
log-GeoMean difference to the matched baseline within each dashed group.}
\label{tab:model_ppl}
\begin{adjustbox}{max width=\linewidth}
\begin{tabular}{l|c|cc|cc}
\toprule
\textbf{Model} &
\shortstack[c]{\textbf{FW-Edu val}\\[-0.15ex]\tiny PPL \(\downarrow\)} &
\shortstack[c]{\textbf{WikiText}\\[-0.15ex]\tiny PPL \(\downarrow\)} &
\shortstack[c]{\textbf{LAMBADA}\\[-0.15ex]\tiny PPL \(\downarrow\)} &
\shortstack[c]{\textbf{GeoMean}\\[-0.15ex]\tiny PPL \(\downarrow\)} &
\shortstack[c]{\(\boldsymbol{\Delta}\mathrm{NLL}\)\\[-0.15ex]\tiny \(\downarrow\)} \\
\midrule
DeltaNet                   & 12.43 & 28.73 & 35.65 & 32.00 & -- \\
\osdnrow
\name{}                   & \textbf{12.32} & 28.57 & 35.09 & 31.67 & \(-0.011\) \\
\apfrow
\apf{}                    & 12.39 & \textbf{28.07} & \textbf{34.21} & \textbf{30.99} & \(\mathbf{-0.032}\) \\
\hdashline
GDN                        & \textbf{11.97} & \textbf{27.43} & 32.83 & 30.01 & -- \\
\osdnrow
OSGDN                     & 12.04 & 27.85 & 31.84 & 29.78 & \(-0.008\) \\
\apfrow
OSGDN-APF                 & 12.07 & 27.80 & \textbf{31.31} & \textbf{29.50} & \(\mathbf{-0.017}\) \\
\hdashline
KDA                        & \textbf{11.38} & 26.56 & \textbf{26.95} & \textbf{26.75} & -- \\
\osdnrow
OSKDA                     & 11.69 & \textbf{26.42} & 29.54 & 27.93 & \(+0.043\) \\
\apfrow
OSKDA-APF                 & 11.70 & 26.60 & 29.52 & 28.02 & \(+0.046\) \\
\bottomrule
\end{tabular}
\end{adjustbox}
\end{table}

Table~\ref{tab:model_ppl} shows that online scaling does not trade
retrieval gains for a systematic language-modelling loss in the DeltaNet
and GDN families: \apf{} gives the best WikiText/LAMBADA GeoMean in both
blocks, and vanilla \name{} also improves the DeltaNet baseline.  The
KDA block is different: KDA remains the strongest perplexity baseline,
while OSKDA variants are used below mainly as retrieval-mechanism scope
checks rather than perplexity improvements.

\subsection{In-context recall, JRT-style cloze breakdown}

\begin{table}[htbp]
\centering\scriptsize
\setlength{\tabcolsep}{3.2pt}
\renewcommand{\arraystretch}{1.13}
\caption{\textbf{In-context recall, JRT-style cloze at 2K context, full
matched 340M rows.}  Reported in \textit{contains} accuracy ($\uparrow$).
Single averages FDA, SWDE, and SQuAD; Repeated averages the corresponding
\emph{-twice} variants.}
\label{tab:ic_recall}
\begin{adjustbox}{max width=\linewidth}
\begin{tabular}{l|ccc|c|ccc|cc}
\toprule
\textbf{Model} &
\shortstack[c]{\textbf{FDA}\\[-0.15ex]\tiny acc. \(\uparrow\)} &
\shortstack[c]{\textbf{SWDE}\\[-0.15ex]\tiny acc. \(\uparrow\)} &
\shortstack[c]{\textbf{SQuAD}\\[-0.15ex]\tiny acc. \(\uparrow\)} &
\shortstack[c]{\textbf{Single}\\[-0.15ex]\tiny avg. \(\uparrow\)} &
\shortstack[c]{\textbf{FDA-tw.}\\[-0.15ex]\tiny acc. \(\uparrow\)} &
\shortstack[c]{\textbf{SWDE-tw.}\\[-0.15ex]\tiny acc. \(\uparrow\)} &
\shortstack[c]{\textbf{SQuAD-tw.}\\[-0.15ex]\tiny acc. \(\uparrow\)} &
\shortstack[c]{\textbf{Repeated}\\[-0.15ex]\tiny avg. \(\uparrow\)} &
\shortstack[c]{\textbf{Overall}\\[-0.15ex]\tiny avg. \(\uparrow\)} \\
\midrule
DeltaNet       & \textbf{0.087} & 0.089 & 0.287 & 0.155 & 0.045 & 0.158 & 0.231 & 0.145 & 0.150 \\
\osdnrow
OSDN          & 0.077 & \textbf{0.156} & \textbf{0.303} & \textbf{0.179} & 0.044 & \textbf{0.227} & \textbf{0.383} & \textbf{0.218} & \textbf{0.198} \\
\apfrow
OSDN-APF      & 0.066 & 0.098 & \textbf{0.292} & 0.152 & \textbf{0.076} & 0.215 & 0.307 & 0.199 & 0.176 \\
\hdashline
GDN            & 0.088 & 0.134 & 0.286 & 0.169 & 0.035 & 0.180 & 0.203 & 0.139 & 0.154 \\
\osdnrow
OSGDN         & \textbf{0.108} & 0.112 & 0.288 & 0.169 & 0.042 & 0.158 & \textbf{0.385} & 0.195 & 0.182 \\
\apfrow
OSGDN-APF     & 0.095 & \textbf{0.155} & \textbf{0.306} & \textbf{0.185} & \textbf{0.049} & \textbf{0.243} & 0.371 & \textbf{0.221} & \textbf{0.203} \\
\hdashline
KDA            & 0.088 & 0.152 & 0.321 & 0.187 & 0.041 & \textbf{0.224} & 0.184 & 0.150 & 0.168 \\
\osdnrow
OSKDA         & \textbf{0.146} & \textbf{0.180} & \textbf{0.327} & \textbf{0.218} & \textbf{0.063} & 0.180 & 0.156 & 0.133 & 0.175 \\
\apfrow
OSKDA-APF     & 0.112 & 0.159 & 0.301 & 0.191 & 0.057 & 0.203 & \textbf{0.277} & \textbf{0.179} & \textbf{0.185} \\
\bottomrule
\end{tabular}
\end{adjustbox}
\end{table}

Table~\ref{tab:ic_recall} resolves the main recall average into the six
JRT-style cloze splits.  The DeltaNet and GDN blocks show the clearest
effect on repeated contexts, especially SWDE-twice and SQuAD-twice,
where the preconditioned rows get a second opportunity to calibrate
recurring key directions.  In the KDA block, OSKDA mainly improves
single-pass recall, while OSKDA-APF is the variant that recovers the
repeated-context average.

\subsection{Mechanism diagnostic: residual contraction with task-wise breakdown}

\begin{table}[htbp]
\centering\scriptsize
\setlength{\tabcolsep}{4.2pt}
\renewcommand{\arraystretch}{1.12}
\caption{\textbf{Direct mechanism diagnostic on repeated-recall prompts,
both scales.}  Lower residual ratios indicate stronger per-token
contraction of the inner Delta-rule regression loss.
\(q_{\mathrm{geo}}\) is the geometric mean of \(q_t = f_t(S_t)/f_t(S_{t-1})\);
\(q_{\mathrm{arith}}\) is the arithmetic mean.  The last three columns
report task-wise \(q_{\mathrm{geo}}\).  Boldface marks the best value
within each scale block.}
\label{tab:mechanism_residual_ratio}
\begin{adjustbox}{max width=\linewidth}
\begin{tabular}{l|cc|ccc}
\toprule
\textbf{Model} &
\(\boldsymbol{q_{\mathrm{geo}}}\downarrow\) &
\(\boldsymbol{q_{\mathrm{arith}}}\downarrow\) &
\shortstack[c]{\textbf{FDA-tw.}\\[-0.15ex]\tiny \(q_{\mathrm{geo}}\downarrow\)} &
\shortstack[c]{\textbf{SWDE-tw.}\\[-0.15ex]\tiny \(q_{\mathrm{geo}}\downarrow\)} &
\shortstack[c]{\textbf{SQuAD-tw.}\\[-0.15ex]\tiny \(q_{\mathrm{geo}}\downarrow\)} \\
\midrule
DeltaNet           & 0.537 & 0.636 & 0.555 & 0.490 & 0.592 \\
\osdnrow
OSDN              & 0.433 & \textbf{0.566} & 0.453 & 0.385 & 0.487 \\
\apfrow
OSDN-APF          & \textbf{0.425} & 0.573 & \textbf{0.452} & \textbf{0.378} & \textbf{0.459} \\
\hdashline
DeltaNet \textit{\scriptsize(1.3B)} & 0.432 & 0.546 & 0.455 & 0.387 & 0.473 \\
\apfrow
\apf{} \textit{\scriptsize(1.3B)}    & \textbf{0.265} & \textbf{0.484} & \textbf{0.375} & \textbf{0.309} & \textbf{0.102} \\
\bottomrule
\end{tabular}
\end{adjustbox}
\end{table}

Table~\ref{tab:mechanism_residual_ratio} is the direct diagnostic behind
the repeated-recall interpretation.  At 340M, both \name{} and \apf{}
reduce the geometric residual ratio on every repeated-recall task, with
\apf{} giving the lowest overall \(q_{\mathrm{geo}}\).  The 1.3B pair
keeps the same direction and strengthens it, most visibly on SQuAD-twice,
where the residual ratio drops from 0.473 to 0.102.

\subsection{Commonsense reasoning, full task breakdown}

\begin{table}[htbp]
\centering\scriptsize
\setlength{\tabcolsep}{3.4pt}
\renewcommand{\arraystretch}{1.13}
\caption{\textbf{Commonsense and short-context language understanding,
full matched 340M sweep.}  Common.\ averages zero-shot PIQA, HellaSwag,
WinoGrande, ARC-Easy, ARC-Challenge, SIQA, BoolQ, and LAMBADA accuracy
($\uparrow$).  PIQA, HSwag, ARC-E, and ARC-C use normalised accuracy;
the other tasks use accuracy.  The no-APF \name{} row reports the best
completed no-APF cell for each task before averaging.}
\label{tab:commonsense_breakdown}
\begin{adjustbox}{max width=\linewidth}
\begin{tabular}{l|cccccccc|c}
\toprule
\textbf{Model} &
\shortstack[c]{\textbf{PIQA}\\[-0.15ex]\tiny \(\uparrow\)} &
\shortstack[c]{\textbf{HSwag}\\[-0.15ex]\tiny \(\uparrow\)} &
\shortstack[c]{\textbf{WinoG.}\\[-0.15ex]\tiny \(\uparrow\)} &
\shortstack[c]{\textbf{ARC-E}\\[-0.15ex]\tiny \(\uparrow\)} &
\shortstack[c]{\textbf{ARC-C}\\[-0.15ex]\tiny \(\uparrow\)} &
\shortstack[c]{\textbf{SIQA}\\[-0.15ex]\tiny \(\uparrow\)} &
\shortstack[c]{\textbf{BoolQ}\\[-0.15ex]\tiny \(\uparrow\)} &
\shortstack[c]{\textbf{LAMB.}\\[-0.15ex]\tiny \(\uparrow\)} &
\shortstack[c]{\textbf{Avg.}\\[-0.15ex]\tiny \(\uparrow\)} \\
\midrule
DeltaNet             & 0.650 & 0.391 & \textbf{0.519} & \textbf{0.526} & 0.262 & 0.381 & \textbf{0.611} & 0.312 & \textbf{0.457} \\
\osdnrow
OSDN                & \textbf{0.662} & 0.393 & \textbf{0.520} & 0.498 & 0.281 & \textbf{0.385} & 0.591 & \textbf{0.320} & 0.456 \\
\apfrow
OSDN-APF            & 0.651 & \textbf{0.399} & 0.516 & 0.506 & \textbf{0.286} & 0.381 & 0.594 & \textbf{0.319} & 0.456 \\
\hdashline
GDN                  & 0.666 & \textbf{0.410} & 0.507 & \textbf{0.520} & \textbf{0.282} & 0.389 & 0.601 & 0.330 & \textbf{0.463} \\
\osdnrow
OSGDN               & \textbf{0.669} & 0.405 & 0.507 & 0.516 & 0.276 & \textbf{0.390} & \textbf{0.607} & \textbf{0.332} & 0.463 \\
\apfrow
OSGDN-APF           & 0.647 & 0.408 & \textbf{0.511} & 0.513 & 0.281 & 0.378 & 0.595 & 0.330 & 0.458 \\
\hdashline
KDA                  & 0.666 & 0.416 & 0.520 & 0.525 & 0.275 & \textbf{0.393} & 0.610 & 0.356 & 0.470 \\
\osdnrow
OSKDA               & 0.668 & \textbf{0.418} & \textbf{0.538} & \textbf{0.541} & \textbf{0.288} & 0.390 & 0.558 & \textbf{0.358} & 0.470 \\
\apfrow
OSKDA-APF           & \textbf{0.681} & 0.417 & 0.532 & 0.531 & 0.283 & 0.378 & \textbf{0.611} & 0.351 & \textbf{0.473} \\
\bottomrule
\end{tabular}
\end{adjustbox}
\end{table}

Table~\ref{tab:commonsense_breakdown} gives the per-task view behind the
Common.\ average.  The row-wise differences are small and the winning
cells are mixed across tasks, which is why we treat commonsense accuracy
as a scope check rather than as the headline claim.  The strongest broad
average in this table is still a KDA-family row, while the DeltaNet and
GDN online-scaled variants stay essentially at parity with their matched
hosts.

\subsection{Short-context, retrieval-LM-eval, and length-extrapolation checks}

\begin{table}[htbp]
\centering\small
\setlength{\tabcolsep}{6pt}
\renewcommand{\arraystretch}{1.16}
\caption{\textbf{General benchmark checks at matched 340M scale.} GDN denotes
Gated DeltaNet.  Higher is better for commonsense, recall, and LongBench; lower is
better for PG-19 PPL.  OSKDA-APF is the data-dependent preconditioner-forgetting variant.
Boldface marks the best value within each dashed block.}
\label{tab:general_checks}
\begin{tabular}{lcccc}
\toprule
\textbf{Model} & \textbf{Common.}\(\uparrow\) & \textbf{Recall}\(\uparrow\) & \textbf{LongBench}\(\uparrow\) & \textbf{PG-19 PPL}\(\downarrow\) \\
\midrule
DeltaNet                   & \textbf{0.457} & 0.150 & 0.072 & 20.78 \\
\osdnrow
\name{}                    & 0.456 & \textbf{0.198} & \textbf{0.087} & 20.02 \\
\apfrow
\apf                       & 0.456 & 0.176 & 0.073 & \textbf{19.85} \\
\hdashline
GDN                        & \textbf{0.463} & 0.154 & 0.073 & 20.11 \\
\osdnrow
OSGDN                     & \textbf{0.463} & 0.182 & 0.073 & \textbf{19.70} \\
\apfrow
OSGDN-APF                 & 0.458 & \textbf{0.203} & \textbf{0.080} & 20.21 \\
\hdashline
KDA                        & 0.470 & 0.168 & 0.088 & 18.73 \\
\osdnrow
OSKDA                     & 0.470 & 0.175 & 0.090 & \textbf{18.53} \\
\apfrow
OSKDA-APF                 & \textbf{0.473} & \textbf{0.185} & \textbf{0.098} & 19.00 \\
\bottomrule
\end{tabular}
\end{table}

Table~\ref{tab:general_checks} compresses the appendix into the same
axes used in the main text.  The targeted pattern is visible across all
three host families: online scaling improves recall in each block, while
the broader commonsense and LongBench columns remain close to the
matched baselines.  APF is most useful when repeated recall or long,
non-stationary contexts matter, whereas vanilla online scaling is often
the cleaner PG-19 variant within a host block.

\begin{figure}[htbp]
\centering
\includegraphics[width=\linewidth]{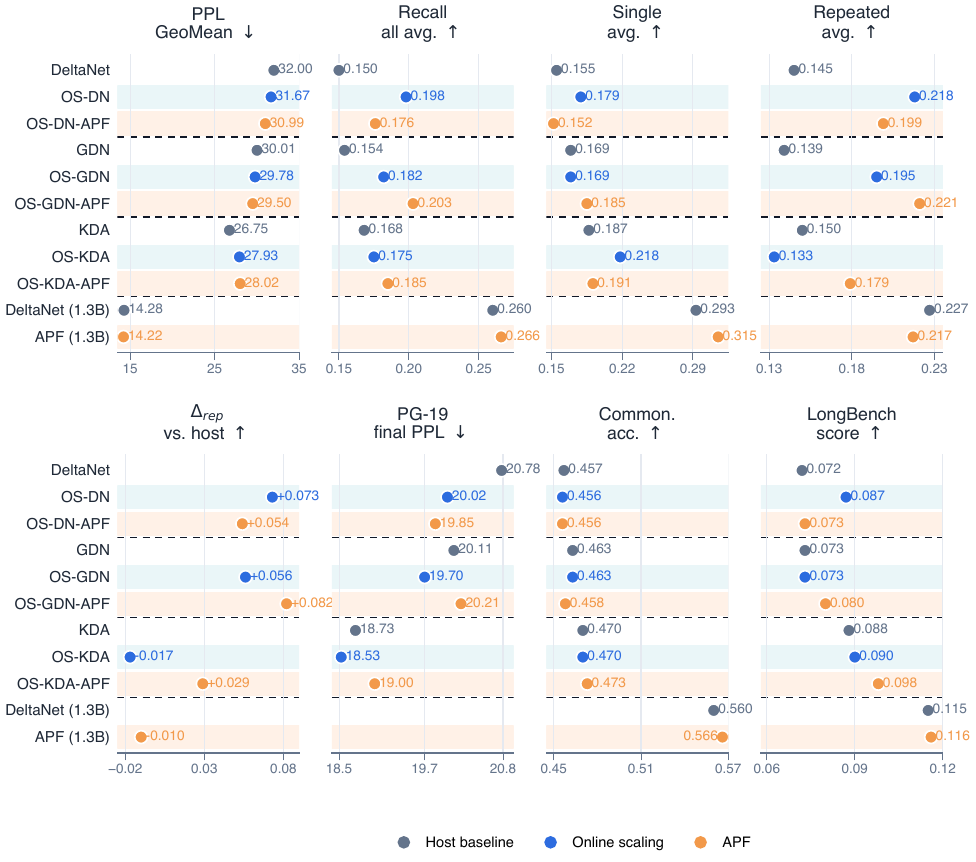}
\caption{\textbf{Visual summary across matched 340M rows
and the 1.3B DeltaNet scale-up.} This is the visual companion to
Table~\ref{tab:main_results}. GDN denotes Gated DeltaNet. PPL is the
WikiText/LAMBADA GeoMean; recall is \textit{contains} accuracy at 2K context;
``single'' averages FDA, SWDE, and SQuAD, while ``repeated'' averages the
corresponding \emph{-twice} variants. \(\Delta_{\mathrm{rep}}\) is the
repeated-recall lift relative to the matched host baseline. PG-19 is the
20K-token length-extrapolation perplexity (lower is better); the 1.3B
PG-19 cells are blank, consistent with the
Appendix~\ref{sec:app_scale_1p3b} reporting protocol. Common.\ averages
PIQA, HellaSwag, WinoGrande, ARC-E/-C, SIQA, BoolQ, and LAMBADA.
LongBench averages 14 tasks.}
\label{fig:aggregate_summary_app}
\end{figure}

Figure~\ref{fig:aggregate_summary_app} is included only as a visual
index to the same numbers, so the table values remain the authoritative
source for comparisons.  Its purpose is to make the separation between
targeted recall gains, broader benchmark parity, and the 1.3B
mechanism-level contraction easier to scan.

\begin{table}[tbh]
\centering\small
\setlength{\tabcolsep}{7pt}
\renewcommand{\arraystretch}{1.12}
\caption{\textbf{FW-Edu checkpoint screen.}  Perplexity is evaluated on the
fixed FineWeb-Edu sample-10BT slice (\texttt{train[-10000:]}, 10.0M labels).
The DeltaNet row is the retrained 4-GPU fair baseline; headline GDN elsewhere
corresponds to GDN v2.}
\label{tab:fwedu_screen}
\begin{tabular}{lc}
\toprule
\textbf{Model} & \textbf{FW-Edu}\(\downarrow\) \\
\midrule
DeltaNet      & 12.43 \\
DeltaNet+gate & 12.54 \\
GDN v2        & 11.97 \\
\osdnrow
OSGDN        & 12.04 \\
\apfrow
OSGDN-APF    & 12.07 \\
KDA           & \textbf{11.38} \\
\osdnrow
\name{}       & 12.32 \\
\apfrow
\apf{}        & 12.39 \\
\bottomrule
\end{tabular}
\end{table}

Table~\ref{tab:fwedu_screen} records the in-domain validation check used
during the checkpoint screen.  KDA and GDN v2 have the lowest FW-Edu
perplexities among these rows, but the refreshed vanilla \name{} run
also improves over the matched DeltaNet baseline.  OSGDN and OSGDN-APF
stay close to GDN v2 on this in-domain slice, so their downstream
differences are not explained by a large validation-perplexity gap.

\begin{table}[htbp]
\centering\small
\setlength{\tabcolsep}{7pt}
\renewcommand{\arraystretch}{1.12}
\caption{\textbf{Retrieval LM-eval checks.} DROP reports F1;
NQ-Open and TriviaQA report exact match after whitespace normalization
(\(\uparrow\)).  Boldface marks the best value within each dashed block.}
\label{tab:oskda_retrieval_lmeval}
\begin{tabular}{l|ccc|c}
\toprule
\textbf{Model} & \textbf{DROP F1} & \textbf{NQ-Open EM} & \textbf{TriviaQA EM} & \textbf{Avg.} \\
\midrule
DeltaNet                   & 0.029 & \textbf{0.016} & 0.003 & 0.016 \\
\osdnrow
\name{}                    & 0.027 & \textbf{0.018} & \textbf{0.004} & \textbf{0.016} \\
\apfrow
\apf{}                     & \textbf{0.031} & 0.013 & \textbf{0.004} & \textbf{0.016} \\
\hdashline
GDN                        & 0.027 & 0.011 & 0.005 & 0.015 \\
\apfrow
OSGDN-APF                 & \textbf{0.028} & \textbf{0.012} & \textbf{0.007} & \textbf{0.016} \\
\hdashline
KDA                        & 0.024 & \textbf{0.024} & 0.005 & \textbf{0.018} \\
\osdnrow
OSKDA                     & 0.024 & 0.012 & \textbf{0.005} & 0.014 \\
\apfrow
OSKDA-APF                 & \textbf{0.032} & 0.009 & 0.003 & 0.014 \\
\bottomrule
\end{tabular}
\end{table}

Table~\ref{tab:oskda_retrieval_lmeval} checks whether the recall gains
also appear on standard retrieval LM-eval tasks.  The absolute scores are
low at this scale and the effects are small: DeltaNet-family online
scaling ties the rounded average, OSGDN-APF slightly improves the GDN
average, and KDA remains the strongest retrieval LM-eval baseline.  We
therefore keep the paper's retrieval claim tied to the controlled
JRT-style cloze setting rather than to these open-domain QA probes.

Taken together, these breakdowns support a deliberately narrow reading of
the 340M sweep.  Online scaling consistently helps the controlled recall
and residual-contraction measurements, but the broader commonsense,
perplexity, LongBench, and retrieval LM-eval checks remain host-dependent.
We therefore use these tables to bound the claim: \name{} is a targeted
associative-retrieval mechanism, not a universal broad-benchmark lift.

\section{PG-19 Length Extrapolation Details}
\label{sec:app_pg19_details}

\begin{figure}[htbp]
\centering
\includegraphics[width=\linewidth]{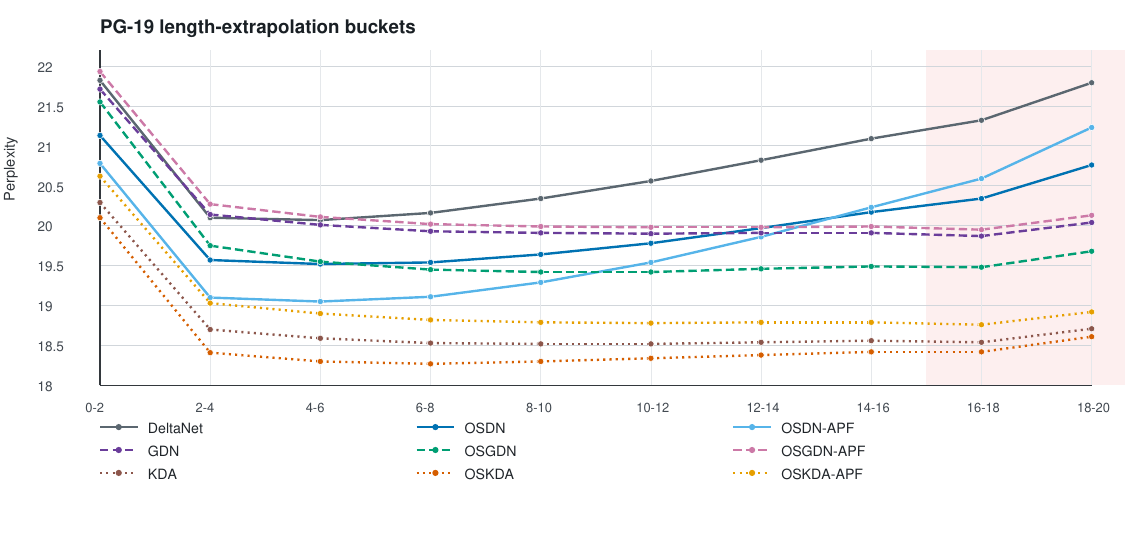}
\caption{\textbf{Auxiliary PG-19 length-extrapolation diagnostic.}
Each GPU consumes a 65{,}536-token packed batch, but training segments are
variable-length FineWeb-Edu documents capped at \(4\)K tokens with
\texttt{cu\_seqlens}-aligned recurrent-state resets at every segment
boundary, so the effective recurrent training context is at most \(4\)K
tokens.  Models are then evaluated on 20K-token PG-19 blocks, well beyond
any contiguous segment seen in training.  Curves show perplexity by
2K-token position bucket; lower is better.  The red background marks the
late buckets used to diagnose drift.  This book-scale perplexity check is
reported as appendix evidence rather than as a main retrieval benchmark.}
\label{fig:pg19_buckets}
\end{figure}

\begin{table}[htbp]
\centering\small
\setlength{\tabcolsep}{3pt}
\caption{\textbf{PG-19 perplexity by 2K-token bucket} (\(\downarrow\)). Final
PPL is the corpus-level evaluation output; \(\Delta\) is \(18\)--\(20\)K minus
\(2\)--\(4\)K and measures late-context drift.  Boldface marks the best value
within each dashed block.}
\label{tab:pg19_buckets}
\scriptsize
\begin{adjustbox}{max width=\linewidth}
\begin{tabular}{lcccccccccc|cc}
\toprule
\textbf{Bucket (K)} & 0--2 & 2--4 & 4--6 & 6--8 & 8--10 & 10--12 & 12--14 & 14--16 & 16--18 & 18--20 & \textbf{Final} & \(\boldsymbol\Delta\) \\
\midrule
DeltaNet        & 21.82 & 20.10 & 20.07 & 20.16 & 20.34 & 20.56 & 20.82 & 21.09 & 21.32 & 21.79 & 20.78 & \(+1.69\) \\
\osdnrow
\name{}         & 21.13 & 19.57 & 19.52 & 19.54 & 19.64 & 19.78 & 19.97 & \textbf{20.17} & \textbf{20.34} & \textbf{20.76} & 20.02 & \(\mathbf{+1.19}\) \\
\apfrow
\apf            & \textbf{20.78} & \textbf{19.10} & \textbf{19.05} & \textbf{19.11} & \textbf{19.29} & \textbf{19.54} & \textbf{19.86} & 20.23 & 20.59 & 21.23 & \textbf{19.85} & \(+2.13\) \\
\hdashline
GDN             & 21.71 & 20.14 & 20.01 & 19.93 & 19.91 & 19.90 & 19.91 & 19.91 & 19.87 & 20.04 & 20.11 & \(-0.10\) \\
\osdnrow
OSGDN          & \textbf{21.55} & \textbf{19.75} & \textbf{19.55} & \textbf{19.45} & \textbf{19.42} & \textbf{19.42} & \textbf{19.46} & \textbf{19.49} & \textbf{19.48} & \textbf{19.68} & \textbf{19.70} & \(-0.07\) \\
\apfrow
OSGDN-APF      & 21.93 & 20.27 & 20.11 & 20.02 & 19.99 & 19.98 & 19.98 & 19.99 & 19.95 & 20.13 & 20.21 & \(\mathbf{-0.14}\) \\
\hdashline
KDA             & 20.29 & 18.70 & 18.59 & 18.53 & 18.52 & 18.52 & 18.54 & 18.56 & 18.54 & 18.71 & 18.73 & \(+0.01\) \\
\osdnrow
OSKDA          & \textbf{20.10} & \textbf{18.41} & \textbf{18.30} & \textbf{18.27} & \textbf{18.30} & \textbf{18.34} & \textbf{18.38} & \textbf{18.42} & \textbf{18.42} & \textbf{18.61} & \textbf{18.53} & \(+0.20\) \\
\apfrow
OSKDA-APF      & 20.62 & 19.03 & 18.90 & 18.82 & 18.79 & 18.78 & 18.79 & 18.79 & 18.76 & 18.92 & 19.00 & \(\mathbf{-0.11}\) \\
\bottomrule
\end{tabular}
\end{adjustbox}
\end{table}

\(\Delta\) is not the headline metric because the gated baselines have the
smallest gaps.  Its role is diagnostic: the refreshed no-APF \name{} run lowers
DeltaNet-row PG-19 final perplexity from 20.78 to 20.02 and gives the smallest
DeltaNet-row late-bucket drift, while \apf{} still gives the lowest absolute
final perplexity at 19.85.  The bounded OSGDN run improves the
GDN final PG-19 perplexity from 20.11 to 19.70 while keeping a small negative
late-bucket drift; OSGDN-APF finishes slightly above GDN at 20.21 but gives the
smallest GDN-row late-bucket drift at \(-0.14\).  The KDA extension shows a
different tradeoff: the refreshed bounded OSKDA no-DD run gives the strongest
absolute PG-19 curve in the block, improving the KDA final perplexity from
18.73 to 18.53, while OSKDA-APF trades that absolute PPL for negative
late-bucket drift at \(-0.11\).

\section{OSDN Variant Ablations}
\label{sec:app_variant_ablations}

\begin{table}[htbp]
\centering\small
\setlength{\tabcolsep}{6pt}
\caption{\textbf{Ablation over \name{} variants on the DeltaNet backbone.}
Boldface marks the best value in each column.}
\label{tab:ablation}
\begin{tabular}{l|cccc|c}
\toprule
\textbf{Variant} & \textbf{Common.}\(\uparrow\) & \textbf{Recall}\(\uparrow\) & \textbf{LongBench}\(\uparrow\) & \textbf{PG-19 PPL}\(\downarrow\) & \(\boldsymbol\Delta\)\(\downarrow\) \\
\midrule
DeltaNet baseline      & 0.457 & 0.150 & 0.072 & 20.78 & \(+1.69\) \\
\midrule
\name{}                & 0.456 & \textbf{0.198} & \textbf{0.087} & 20.02 & \textbf{\(+1.19\)} \\
\name{}+\(D_0\)        & 0.459 & 0.172 & 0.069 & 21.57 & \(+5.34\) \\
\name{}+\(D_0^\star\)  & \textbf{0.461} & 0.164 & 0.063 & 21.42 & \(+4.25\) \\
\apf                   & 0.454 & 0.176 & 0.073 & \textbf{19.85} & \(+2.13\) \\
\bottomrule
\end{tabular}
\end{table}

These ablations are diagnostic rather than headline methods. Learning $D_0$ helps commonsense slightly but does not fix length drift; relaxing the $D_0$ projection improves commonsense further but weakens retrieval. The refreshed no-APF \name{} screen improves the rounded commonsense and LongBench averages relative to the original full-evaluation checkpoint and also reduces PG-19 drift; \apf{} still gives the lowest DeltaNet-row PG-19 final perplexity.

\section{Training-Loss Diagnostic}
\label{sec:app_training_loss}

\begin{figure}[htbp]
\centering
\includegraphics[width=\linewidth]{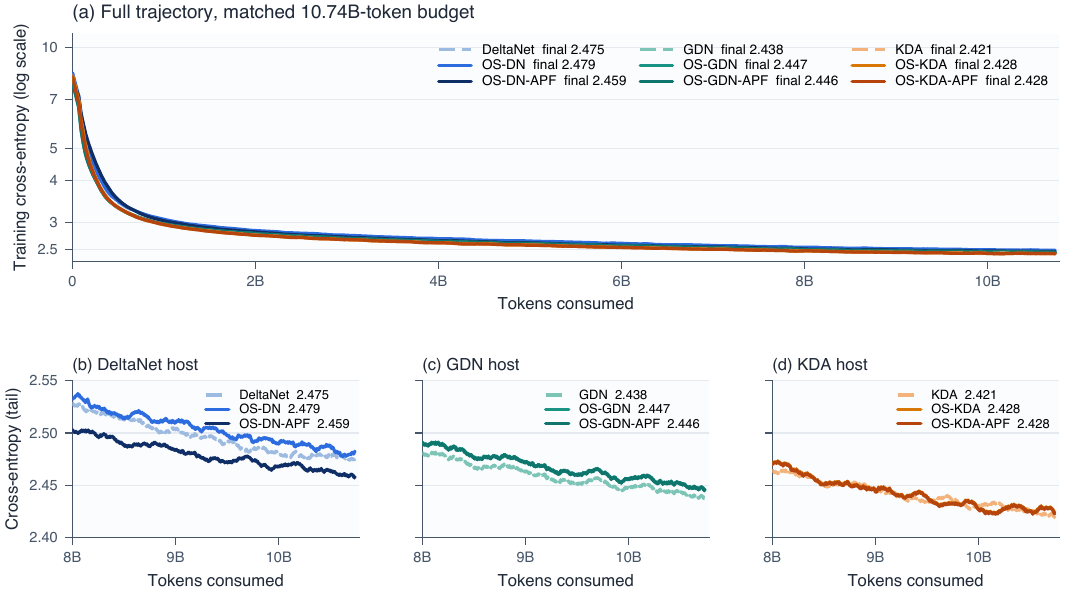}
\caption{\textbf{Token-vs-CE training trajectories at matched 340M scale.}
Each curve consumes the same \(10.74\)B-token budget at \(524{,}288\) tokens
per optimizer step (Appendix~\ref{sec:app_reproducibility}); curves are
rendered as a centred \(257\)-step rolling mean.  Within each panel, the
baseline uses a lighter dashed tone, online-scaled (OS) variants a
mid-saturation tone, and APF variants the deepest tone.
Panel~(a) plots the full trajectory on a logarithmic \(y\)-axis.
Panels~(b)/(c)/(d) zoom into the final \(25\%\) of the budget on a linear
scale, separated by baseline.  Final CE (mean over the last \(256\)
steps) lies in \([2.421, 2.479]\) across all plotted completed
configurations, a spread of \(0.058\) nats.}
\label{fig:training_loss}
\end{figure}

Figure~\ref{fig:training_loss} reports training cross-entropy for the
matched 340M sweep on the FineWeb-Edu sample-10BT slice; the
matched-budget protocol is enforced in Appendix~\ref{sec:app_reproducibility}
rather than diagnosed from this figure.  The plotted completed final-CE
values fall in a \(0.058\)-nat band; within each dashed group the
baseline and its online-scaled variants finish within a few hundredths of
a nat of each other---a separation that is small relative to single-step CE
fluctuation on this corpus.  We therefore do not read training CE as a
mechanism diagnostic for online preconditioning, and do not draw any
within-group ranking from it.  The mechanism-level claims in
Section~\ref{subsec:exp_mechanism} are stated on the per-token regression
residual ratio \(q_t = f_t(S_t)/f_t(S_{t-1})\), with
\(f_t(S_{t-1})=\tfrac12\|S_{t-1}k_t-v_t\|^2\), which is the quantity
controlled by Theorem~\ref{thm:superlinear} and specialised to the
no-conflict repeated-key regime by Corollary~\ref{cor:repeated_keys}; the
downstream evidence is the JRT-style repeated-recall gains and the PG-19
length-bucket profile reported in Section~\ref{sec:exp}.

\section{1.3B / 100B Scaling Breakdown}
\label{sec:app_scale_1p3b}

This appendix records the per-task breakdowns underlying the 1.3B / 100B
scaling check in Section~\ref{subsec:exp_scale}. The optimizer schedule,
training data, evaluation protocols, and tokenizer match the 340M sweep; the
scale-specific settings (architecture, token budget, hardware) are summarised
at the end of Appendix~\ref{sec:app_reproducibility}.

\paragraph{Language modelling perplexity.}
Table~\ref{tab:scale_1p3b_ppl} reports the same WikiText, LAMBADA, and
FineWeb-Edu validation sources as the 340M Table~\ref{tab:model_ppl}; the
1.3B FW-Edu rows use the matched packed \texttt{train[-10000:]} slice (11.7M
labels).  PG-19 length-extrapolation perplexity is not part of the
1.3B / 100B reporting protocol because the full 20K-token sweep did not
complete on a matched harness for both rows; it is therefore omitted from
this scale's perplexity summary.

\begin{table}[H]
\centering\small
\setlength{\tabcolsep}{6pt}
\renewcommand{\arraystretch}{1.15}
\caption{\textbf{1.3B / 100B perplexity summary.} Lower is better.  GeoMean
is computed over the completed zero-shot PPL columns: WikiText and LAMBADA.
\(\Delta\mathrm{NLL}\) is the log-GeoMean difference relative to the
DeltaNet baseline at this scale.}
\label{tab:scale_1p3b_ppl}
\begin{adjustbox}{max width=\linewidth}
\begin{tabular}{l|c|cc|cc}
\toprule
\textbf{Model} &
\shortstack[c]{\textbf{FW-Edu val}\\[-0.15ex]\tiny PPL \(\downarrow\)} &
\shortstack[c]{\textbf{WikiText}\\[-0.15ex]\tiny PPL \(\downarrow\)} &
\shortstack[c]{\textbf{LAMBADA}\\[-0.15ex]\tiny PPL \(\downarrow\)} &
\shortstack[c]{\textbf{GeoMean}\\[-0.15ex]\tiny PPL \(\downarrow\)} &
\shortstack[c]{\(\boldsymbol{\Delta}\mathrm{NLL}\)\\[-0.15ex]\tiny \(\downarrow\)} \\
\midrule
DeltaNet & 8.71 & 17.23 & 11.84 & 14.28 & -- \\
\apfrow
\apf{}   & 9.10 & 18.42 & 10.98 & 14.22 & \(-0.004\) \\
\bottomrule
\end{tabular}
\end{adjustbox}
\end{table}

\paragraph{In-context recall.}
Table~\ref{tab:scale_1p3b_recall} expands the JRT-style cloze breakdown,
mirroring the 340M Table~\ref{tab:ic_recall}.  At this scale, the SQuAD and
SQuAD-twice splits did not produce a usable JRT \textit{contains}-accuracy
signal under the matched evaluation harness (returned values were within
generation-failure noise rather than tracking the model's recall on the
prompt) and are therefore excluded from this table; the 1.3B recall
discussion is restricted to FDA and SWDE, where the harness completed cleanly
on both checkpoints.

\begin{table}[H]
\centering\scriptsize
\setlength{\tabcolsep}{3.6pt}
\renewcommand{\arraystretch}{1.13}
\caption{\textbf{1.3B / 100B in-context recall, JRT-style cloze at 2K
context.}  Reported in \textit{contains} accuracy ($\uparrow$).  Single
averages FDA and SWDE; Repeated averages their \emph{-twice} variants;
Overall averages all four task accuracies.  SQuAD splits are omitted because
the evaluation harness did not return reliable contains-accuracy values for
them at this scale.}
\label{tab:scale_1p3b_recall}
\begin{adjustbox}{max width=\linewidth}
\begin{tabular}{l|cc|c|cc|cc}
\toprule
\textbf{Model} &
\shortstack[c]{\textbf{FDA}\\[-0.15ex]\tiny contains acc. \(\uparrow\)} &
\shortstack[c]{\textbf{SWDE}\\[-0.15ex]\tiny contains acc. \(\uparrow\)} &
\shortstack[c]{\textbf{Single}\\[-0.15ex]\tiny avg. \(\uparrow\)} &
\shortstack[c]{\textbf{FDA-tw.}\\[-0.15ex]\tiny contains acc. \(\uparrow\)} &
\shortstack[c]{\textbf{SWDE-tw.}\\[-0.15ex]\tiny contains acc. \(\uparrow\)} &
\shortstack[c]{\textbf{Repeated}\\[-0.15ex]\tiny avg. \(\uparrow\)} &
\shortstack[c]{\textbf{Overall}\\[-0.15ex]\tiny avg. \(\uparrow\)} \\
\midrule
DeltaNet & 0.215 & 0.371 & 0.293 & 0.061 & 0.392 & 0.227 & 0.260 \\
\apfrow
\apf{}   & \textbf{0.241} & \textbf{0.388} & \textbf{0.315} & \textbf{0.073} & 0.360 & 0.217 & \textbf{0.266} \\
\bottomrule
\end{tabular}
\end{adjustbox}
\end{table}

\paragraph{Commonsense reasoning.}
Table~\ref{tab:scale_1p3b_commonsense} expands the eight-task commonsense
average, mirroring the 340M Table~\ref{tab:commonsense_breakdown}.

\begin{table}[H]
\centering\scriptsize
\setlength{\tabcolsep}{3.4pt}
\renewcommand{\arraystretch}{1.13}
\caption{\textbf{1.3B / 100B commonsense and short-context language
understanding.}  The Avg.\ column averages zero-shot PIQA, HellaSwag,
WinoGrande, ARC-Easy, ARC-Challenge, SIQA, BoolQ, and LAMBADA accuracy
($\uparrow$). PIQA, HSwag, ARC-E, and ARC-C use normalised accuracy; the other
tasks use accuracy.}
\label{tab:scale_1p3b_commonsense}
\begin{adjustbox}{max width=\linewidth}
\begin{tabular}{l|cccccccc|c}
\toprule
\textbf{Model} &
\shortstack[c]{\textbf{PIQA}\\[-0.15ex]\tiny norm acc. \(\uparrow\)} &
\shortstack[c]{\textbf{HSwag}\\[-0.15ex]\tiny norm acc. \(\uparrow\)} &
\shortstack[c]{\textbf{WinoG.}\\[-0.15ex]\tiny acc. \(\uparrow\)} &
\shortstack[c]{\textbf{ARC-E}\\[-0.15ex]\tiny norm acc. \(\uparrow\)} &
\shortstack[c]{\textbf{ARC-C}\\[-0.15ex]\tiny norm acc. \(\uparrow\)} &
\shortstack[c]{\textbf{SIQA}\\[-0.15ex]\tiny acc. \(\uparrow\)} &
\shortstack[c]{\textbf{BoolQ}\\[-0.15ex]\tiny acc. \(\uparrow\)} &
\shortstack[c]{\textbf{LAMB.}\\[-0.15ex]\tiny acc. \(\uparrow\)} &
\shortstack[c]{\textbf{Avg.}\\[-0.15ex]\tiny acc. \(\uparrow\)} \\
\midrule
DeltaNet & 0.727 & 0.581 & 0.598 & 0.687 & 0.418 & 0.415 & 0.579 & 0.473 & 0.560 \\
\apfrow
\apf{}   & 0.733 & 0.592 & 0.606 & 0.690 & 0.404 & 0.417 & 0.605 & 0.480 & 0.566 \\
\bottomrule
\end{tabular}
\end{adjustbox}
\end{table}

\paragraph{LongBench.}
Table~\ref{tab:scale_1p3b_longbench} reports the LongBench English 14-task
average. The per-task breakdown is omitted to match the 340M reporting
convention.

\begin{table}[H]
\centering\small
\setlength{\tabcolsep}{8pt}
\renewcommand{\arraystretch}{1.12}
\caption{\textbf{1.3B / 100B LongBench English 14-task average.}}
\label{tab:scale_1p3b_longbench}
\begin{tabular}{lc}
\toprule
\textbf{Model} & \textbf{LongBench} \(\uparrow\) \\
\midrule
DeltaNet & 0.115 \\
\apfrow
\apf{}   & 0.116 \\
\bottomrule
\end{tabular}
\end{table}

\paragraph{Reading the 1.3B per-task tables.}
The per-token residual-ratio diagnostic of
Section~\ref{subsec:exp_mechanism} was rerun on the 1.3B / 100B checkpoints
under the same repeated-recall prompt protocol; the resulting numbers are
reported alongside the 340M rows in
Table~\ref{tab:mechanism_residual_ratio}, where \apf{} achieves the lowest
\(q_{\mathrm{geo}}\) recorded across either scale.  The per-task tables in
this appendix complement that diagnostic by making the underlying
language-modelling and capability axes explicit: \apf{} improves single-pass
FDA / SWDE recall from 0.293 to 0.315 and LAMBADA perplexity from 11.84 to
10.98, while the WikiText / LAMBADA GeoMean (14.28 vs.\ 14.22), the
eight-task commonsense average (0.560 vs.\ 0.566), the LongBench English
average (0.115 vs.\ 0.116), and repeated FDA / SWDE recall (0.227 vs.\
0.217) all stay at parity with the matched DeltaNet baseline.  Together,
the appendix tables and the main mechanism table support the same reading:
at 1.3B / 100B the OSDN-APF residual-ratio contraction transfers and
continues to amplify relative to the 340M sweep, while downstream
language-modelling and capability averages stay at parity with DeltaNet.

\section{Inference Throughput Protocol and Results}
\label{sec:app_throughput}

This appendix records the throughput numbers for all matched 340M
variants and the 1.3B scaling pair, the measurement protocol, the
recurrent-state size formula, and the reference single-token kernel used
for the throughput summary referenced in the main text.

\begin{table}[H]
\centering\scriptsize
\setlength{\tabcolsep}{4pt}
\renewcommand{\arraystretch}{1.16}
\caption{\textbf{Inference throughput and persistent recurrent-state size
at matched 340M scale.}  Single H100 80GB, \texttt{batch=1},
2{,}048-token prefill + 128-token greedy decode, bfloat16, median of 5
timed repeats.  $\Delta$ is relative to the baseline within each dashed
group.  KV/state is the per-sequence persistent
recurrent-state size (excludes weights, activations, and short-conv
cache); the OSGM diagonal vector adds $\le 0.05\%$ to the recurrent state
size.}
\label{tab:throughput}
\begin{tabular}{@{}lcccc@{}}
\toprule
\textbf{Model} &
\shortstack[c]{\textbf{tokens/sec}\\[-0.15ex]\tiny \(\uparrow\)} &
\shortstack[c]{\textbf{decode ms}\\[-0.15ex]\tiny \(\downarrow\)} &
\shortstack[c]{\(\boldsymbol{\Delta}\)\textbf{tokens/sec}\\[-0.15ex]\tiny vs.\ base} &
\shortstack[c]{\textbf{KV/state}\\[-0.15ex]\tiny MiB bf16} \\
\midrule
DeltaNet                   & \textbf{784.5} & \textbf{2735.5} & --              & 6.000 \\
\osdnrow
\name{}                   & 782.7 & 2740.5 & \(\mathbf{-0.2\%}\) & 6.047 \\
\apfrow
\apf{}                    & 767.2 & 2796.2 & \(-2.2\%\)      & 6.047 \\
\hdashline
GDN                        & \textbf{882.9} & \textbf{2429.7} & --              & 7.875 \\
\osdnrow
OSGDN                     & 865.1 & 2469.7 & \(\mathbf{-2.0\%}\) & 7.906 \\
\apfrow
OSGDN-APF                 & 865.4 & 2466.6 & \(\mathbf{-2.0\%}\) & 7.906 \\
\hdashline
KDA                        & 761.1 & 2817.6 & --              & 5.750 \\
\osdnrow
OSKDA                     & \textbf{803.0} & \textbf{2665.9} & \(\mathbf{+5.5\%}\) & 5.795 \\
\apfrow
OSKDA-APF                 & 771.9 & 2774.0 & \(+1.4\%\)      & 5.795 \\
\bottomrule
\end{tabular}
\end{table}

The same script and prompt/decode protocol were also run on the 1.3B / 100B
DeltaNet scaling pair, using a single H100 80GB on \texttt{phoenix2}.  The
\apf{} row remains close to the matched DeltaNet baseline, with a larger
single-GPU slowdown than the 340M DeltaNet-family rows.

\begin{table}[H]
\centering\scriptsize
\setlength{\tabcolsep}{4pt}
\renewcommand{\arraystretch}{1.16}
\caption{\textbf{Inference throughput for the 1.3B / 100B scaling pair.}
Single H100 80GB, \texttt{batch=1}, 2{,}048-token prefill + 128-token greedy
decode, bfloat16, median of 5 timed repeats.  $\Delta$ is relative to the
matched DeltaNet baseline.}
\label{tab:throughput_1p3b}
\begin{tabular}{@{}lcccc@{}}
\toprule
\textbf{Model} &
\shortstack[c]{\textbf{tokens/sec}\\[-0.15ex]\tiny \(\uparrow\)} &
\shortstack[c]{\textbf{decode ms}\\[-0.15ex]\tiny \(\downarrow\)} &
\shortstack[c]{\(\boldsymbol{\Delta}\)\textbf{tokens/sec}\\[-0.15ex]\tiny vs.\ base} &
\shortstack[c]{\textbf{KV/state}\\[-0.15ex]\tiny MiB bf16} \\
\midrule
DeltaNet \textit{\scriptsize(1.3B)} & \textbf{794.0} & \textbf{2702.4} & --         & 12.000 \\
\apfrow
\apf{} \textit{\scriptsize(1.3B)}   & 739.9          & 2899.2          & \(-6.8\%\) & 12.094 \\
\bottomrule
\end{tabular}
\end{table}

The pattern is consistent across the three baselines: every OS-* variant lands
within $\pm 5.5\%$ of its baseline on tokens/sec, and the persistent
recurrent state grows by $\le 0.05\%$ from the OSGM diagonal vector.
Online preconditioning is essentially throughput-neutral at this scale.  At
1.3B, the same APF mechanism adds \(0.094\)~MiB of persistent state and runs
6.8\% slower than the matched DeltaNet checkpoint under the same single-GPU
generation benchmark.

\paragraph{Hardware and protocol.}
Single H100 80GB SXM, bfloat16 weights and activations with float32
reductions inside the recurrent state. \texttt{batch\_size}\,=\,1, prompt
length 2{,}048 tokens, decode length 128 greedy tokens. Prefill is one
forward call; decode is emitted token-by-token through HuggingFace's
\texttt{past\_key\_values} cache so that the recurrent kernel is exercised
realistically. Each variant is loaded from its trained HuggingFace
checkpoint, warmed up twice, then timed five times; medians of
prefill\,/\,decode milliseconds are reported. The reported tokens/sec is
$(2048+128) / (\texttt{prefill\_ms} + \texttt{decode\_ms}) \cdot 10^3$. All
nine variants are measured back-to-back on the same physical GPU.

\paragraph{Persistent recurrent state.}
The KV/state column counts the per-sequence persistent recurrent state
(excluding model weights, activations, allocator overhead, and the
short-convolution cache):
\[
  \texttt{state\_elements}    = L \cdot H \cdot d_k \cdot d_v, \quad
  \texttt{osgm\_d\_elements}  = L \cdot H \cdot d_k \;\;(\text{only if use\_osgm}),
\]
\[
  \texttt{state\_MiB\_bf16}   = (\texttt{state\_elements} + \texttt{osgm\_d\_elements}) \cdot 2 / 1024^2.
\]
At our 340M shape the OSGM diagonal contributes $0.047$~MiB on the
DeltaNet/KDA backbone and $0.031$~MiB on the GDN backbone -- $\le 0.05\%$
of the recurrent state size in every case.

\paragraph{Single-token recurrence kernel for OSGDN.}
At decode time (\texttt{q\_len}~$\le 64$) the OSGDN forward dispatches to a
fused Triton kernel that performs the GDN forget gate
$\alpha_t = \exp(g_t)$, the post-gate-regret OSGM update
\[
  d_{t+1} \;=\; \gamma_t\,d_t \;+\; \eta\,\beta_t\bigl(1 - \beta_t \langle d_t,\,k_t^2 \rangle\bigr)\,k_t^2
\]
with optional clamp into $[d_{\min}, d_{\max}]$, and the rank-1
delta-rule write
\[
  S_t \;=\; \alpha_t\, S_{t-1} \;+\; (k_t \odot d_t) \otimes \beta_t\bigl(v_t - (\alpha_t S_{t-1})^\top k_t\bigr)
\]
inside one kernel launch per token. The kernel covers
\texttt{decay\_mode}~$\in$~\{\texttt{none}, \texttt{data\_dependent}\}; for
the data-dependent variant the decay signal is aliased to the same raw
log-decay $g_t$ that drives GDN's forget gate ($\gamma_t = \alpha_t$).
Both production OSGM-on-GDN configurations (no-DD and APF) use this
dispatch.

\paragraph{Numerical equivalence with the chunk forward.}
We verified the single-token kernel against the chunk training reference
on synthetic inputs ($T=64$, bfloat16 activations, fp32 state). End-to-end
output relative error is $\le 6.5\times 10^{-3}$, final state error is
$\le 6.7\times 10^{-3}$, and final $d$ error is $\le 2.0\times 10^{-3}$,
all within the bfloat16 cross-schedule tolerance for two-mode reduction
(chunk-parallel vs.\ sequential). A layer-level check -- prefill of
$T=128$ followed by a single decode token, compared against a chunk-only
forward of $T=129$ -- gives an end-to-end output relative error
$\le 5.8\times 10^{-3}$ across both \texttt{decay\_mode} settings.

\section{Extended Related Work Taxonomy}
\label{sec:app_related}

\paragraph{Linear attention and state-space models.}
Linear attention~\citep{katharopoulos_transformers_2020,choromanski_rethinking_2021}
replaces the softmax kernel with a feature map, collapsing attention to a
constant-size matrix-valued state $S_t$ updated by an additive Hebbian
recurrence $S_t = S_{t-1} + v_t k_t^\top$, and reducing inference cost
from $\mathcal{O}(N^2)$ to $\mathcal{O}(N)$. The purely additive update
has documented retrieval limitations~\citep{schlag_linear_2021,Sun2023RetentiveNA,arora_zoology_2024,arora_simple_2024}:
writes to the same key direction superpose, and the model lacks a
mechanism to overwrite a stale association. A first remedy is to introduce
a multiplicative \emph{decay} or \emph{gate}, which has produced a rich
family of recurrent linear-attention architectures: RetNet's constant
decay~\citep{Sun2023RetentiveNA}, RWKV's time-mixing channel
decay~\citep{peng_eagle_2024,peng_rwkv7_2025}, GLA's data-dependent
diagonal gate~\citep{yang_gated_2024}, mLSTM/xLSTM~\citep{beck_xlstm_2024},
HGRN-2~\citep{qin_hgrn2_2024}, and GSA~\citep{zhang_gated_2024}. A
parallel line investigates the structured-state-space view: S4 and its
diagonal simplifications~\citep{gu_efficiently_2021,smith_simplified_2022,smith_simplified_2023,orvieto_resurrecting_2023},
the long-convolution Hyena~\citep{poli_hyena_2023}, the input-selective
Mamba~\citep{gu_mamba_2024}, Mamba-2's structured-state-space
duality~\citep{dao_transformers_2024}, and the more expressive
Mamba-3~\citep{gu_mamba3_2026}. \name{} is orthogonal to the choice of
decay or selectivity mechanism: it modifies the \emph{write-scale
geometry} of the rank-one update without changing the surrounding
architecture, and the chunkwise UT-transform in
Section~\ref{subsec:recurrent_chunk} preserves the SSD/WY computation
pipeline shared by these models.

\paragraph{The delta rule and fast-weight programmers.}
\name{} belongs to the lineage of \emph{fast-weight
programmers}~\citep{schmidhuber_learning_1992,ba_using_2016,schlag_linear_2021},
which view a sequence layer as a network whose ``fast'' weights are
written and read by the slow network on the fly.
\citet{schlag_linear_2021} formalised this view for linear attention by
showing that DeltaNet's read-then-write update is exactly an online
gradient step on a per-token regression loss with scalar learning rate
$\beta_t$. Subsequent work has improved the parallelism, expressivity, or
gating of this update: \citet{yang_parallelizing_2024} introduce a
chunkwise WY parallelisation that closes the speed gap with softmax
attention; \citet{yang_gated_2024-1} add a head-wise data-dependent forget
gate (Gated DeltaNet); \citet{kimiteam2025kimilinearexpressiveefficient}
replace the scalar gate with a fine-grained per-channel decay and
demonstrate strong long-context retrieval (KDA / Kimi Linear);
\citet{siems_deltaproduct_2025} stack Householder transitions to raise
per-step expressivity (DeltaProduct); and RWKV-7~\citep{peng_rwkv7_2025}
arrives at a closely related diagonal-plus-low-rank transition.
\citet{liu_longhorn_2024} take a complementary route, deriving the
delta-rule update as the closed-form solution of an instantaneous
quadratic regression. \name{} differs from these in a single, focused way:
prior work either keeps the write step scalar ($\beta_t$, $\alpha_t$) or
solves the per-token regression in closed form, whereas \name{} retains
$\beta_t$ and learns a per-feature multiplier $d_t \in \mathbb{R}^K$ via
online optimisation on the next-step loss; this is shown to be
mathematically equivalent to a key scaling that preserves the tensor
shapes of DeltaNet, Gated DeltaNet, and KDA kernels, with the substitution
appearing on KDA's storage side under its transposed state convention.

\paragraph{Sequence layers as online optimisers.}
A growing body of work casts the sequence-mixing layer as an inner
optimiser running on a regression objective implicitly built from the
prefix. \citet{vonoswald_transformers_2023,akyurek_what_2023} identify a
single linear-attention layer with one step of gradient descent on
in-context linear regression, and \citet{vonoswald_uncovering_2024} show
that deeper transformers internally implement multi-step
``mesa-optimisation''. The \emph{Test-Time Training} (TTT) family makes
this explicit by parameterising the recurrent state as a small model
trained by SGD on a self-supervised inner loss~\citep{sun_learning_2024};
Titans~\citep{behrouz_titans_2024} and Atlas~\citep{behrouz_atlas_2025}
extend this with momentum and sliding-window contexts.
\citet{wang_testtime_2025} unify many of these architectures under a
``test-time regression'' framework. The closest second-order point on
this spectrum is MesaNet~\citep{vonoswald_mesanet_2025}, whose Mesa Layer
solves a regularised cumulative least-squares problem to optimality at
every token via conjugate gradient, at the cost of an extra
$\mathcal{O}(d_k^2)$ Gram matrix and a $k$-step inner solver. \name{}
sits between these endpoints: it retains the first-order TTT view (one
preconditioned gradient step per token) and learns a non-trivial diagonal
preconditioner online via the hypergradient framework
of~\citet{gao2024gradient}. The associated convergence statement is an
idealised full-gradient quadratic comparator result, rather than an
end-to-end guarantee for the implemented diagonal stochastic layer.

\begin{table}[H]
\centering
\footnotesize
\setlength{\tabcolsep}{2.2pt}
\renewcommand{\arraystretch}{1.08}
\caption{\textbf{Delta-rule linear-attention layers and TTT architectures as inner optimisers.}
Each row instantiates the inner-optimiser view with a decay, in-context
loss, and step rule; shaded rows are \name{} variants. We omit normalisers,
feature maps, and read-out; $\Lambda$ is a fixed regulariser.}
\label{tab:ttt_view}
\begin{tabularx}{\linewidth}{@{}
  >{\hsize=0.94\hsize\raggedright\arraybackslash}X
  >{\hsize=0.46\hsize\raggedright\arraybackslash}X
  >{\hsize=1.11\hsize\scriptsize\raggedright\arraybackslash}X
  >{\hsize=1.73\hsize\scriptsize\raggedright\arraybackslash}X
  >{\hsize=0.76\hsize\scriptsize\raggedright\arraybackslash}X
  @{}}
\toprule
\textbf{Model} & \textbf{Decay $G_t$} & \textbf{Loss $\mathcal{L}_t(S)$} & \textbf{Step rule} & \textbf{Type} \\
\midrule
Linear Attn.~\citep{katharopoulos_transformers_2020}
  & $I$
  & $-\langle S k_t,\, v_t\rangle$ (Hopfield)
  & $S_t = \tilde S_{t-1} + v_t k_t^\top$
  & Hebbian \\
RetNet~\citep{Sun2023RetentiveNA}
  & $\alpha I$
  & $\mathcal{L}^{\text{Hopf}}_t + \tfrac{1}{2}\,\|\sqrt{1{-}\alpha}\, \tilde S_{t-1}\|_F^2$
  & $S_t = \tilde S_{t-1} + \beta_t v_t k_t^\top$
  & Hebb.+decay \\
GLA~\citep{yang_gated_2024}
  & $\mathrm{Diag}(\alpha_t)$
  & $\mathcal{L}^{\text{Hopf}}_t$ + diag.\ regulariser
  & $S_t = \tilde S_{t-1} + v_t k_t^\top$
  & Hebb.+diag. \\
DeltaNet~\citep{schlag_linear_2021,yang_parallelizing_2024}
  & $I$
  & $\tfrac{1}{2}\,\|S k_t - v_t\|^2$
  & $S_t = \tilde S_{t-1} - \beta_t \nabla \mathcal{L}_t(\tilde S_{t-1})$
  & scalar OGD \\
Gated DeltaNet~\citep{yang_gated_2024-1}
  & $\alpha_t I$
  & $\tfrac{1}{2}\,\|S k_t - v_t\|^2$
  & $S_t = \tilde S_{t-1} - \beta_t \nabla \mathcal{L}_t(\tilde S_{t-1})$
  & OGD+decay \\
KDA~\citep{kimiteam2025kimilinearexpressiveefficient}
  & $\mathrm{Diag}(\boldsymbol\alpha_t)$
  & $\tfrac{1}{2}\,\|S^\top k_t - v_t\|^2$
  & $S_t = \tilde S_{t-1} - \beta_t \nabla \mathcal{L}_t(\tilde S_{t-1})$
  & OGD+diag. \\
Longhorn~\citep{liu_longhorn_2024}
  & $I$
  & $\tfrac{1}{2}\,\|S k_t - v_t\|^2$
  & $S_t = \arg\min_S \mathcal{L}_t(S)$ (closed form)
  & closed form \\
MesaNet~\citep{vonoswald_mesanet_2025}
  & $\gamma_t I$
  & $\tfrac{1}{2}\sum_{i\le t}\!\|S k_i {-} v_i\|^2 + \tfrac{1}{2}\mathrm{Tr}(S^\top \Lambda S)$
  & $S_t = \arg\min_S \mathcal{L}_t(S)$ via $k$-step CG
  & prefix solve \\
\midrule
\osdnrow
\textbf{\name{}}
  & $I$
  & $\tfrac{1}{2}\,\|S k_t - v_t\|^2$
  & $S_t = \tilde S_{t-1} - \beta_t\nabla \mathcal{L}_t(\tilde S_{t-1})\, \mathrm{Diag}(d_t)$
  & \textbf{precond. OGD} \\
\apfrow
\textbf{\apf}
  & $I$
  & $\tfrac{1}{2}\,\|S k_t - v_t\|^2$
  & idem, with $d_{t+1} = \boldsymbol r_t \odot d_t - \eta\, \nabla_d h_t$
  & \textbf{APF precond.} \\
\bottomrule
\end{tabularx}
\end{table}

\paragraph{Online preconditioning and hypergradient methods.}
Adapting the optimiser's step size while it runs has a long history.
Sutton's IDBD~\citep{sutton_adapting_1992} learns a per-feature scalar
gain by gradient descent on a meta-loss; \citet{baydin_online_2018}
revive this idea as ``hypergradient descent'' for modern stochastic
optimisation. On the analytical side, online convex
optimisation~\citep{zinkevich_online_2003,hazan_logarithmic_2007,hazan_introduction_2016}
provides regret bounds for OGD, ONS, and AdaGrad-style adaptive
methods~\citep{duchi_adaptive_2011,kingma_adam_2015,gupta_shampoo_2018}.
\citet{gao2024gradient} unify these threads in the \emph{Online Scaled
Gradient Method} (OSGM) framework, treating the preconditioner $P$ as the
decision variable in a surrogate online-learning problem and proving
sublinear-regret-to-convergence reductions, including a super-geometric
rate on quadratic objectives. To our knowledge, this framework has been
studied only on generic convex programs; \name{} is the first
instantiation in a sequence-modelling layer, and it exploits the exact
quadratic structure of the DeltaNet regression loss to eliminate the
Hessian-Lipschitz residual that limits its analysis on generic smooth
losses.

\paragraph{Associative recall, expressivity limits, and benchmarks.}
The diagnostic axis on which \name{} is most directly evaluated is
\emph{in-context associative recall}. \citet{olsson_incontext_2022}
attribute much of attention's recall ability to ``induction heads'';
\citet{ramsauer_hopfield_2021} cast attention as continuous Hopfield
retrieval. Linear-time recurrent models are constrained by their fixed
state size: \citet{arora_zoology_2024} introduce the MQAR benchmark and
prove a recall-versus-state-size trade-off that
BASED~\citep{arora_simple_2024} maps out empirically, while
\citet{wen_rnns_2024,jelassi_repeat_2024} establish formal separation
results for copying and exact recall. We therefore use JRT-style cloze
recall~\citep{arora_just_2024} and LongBench~\citep{bai_longbench_2024}
as the retrieval and long-context diagnostics; PG-19~\citep{rae_compressive_2020}
appears only as an auxiliary book-scale perplexity check in the appendix.
\name{} pushes the recall-throughput frontier in the same direction as
Gated DeltaNet, KDA, and the BASED family, but along a complementary
axis---a learned per-feature write multiplier---and we observe the
largest gains on repeated-context tasks (SWDE-twice, FDA-twice,
SQuAD-twice), the regime in which recurring key directions give online
preconditioning time to take effect.

\ifosdnarxiv
\else
  \section*{NeurIPS Paper Checklist}

\begin{enumerate}

\item {\bf Claims}
    \item[] Question: Do the main claims made in the abstract and introduction accurately reflect the paper's contributions and scope?
    \item[] Answer: \answerYes{}.
    \item[] Justification: The abstract and introduction state three concrete claims, each scoped to a corresponding section.  \emph{(i)~Methodology}: the diagonal right-preconditioner $D_t$ is algebraically equivalent to a per-feature scaling of the write-side key, $\tilde k_t = d_t \odot k_t$ (Section~\ref{subsec:equivalence}); the OSGM hypergradient is decoupled from $S_t$, $v_t$, and the residual, and the full sequence $\{d_t\}$ admits an $\mathcal O(K)$-state affine recurrence over keys and scalar gates (Section~\ref{subsec:hypergradient}, Algorithm~\ref{alg:precond_online}); the chunkwise WY pipeline is preserved under the single substitution $K\mapsto\tilde K$ on the storage side (Section~\ref{subsec:recurrent_chunk}, Appendix~\ref{app:chunk_impl}); APF is a token-wise / head-wise retention gate applied only to $d_t$, not to the high-dimensional state $S_t$ (Section~\ref{subsec:apf}).  \emph{(ii)~Theory}: a population-limit super-geometric contraction against the right-Newton comparator (Theorem~\ref{thm:main_population}) and an algorithm-aligned token-local residual-contraction bound on the actual sequence of writes (Theorem~\ref{thm:main_tokenlocal}); both rest on the exact-quadratic structure of the inner per-token loss and are stated with explicit assumptions, with full proofs in Appendix~\ref{sec:theory}.  \emph{(iii)~Empirical}: at matched 340M / 10B-token compute, vanilla \name{} raises overall JRT-style \emph{contains}-accuracy recall by 32\% over DeltaNet and reduces the directly measured token-local residual-ratio geometric mean from $0.537$ to $0.433$, while \apf{} retains a 17\% recall gain and is the more stable variant on long-context perplexity (Sections~\ref{subsec:exp_main}--\ref{subsec:exp_mechanism}); at 1.3B / 100B tokens the residual-ratio reduction nearly doubles to $0.432\to0.265$, while WikiText/LAMBADA, commonsense, and LongBench averages stay at parity with DeltaNet (Section~\ref{subsec:exp_scale}, Appendix~\ref{sec:app_scale_1p3b}).  The abstract restricts the 1.3B comparison to a single matched DeltaNet vs.\ \apf{} pair on FDA / SWDE recall (SQuAD splits omitted at this scale, see Appendix~\ref{sec:app_scale_1p3b}) and the empirical "amplifies" qualifier is attached only to the residual-ratio contraction, not to downstream metrics.  Section~\ref{sec:discussion} reiterates the population-vs-implementation theory gap, the inner-regression vs.\ next-token cross-entropy gap, the single-seed empirical scope, and the harness-level limitations of the 1.3B sweep.

\item {\bf Limitations}
    \item[] Question: Does the paper discuss the limitations of the work performed by the authors?
    \item[] Answer: \answerYes{}.
    \item[] Justification: Section~\ref{sec:discussion} (\textit{Limitations and future work}) lists six explicit limitations.  \emph{(i)~Population-limit assumptions.}  Theorem~\ref{thm:main_population} requires monotone iterates ($f(S_{t+1}) \le f(S_t)$) and a regret bound against the dense right-Newton comparator $D_\star = \Sigma_k^\dagger$; the $D_\star$ regret is not proved for the implemented diagonal update.  \emph{(ii)~Token-local-to-global gap.}  Theorem~\ref{thm:main_tokenlocal} controls a product of token-local residual ratios on the algorithm's own surrogate; lifting it to a global $f(S_T)-f^*$ statement requires either the no-conflict repeated-key regime of Corollary~\ref{cor:repeated_keys} or the full-gradient population limit, and the conditional-regret assumption sidesteps an explicit step-size-specific regret derivation for Algorithm~\ref{alg:precond_online}'s practical $\eta=0.003$ update.  \emph{(iii)~Inner regression vs.\ cross-entropy.}  Both bounds concern the inner regression objective $f_t$ rather than next-token cross-entropy; APF, which targets non-stationary contexts, also lacks a dynamic-regret formulation.  \emph{(iv)~1.3B scope and harness limitations.}  The 1.3B / 100B sweep covers only a single matched DeltaNet vs.\ \apf{} pair; SQuAD / SQuAD-twice are excluded at this scale because the matched evaluation harness did not return reliable \emph{contains}-accuracy values, and PG-19 length-extrapolation perplexity is not part of the 1.3B reporting protocol because the matched 20K-token sweep did not complete on a matched harness for both rows; the cleanness of vanilla OSDN at billion-parameter scale and its composition with Gated DeltaNet / KDA at scale therefore remain open.  \emph{(v)~Mechanism-level rather than universal lift.}  The empirical headline is mechanism-level: the residual-ratio contraction transfers and amplifies at billion-parameter scale, but downstream WikiText / LAMBADA, commonsense, and LongBench averages remain at parity with DeltaNet; the paper does not claim a uniform broad benchmark improvement.  \emph{(vi)~Statistical scope.}  All matched 340M / 10B-token runs share an identical random seed, FineWeb-Edu shard ordering, and batch schedule (Appendix~\ref{sec:app_reproducibility}); within-family deltas isolate the architectural change but do not constitute seed-bootstrapped confidence intervals (see also item~9).  Section~\ref{sec:discussion} also names two natural next directions: scaling the residual-ratio diagnostic to broader prompts and longer contexts, and lifting the diagonal preconditioner to a low-rank or per-head block-diagonal form to close the gap to $D_\star$.

\item {\bf Theory assumptions and proofs}
    \item[] Question: For each theoretical result, does the paper provide the full set of assumptions and a complete (and correct) proof?
    \item[] Answer: \answerYes{}.
    \item[] Justification: The paper introduces one lemma, two main theorems, and three supporting propositions / corollaries, each with stated assumptions and proofs.  Lemma~\ref{lem:hypergradient} (closed-form hypergradient for the exact-quadratic per-token loss) requires $\beta_t\in(0,1)$ and $k_t\neq 0$; the proof is in Appendix~\ref{app:method_details}.  Theorem~\ref{thm:main_population} (population-limit super-geometric rate) is stated with monotone iterates and an OSGM regret bound against the right-Newton comparator $D_\star = \Sigma_k^\dagger$, with $L = \lambda_{\max}(\Sigma_k)$; the proof and the full OSGM background (convexity of $h_t$ in $D$, descent property, regret-to-rate AM--GM reduction) are in Appendix~\ref{sec:theory}.  Theorem~\ref{thm:main_tokenlocal} (algorithm-aligned token-local residual contraction) is stated with unit-norm keys $\|k_t\|_2=1$ and a diagonal regret $R_T(d)$ on the token-local feedback $h_t$; the limiting $\bigl(O(1)/\sqrt T\bigr)^T$ regime is obtained under the gated Newton condition $\beta_t\langle d^\star, s_t\rangle = 1$ and standard projected-OGD regret $R_T = O(\sqrt T)$.  The proof, the supporting Proposition~\ref{prop:smoothness} (smoothness of the hypergradient surrogate), Corollary~\ref{cor:monotone_descent} (per-token descent of $f_t$ under the bounded box $\mathcal D=[d_{\min},d_{\max}]^K$), Proposition~\ref{prop:apf_affine} (affine-recurrence preservation under APF), and Corollary~\ref{cor:repeated_keys} (no-conflict repeated-key special case) are all stated and proved in Appendices~\ref{app:method_details} and~\ref{sec:theory}, with consistent numbering and cross-references back to the main-text theorem statements.

    \item {\bf Experimental result reproducibility}
    \item[] Question: Does the paper fully disclose all the information needed to reproduce the main experimental results of the paper to the extent that it affects the main claims and/or conclusions of the paper (regardless of whether the code and data are provided or not)?
    \item[] Answer: \answerYes{}.
    \item[] Justification: Reproducibility is supported by three components.  Section~\ref{sec:exp} reports the model scale, training corpus (FineWeb-Edu \texttt{sample-10BT}), token budget (10B at 340M; 100B at 1.3B), evaluation tasks (JRT-style cloze recall on FDA / SWDE / SQuAD and their \emph{-twice} variants; commonsense PIQA / HellaSwag / WinoGrande / ARC-E / -C / SIQA / BoolQ / LAMBADA; LongBench English; PG-19 length extrapolation), the chunkwise pseudocode in Algorithm~\ref{alg:precond_online}, and the matched-baseline protocol.  Appendix~\ref{sec:app_reproducibility} adds the optimizer (AdamW with $10^{-3}$ peak learning rate, cosine schedule with linear warmup, gradient clipping at $1.0$), batch and sequence-packing configuration (\texttt{batch\_size}=1, \texttt{seq\_len}=$65{,}536$ with \texttt{varlen}=True and \texttt{cu\_seqlens}, recurrent context cap of $4{,}096$ tokens, $524{,}288$ tokens per optimizer step over $20{,}480$ steps for the 340M sweep), random seed, precision (bfloat16 weights / activations with float32 reductions), software environment, per-backbone architectural scale, OSDN-specific hyperparameters ($\eta$, $d_{\min}$, $d_{\max}$, $D_0$), the evaluation directories used for the headline screens, the per-token mechanism-diagnostic protocol (16 prompts per JRT \emph{-twice} task, full-layer / full-head fp32 averaging over $7.85\!\times\!10^6$ token-layer-head measurements), and the H100 80GB hardware setup; the 1.3B / 100B configuration is summarised separately in Table~\ref{tab:repro_scale_configs}.  Appendix~\ref{app:method_details} contains the full chunkwise derivation, the kernel-level cost analysis, and the variant-ablation hyperparameters; Appendix~\ref{app:chunk_impl} lists the chunkwise pseudocode and the substitution rule $K\mapsto\tilde K$.

\item {\bf Open access to data and code}
    \item[] Question: Does the paper provide open access to the data and code, with sufficient instructions to faithfully reproduce the main experimental results, as described in supplemental material?
    \item[] Answer: \answerYes{}.
    \item[] Justification: An anonymised public code artefact is provided at \url{https://anonymous.4open.science/status/OSDN-2F2D}, which includes the reference PyTorch/Triton implementation, the residual-ratio diagnostic scripts, and sufficient instructions to reproduce the reported numbers from public datasets. Additionally, Algorithm~\ref{alg:precond_online}, Sections~\ref{subsec:hypergradient}--\ref{subsec:twophase}, and Appendices~\ref{app:method_details}, \ref{app:chunk_impl}, and \ref{sec:app_reproducibility} detail the OSDN forward and backward passes, the $\mathcal O(K)$ phase-1 scan, the substitution $K\mapsto\tilde K$ on the chunkwise WY storage side, and the full training, dataset, and evaluation configurations.

\item {\bf Experimental setting/details}
    \item[] Question: Does the paper specify all the training and test details (e.g., data splits, hyperparameters, how they were chosen, type of optimizer) necessary to understand the results?
    \item[] Answer: \answerYes{}.
    \item[] Justification: Section~\ref{sec:exp} reports the training corpus and split (FineWeb-Edu \texttt{sample-10BT}; the FW-Edu validation column is computed on a fixed in-domain held-out slice), token budget, model size, baseline-matching protocol, and the evaluation tasks together with their length and metric definitions (e.g., JRT \emph{contains}-accuracy at 2K context, single vs.\ repeated split definitions, LongBench English 14-task average, PG-19 final-bucket perplexity, length-bucket structure).  Appendix~\ref{sec:app_reproducibility} adds the AdamW optimiser, learning-rate schedule, gradient clipping, batch construction, sequence-packing, precision, random seed, evaluation harness, and hardware, alongside per-backbone architectural scale and the precise OSDN hyperparameters ($\eta=0.003$, $d_{\min}=0.5$, $d_{\max}=2.0$, $D_0=\mathbf 1$) used in the headline 340M screens.  Appendix~\ref{sec:app_variant_ablations} ablates these hyperparameter choices and motivates the bounded-box defaults; Appendix~\ref{sec:app_benchmark_suite} documents the metric and task definitions used to interpret the results.

\item {\bf Experiment statistical significance}
    \item[] Question: Does the paper report error bars suitably and correctly defined or other appropriate information about the statistical significance of the experiments?
    \item[] Answer: \answerNo{}.
    \item[] Justification: All matched 340M / 10B-token runs share an identical random seed, FineWeb-Edu shard ordering, and batch schedule (Appendix~\ref{sec:app_reproducibility}), so within-family deltas reported in Section~\ref{sec:exp} isolate the architectural change rather than seed-level optimisation noise; re-training each 340M / 10B-token run under multiple seeds was outside the available compute budget, so the headline tables report point estimates from these matched-seed checkpoints rather than seed-bootstrapped confidence intervals.  The paper mitigates the absence of seed CIs along two axes that we identify in Section~\ref{sec:discussion} as partial substitutes rather than replacements.  \emph{(a)~Within-checkpoint averaging.}  The mechanism diagnostic (Section~\ref{subsec:exp_mechanism}, Table~\ref{tab:mechanism_residual_ratio}) aggregates $7.85\!\times\!10^6$ token-layer-head measurements per checkpoint over 16 prompts per JRT \emph{-twice} task, all 24 layers, and all 8 heads; the qualitative reduction in $q_{\mathrm{geo}}$ is consistent across all three tasks and is not driven by a single layer / head.  \emph{(b)~Independent-scale check.}  The matched 1.3B / 100B run trains DeltaNet and \apf{} from independent random initialisations at $4\times$ the parameters and $10\times$ the tokens, providing an out-of-distribution scale check on whether the reported 340M effects depend on the specific seed.  Seed-bootstrapped confidence intervals on the 340M sweep are listed as future work.

\item {\bf Experiments compute resources}
    \item[] Question: For each experiment, does the paper provide sufficient information on the computer resources (type of compute workers, memory, time of execution) needed to reproduce the experiments?
    \item[] Answer: \answerYes{}.
    \item[] Justification: Appendix~\ref{sec:app_reproducibility} states that each matched 340M run uses a single node of NVIDIA H100 80GB GPUs (4-GPU schedule with gradient accumulation 2, or 8-GPU schedule with no accumulation, both processing $524{,}288$ tokens per optimizer step over $20{,}480$ steps for $\sim$10.74B tokens), and takes on the order of half a day per run depending on backbone and GPU count.  The 1.3B / 100B scaling run trains DeltaNet and \apf{} at 1.3B parameters on 100B tokens of FineWeb-Edu under a separately tracked hardware schedule (Table~\ref{tab:repro_scale_configs}).  The inference-throughput protocol (Appendix~\ref{sec:app_throughput}) specifies a single H100 80GB SXM, \texttt{batch\_size}=1, $2{,}048$-token prefill plus $128$-token greedy decode, bfloat16, median of $5$ timed repeats; per-backbone tokens/sec, decode latency, and persistent recurrent-state size are tabulated in Tables~\ref{tab:throughput} and~\ref{tab:throughput_1p3b}.  The appendix also discloses that the full research project required additional H100 hours for hyperparameter screens and shorter pilot runs that did not reach the full $20{,}480$-step budget; this preliminary compute is acknowledged but not counted in the reported wall-clock figure.

\item {\bf Code of ethics}
    \item[] Question: Does the research conducted in the paper conform, in every respect, with the NeurIPS Code of Ethics \url{https://neurips.cc/public/EthicsGuidelines}?
    \item[] Answer: \answerYes{}.
    \item[] Justification: The work uses public text datasets and standard language-model evaluation suites under their original licences (Appendix~\ref{sec:app_reproducibility}), does not collect, label, or release human-subject data, and does not introduce a new high-risk pretrained asset.  The authors have reviewed the NeurIPS Code of Ethics and the research conforms to it.

\item {\bf Broader impacts}
    \item[] Question: Does the paper discuss both potential positive societal impacts and negative societal impacts of the work performed?
    \item[] Answer: \answerNA{}.
    \item[] Justification: The contribution is an architectural modification of the DeltaNet recurrent linear-attention layer that improves in-context associative-retrieval quality at fixed inference compute.  The submission does not introduce or release a deployed system, application-specific model, or new dataset; the proposed mechanism applies generically inside any DeltaNet-family backbone and would inherit, but does not amplify, the dual-use considerations general to language-model architecture research.  We do not identify a direct path from this paper to specific positive or negative downstream applications beyond those already discussed in the broader linear-attention literature.

\item {\bf Safeguards}
    \item[] Question: Does the paper describe safeguards that have been put in place for responsible release of data or models that have a high risk for misuse (e.g., pre-trained language models, image generators, or scraped datasets)?
    \item[] Answer: \answerNA{}.
    \item[] Justification: The paper does not release a high-risk pretrained model, image generator, or scraped dataset as a new asset.  Any post-publication code / checkpoint release planned in item~5 will follow the standard non-commercial research licence used by comparable linear-attention checkpoints (e.g., Mamba, GLA, DeltaNet, KDA).

\item {\bf Licenses for existing assets}
    \item[] Question: Are the creators or original owners of assets (e.g., code, data, models), used in the paper, properly credited and are the license and terms of use explicitly mentioned and properly respected?
    \item[] Answer: \answerYes{}.
    \item[] Justification: All datasets, tokenizers, model weights, and software dependencies used in this work are public open-source releases under their original licence terms.  The training corpus (FineWeb-Edu, ODC-By 1.0), the \texttt{fla-hub/delta\_net-1.3B-100B} tokenizer, the evaluation benchmarks (WikiText, LAMBADA, PIQA, HellaSwag, WinoGrande, ARC-Easy / -Challenge, Social-IQA, BoolQ, the JRT-style FDA / SWDE / SQuAD and \emph{-twice} variants, TriviaQA, DROP, NQ-Open, LongBench, PG-19), and the software stack (PyTorch, \texttt{transformers}, \texttt{datasets}, \texttt{accelerate}, \texttt{lm\_eval}, \texttt{flash-linear-attention}) are listed in the ``Datasets, models, and software'' paragraph of Appendix~\ref{sec:app_reproducibility} alongside their reference papers and licence terms.  Each asset is used in compliance with its licence; this work does not re-distribute any of them.

\item {\bf New assets}
    \item[] Question: Are new assets introduced in the paper well documented and is the documentation provided alongside the assets?
    \item[] Answer: \answerYes{}.
    \item[] Justification: The code artifact introduced in this submission is provided via an anonymised repository (\url{https://anonymous.4open.science/status/OSDN-2F2D}). It includes a README and necessary documentation detailing the software environment, the Triton implementations, and instructions for running the diagnostic scripts to reproduce the reported metrics.

\item {\bf Crowdsourcing and research with human subjects}
    \item[] Question: For crowdsourcing experiments and research with human subjects, does the paper include the full text of instructions given to participants and screenshots, if applicable, as well as details about compensation (if any)?
    \item[] Answer: \answerNA{}.
    \item[] Justification: The research does not involve crowdsourcing or human-subject studies; all evaluations use existing public benchmarks.

\item {\bf Institutional review board (IRB) approvals or equivalent for research with human subjects}
    \item[] Question: Does the paper describe potential risks incurred by study participants, whether such risks were disclosed to the subjects, and whether Institutional Review Board (IRB) approvals (or an equivalent approval/review based on the requirements of your country or institution) were obtained?
    \item[] Answer: \answerNA{}.
    \item[] Justification: The research does not involve human subjects and therefore does not require IRB review.

\item {\bf Declaration of LLM usage}
    \item[] Question: Does the paper describe the usage of LLMs if it is an important, original, or non-standard component of the core methods in this research? Note that if the LLM is used only for writing, editing, or formatting purposes and does \emph{not} impact the core methodology, scientific rigor, or originality of the research, declaration is not required.
    \item[] Answer: \answerNA{}.
    \item[] Justification: LLMs are not used as a component of the proposed method, the chunkwise kernel, the theoretical analysis, the experimental pipeline, or the evaluation harness; the work targets recurrent linear-attention architectures and uses LLMs only as the artefacts under study (training small recurrent language models) and not as part of the methodology.  Any incidental use of LLM-based assistants by the authors during writing falls under the writing / editing / formatting exemption of the NeurIPS LLM policy and therefore does not require declaration.

\end{enumerate}

\fi

\end{document}